\documentclass[final]{l4dc2026}

\usepackage{comment}
\usepackage{wrapfig}
\usepackage{enumitem}

\newtheorem{claim}{Claim}
\newtheorem*{claim*}{Claim}

\newtheorem*{assumption*}{Assumption}

\DeclareMathOperator{\pr}{\mathbb{P}}

\newcommand{\E}[1]{\mathbb{E}\left[#1\right]}

\usepackage{mathtools} 
\usepackage{dsfont}    

\newcommand{\sigmaMin}{\sigma_{\min}^\star}
\newcommand{\sigmaMax}{\sigma_{\max}^\star}

\newcommand{\phiPop}{\varphi^\star}
\newcommand{\nEffStar}{n_{\mathrm{eff}}^\star}
\newcommand{\nEff}{n_{\mathrm{eff}}}
\newcommand{\gamERhet}{\gamma_{\mathrm{het}}}

\newcommand{\Texp}{T_{\exp}}
\newcommand{\Gap}{\Delta_{\min}}

\usepackage{mathtools} 

\usepackage{dsfont}    

\title[Flickering Multi-Armed Bandits]{Flickering Multi-Armed Bandits}
\usepackage{times}

\newcommand{\CU}{\textsuperscript{\textdagger}}      
\newcommand{\INRIA}{\textsuperscript{\textdaggerdbl}}  

\author{%
  \Name{Sourav Chakraborty}\CU \Email{sourav.chakraborty@colorado.edu}\\
  \Name{Amit Kiran Rege}\CU\footnote{Equal Contribution with S. Chakraborty.} \Email{amit.rege@colorado.edu}\\
  \Name{Claire Monteleoni}\CU\INRIA \Email{claire.monteleoni@colorado.edu}\\
  \Name{Lijun Chen}\CU \Email{lijun.chen@colorado.edu}\\
  \addr\CU University of Colorado Boulder, USA.\\
  \addr\INRIA INRIA Paris, France.%
}

\begin{document}

\maketitle

\begin{abstract}
    We introduce Flickering Multi-Armed Bandits (FMAB) to model sequential decision-making in environments with changing action availability, where accessibility of the next action is restricted to a subset dependent on the agent's current choice. We formalize these constraints through stochastically evolving graphs where actions are limited to local neighborhoods. This mobility-constrained structure imposes a dual challenge: the statistical requirement of information acquisition and the physical overhead of navigation. We analyze FMAB under i.i.d. Erdős–Rényi and Edge-Markovian process, proposing a two-phase lazy random walk algorithm for robust exploration. We establish high-probability sublinear regret bounds and prove near-optimality via a matching information-theoretic lower bound. Our results characterize the intrinsic cost of learning under local-move constraints, complimented by a robotic disaster-response simulation.
\end{abstract}

\begin{keywords}%
  Online Learning; Multi-Armed Bandits.%
\end{keywords}

\section{Introduction}\label{sec:intro}

In the aftermath of a natural disaster, emergency responders must rapidly restore communication capabilities across a disrupted and unstable region. A robotic ground vehicle is deployed to scout candidate locations, moving carefully through streets that may be temporarily blocked by debris or diverted traffic. The environment is fluid: pathways open and close unexpectedly, and access to promising sites may vanish without warning. At each stop, the vehicle assesses the quality of local coverage it could provide if deployed there. Over time, it explores locations, limited to nearby, currently accessible sites, while building a sense of which locations offer the most promise. After navigating this uncertain landscape, the vehicle must ultimately select the site where it can provide the most reliable service, deploying the relay at that location to support ongoing recovery efforts.

This scenario reveals a decision-making structure where the agent interacts with an environment with changing availability of actions. At each step, it can access only a limited subset of those actions, dependent on the agent's most recent choice. This class of problems falls under the broad umbrella of sequential decision-making under uncertainty, a domain formally studied through the multi-armed bandit (MAB) framework \citep{thompson1933likelihood, lai1985asymptotically}. In its classical form, the learner repeatedly interacts with a fixed set of actions (referred to interchangeably as \textit{arms}), each of which yields an unknown real-valued payoff, or \textit{reward}. At each round, the learner selects one arm, receives a corresponding reward, and uses this information to guide future decisions. The goal is to make a sequence of choices that collectively lead to high cumulative reward over time. At its core, the MAB model captures the tension between selecting uncertain actions to improve the agent’s knowledge (\textit{exploration}), and leveraging current beliefs to choose the action that appears most promising (\textit{exploitation}). This model has been successfully applied across a wide range of applications, including clinical trials \citep{gittins1}, wireless communication \citep{cogradio}, recommendation systems \citep{Li_2010, bouneffouf:hal-00753401}, financial optimization \citep{brochu}, and economically motivated incentivized explorations \citep{fraz, wang, liu20,chakraborty2024incentivized,chakraborty2025incentivized}.

Classical MAB framework provides a rich suite of strategies for balancing exploration and exploitation, supported by well-established theoretical guarantees (see surveys \citep{slivkins2024introduction, lattimore_2020}). However, these results assume that the learner has unrestricted access to all actions. This assumption fails in many real-world settings, including the emergency relay deployment example mentioned earlier, where the set of available actions may change over time and depend on the agent’s most recent choice. In such environments, the learner is constrained to select from only a limited, dynamically varying subset of actions at each step. These limitations fundamentally alter the learning problem and require new models that explicitly account for restricted and evolving access to actions.


We introduce the Flickering Multi-Armed Bandit (FMAB) model, where the set of available actions: (i) may be a restricted subset of the full action set, (ii) can change over time, and (iii) may depend on the agent’s most recent arm choice. As in standard bandit problems, the learner’s objective is to maximize cumulative reward by identifying and repeatedly selecting high-performing actions. However, FMAB introduces a local-move constraint: while the agent may eventually reach any arm over the long-term, it may lack global access to the entire action set at any single round. This requires a strategy that balances the need to visit all arms for identification with the challenge of navigating a landscape of fluctuating connectivity.

We formalize the constraints through random graph processes defined over the action set. Under this representation, the available actions at each round are restricted to the neighborhood of the learner’s current arm. We consider two primary graph processes. First, we analyze the i.i.d. setting, defined by the Erd\H{o}s--R\'enyi (ER) model \citep{erdos1960evolution}, where each potential edge is sampled independently at every round. Our analysis addresses the general heterogeneous case where probabilities are edge-specific, which inherently encompasses the homogeneous setting of uniform probabilities as a special instance (see Section \ref{para:er-model}). Second, we introduce the Edge-Markovian model \citep{edge-flip1, edge-flip2}. In this setting, the graph is not fully re-sampled; instead, it evolves from an initial map as edges independently appear or disappear according to fixed probabilities (see Section \ref{para:markovian-model}).

\paragraph{Contributions.}
We formalize the FMAB framework (see Figure \ref{fig:illustration}) to characterize environments where action availability is both dynamic and dependent on the agent's preceding selection. We model these constraints using two canonical processes: (i) i.i.d. Erdős–Rényi evolution (encompassing both homogeneous and heterogeneous cases) and (ii) Edge-Markovian process. We analyze learning under the structural evolution of random graphs, not previously addressed. Our work bridges spectral techniques for evolving Markov chains with concentration tools for time-varying graphs to provide a unified regret analysis. We establish high-probability sublinear regret bounds and demonstrate their near-optimality by deriving a matching information-theoretic lower bound. To our knowledge, this is the first analysis to couple lazy-walk exploration with these specific classes of evolving availability. Finally, we validate our theoretical findings through simulations, including a robotic scouting scenario in a disaster-response environment.

\begin{figure}[t]
  \centering
  \includegraphics[width=\linewidth]{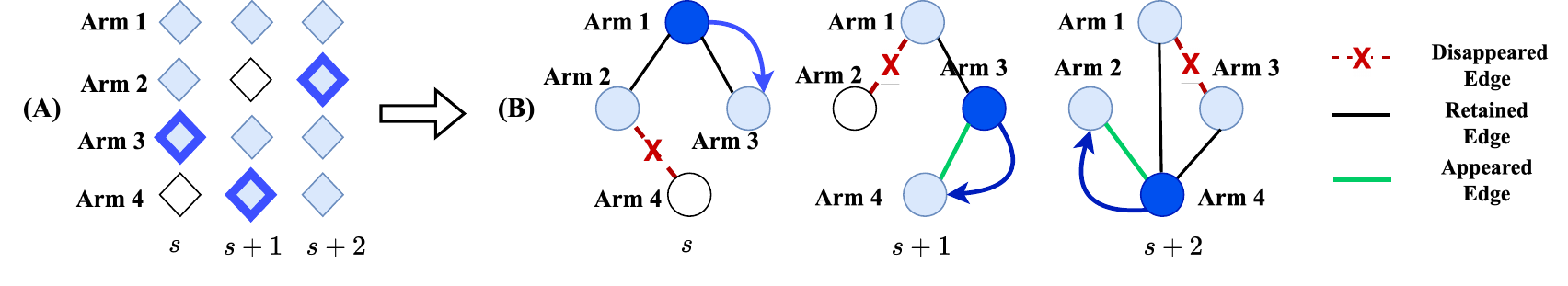}
  \caption{A 4-arm FMAB problem at $t=s, s+1, s+2$, with pull sequence $a_s=3, a_{s+1}=4, a_{s+2}=2$ (from $a_{s-1}=1$). \textbf{(A)} Problem view: Blue/White arms are accessible/inaccessible. Pulled arm has a dark blue border. \textbf{(B)} Graph view: Learner starts at the dark blue node ($a_{t-1}$) and can move to blue neighbors ($L_t(a_{t-1})$).}
  \label{fig:illustration}
\end{figure}

\paragraph{Related work.}
The Sleeping Bandits model \citep{kleinberg2010regret} features time-varying arm sets, but availability is independent of the learner's actions and measured against a different, time-varying benchmark. Mortal Bandits \citep{chakrabarti2008mortal} consider arms that disappear intermittently or permanently, but maintain a \textit{constant number} of active arms and reveal the availability set at each round regardless of the past action. Volatile Bandits \citep{bnaya2013volatile} associate arms with unknown lifespans and have time-varying rewards, causing the optimal action to shift. Graph-based bandit models \citep{lina-graph, lina-graph-2} also place arms on a graph, but typically assume a \textit{fixed}, globally-known graph, which enables shortest-path-based planning. In contrast, FMAB models availability using a \textit{changing} graph, and we develop lightweight algorithms that rely only on neighborhood information.


\section{Models and Problem Formulation}\label{sec:problem-formulation}

We consider a finite set of $n$ arms (or \emph{actions}) $A = \{1, 2, \dots, n\}$, with each arm $a \in A$ associated with a fixed but unknown reward distribution $\mathcal{D}(a)$ on $[0,1]$ and mean $\mu(a) \in [0,1]$. Without the loss of generality, we assume $a^* \triangleq \arg\max_{a \in A} \mu(a)$ to be the unique optimal arm. We denote the sub-optimality gap for any other arm as $\Delta(a) \triangleq \mu(a^*) - \mu(a) > 0$. The learner's objective is to maximize the cumulative reward over a time horizon $T$. Since the reward of the optimal arm is a fixed benchmark, this is equivalent to minimizing the regret $R(T)$, defined as the difference between the total reward achievable by an oracle and the learner's cumulative reward:
\[
R(T) \triangleq \sum_{t=1}^{T} \left( \mu(a^*) - r_t(a_t) \right).
\]
The goal is to design a learning algorithm that selects a sequence of actions $a_1, \dots, a_T$ to ensure the expected regret is sublinear, i.e., $\mathbb{E}[R(T)] = o(T)$, as $T \to \infty$. A key feature of our model is that the learner's access to arms is constrained by a stochastically evolving graph. At each round $t$, the environment reveals a graph $G_t = (A, E_t)$ over the action set. The learner, residing at its previously chosen arm $a_{t-1}$, gets access only to its neighborhood, which defines the set of available arms as:
\(
L_t(a_{t-1}) \triangleq \{\,a \in A : (a, a_{t-1}) \in E_t\,\} \cup \{a_{t-1}\}.
\)
The learner then selects an arm $a_t \in L_t(a_{t-1})$, and  receives a stochastic reward $r_t(a_t) \sim \mathcal{D}(a_t)$. To describe the stochastic evolution of the graphs $\{G_t\}$, we introduce a general graph evolution operator
\(
\Psi : \mathcal{G} \times \Theta \to \mathcal{P}(\mathcal{G}),
\)
where $\mathcal{G}$ denotes the space of graphs on the node set $A$, $\Theta$ is a model-specific parameter space, and $\mathcal{P}(\mathcal{G})$ is the space of probability measures over $\mathcal{G}$. At each round, the environment samples the next graph according to $G_t \sim \Psi(G_{t-1}; \theta)$, where $\theta \in \Theta$ specifies the parameters governing the evolution. We now define the two graph processes that form the basis of our analysis.


\begin{figure}[t]
  \centering
  \includegraphics[width=0.8\linewidth]{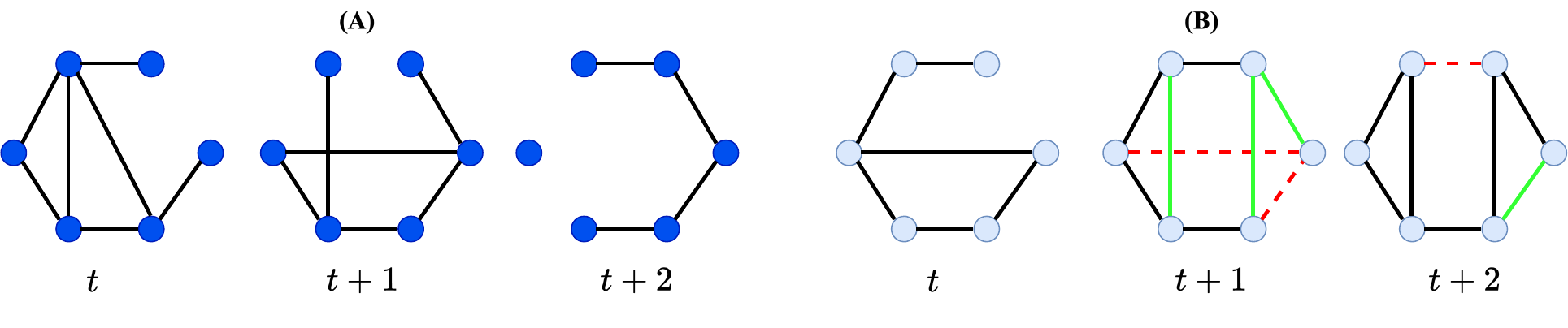}
  \caption{\textbf{Evolution of FMAB availability graphs}. (A) i.i.d. ER evolution: edge re-sampling with some probability $p_{ij}$ for each edge $(i,j)$ across snapshots $t, t+1, t+2$. (B) Edge-Markovian process: evolution featuring appearing edges (green solid, with probability $\alpha$) and disappearing edges (red dotted, with probability $\beta$) relative to the preceding round.}
  \label{fig:er-edge-markovian}
\end{figure}

\paragraph{Erdős--Rényi (ER) Graph Process.}
\phantomsection
\label{para:er-model}
In this setting, the environment generates an i.i.d. sequence of graphs $\{G_t\}$ (see Figure \ref{fig:er-edge-markovian} (A)). We define this through the \textit{heterogeneous} ER model, $\Psi^{\mathrm{ER}}_{\mathrm{het}}(\theta)$, with parameters $\theta = (n, \{p_{ij}\})$. At each round $t$, a new graph $G_t = (A, E_t)$ is realized by including each possible edge $(i,j) \in E_t$ independently with probability $p_{ij}$. This model effectively captures environments with non-uniform connectivity across the action set. The \textit{homogeneous} model, $\Psi^{\mathrm{ER}}_{\mathrm{hom}}(\theta)$ with $\theta = (n, p)$, is a special case where $p_{ij} = p$ for all pairs $(i,j)$. Because the sequence of graphs is i.i.d., the distribution of $G_t$ is independent of the previous graph $G_{t-1}$. Consequently, we omit the conditional argument from the notation for this process.

\paragraph{Edge-Markovian Graph Process.}
\phantomsection
\label{para:markovian-model}
This process is defined by the operator $\Psi^{\mathrm{M}}(G_{t-1};\theta)$ with parameters $\theta = (n, \alpha, \beta)$ (see Figure \ref{fig:er-edge-markovian} (B)). In this model, the graph $G_t = (A, E_t)$ evolves from the previous configuration $G_{t-1}$ as each potential edge follows an independent, time-homogeneous Markov chain with two states. Specifically, an edge absent in $G_{t-1}$ appears in $G_t$ with probability $\alpha \in (0,1)$, and an edge present in $G_{t-1}$ disappears with probability $\beta \in (0,1)$. These transitions occur independently across all edges. This formulation allows the framework to capture environments where action availability is dictated by the preceding graph configuration, effectively modeling the ``stickiness" or persistence of connectivity over time.

The evolving graph models above impose localized access (see Figure \ref{fig:illustration}): at each round, the learner can only move to neighbors in $G_t$, restricting standard bandit strategies that assume global access. In the next section, we introduce algorithms that use only local movement and feedback.

\section{A Phase-Based Learning Algorithm}
\label{sec:algorithm}

Our algorithm is a two-phase learning strategy: an \emph{exploration} phase of a predetermined length $T_0$, followed by an \emph{exploitation} phase. To facilitate this, the learner maintains two quantities for all arms $a \in A$, initialized to zero: the total visitation count $\phi_t(a)$ and cumulative reward sum $S_t(a)$.

\begin{itemize}[leftmargin=*,itemsep=2pt,topsep=3pt]
    \item \textbf{Phase I} (Best Arm Identification, $1 \le t \le T_0$):
    The goal of this phase is to collect sufficient information to identify the optimal arm. To do this, the learner employs a \emph{lazy random walk} policy, defined by a one-step transition kernel $W_t$. At each round, the learner selects $a_t$ uniformly at random from the available set $L_t(a_{t-1})$:
    \begin{align*}
        W_t(i, j) &\triangleq \mathbb{P}(a_t = j | a_{t-1} = i, G_t) = \frac{1}{|L_t(i)|} \quad \forall j \in L_t(i) \\
        W_t(i, j) &= 0 \quad \text{if } j \notin L_t(i).
    \end{align*}
    Upon selecting $a_t$ and receiving $r_t(a_t)$, the learner updates its trackers:
    $\phi_t(a_t) = \phi_{t-1}(a_t) + 1$ and $S_t(a_t) = S_{t-1}(a_t) + r_t(a_t)$. At the end of this phase (at round $T_0$), the learner identifies the target arm by computing the empirical means $\hat{\mu}(a) \triangleq S_{T_0}(a) / \phi_{T_0}(a)$ (using $\hat{\mu}(a) = 0$ if $\phi_{T_0}(a) = 0$), and setting $\hat{a}^* \triangleq \arg\max_{a \in A} \hat{\mu}(a)$.

    \item \textbf{Phase II} (Navigation \& Commitment, $T_0 < t \leq T$):
    With the target arm $\hat{a}^*$ now identified, the learner continues to follow the same lazy random walk policy $W_t$ until the first round $T_0 < \tau \leq T$ that it lands on $\hat{a}^*$. For all subsequent rounds $\tau < t \leq T$, the learner selects $a_t = \hat{a}^*$.
\end{itemize}

The above algorithm serves as a unified template across all graph models, but its performance analysis and the choice of the phase length $T_0$ depend on the evolution dynamics of the environment, that we will analyze in Section~\ref{sec:performance-guarantees}.

\section{Performance Guarantees}
\label{sec:performance-guarantees}

We now present our main theoretical results, establishing high-probability, sublinear regret guarantees for the two-phase learning algorithm introduced in Section \ref{sec:algorithm}. We analyze the algorithm's performance under the two primary random graph models mentioned in Section \ref{sec:problem-formulation}.

\subsection{Regret for FMAB with i.i.d. Erd\H{o}s--R\'enyi Graphs}
\label{sec:er-results}

We first analyze our algorithm (Section \ref{sec:algorithm}) under i.i.d. graph evolution. During the exploration phase (Phase I), the agent performs a \emph{lazy random walk} for a predetermined length $T_0$. The goal of this phase is to collect enough samples to identify the optimal arm $a^*$, which has the highest mean reward $\mu(a^*)$. We define the minimum sub-optimality gap as $\Gap \triangleq \min_{a \ne a^*} (\mu(a^*) - \mu(a))$.

In the homogeneous setting, $\Psi^{\mathrm{ER}}_{\mathrm{hom}}(n,p)$, the exploration walk is characterized by its expected transition kernel $\bar{W} \triangleq \mathbb{E}[W_t]$. Because the environment re-samples the graph independently and every potential edge shares the same probability $p$, all nodes are statistically indistinguishable. This inherent symmetry ensures that $\bar{W}$ is a symmetric and doubly stochastic matrix, which implies a uniform stationary distribution $\pi(a) = 1/n$. We analyze this process using a \textit{regeneration} argument (Appendix \ref{sec:app-er-homo}), showing that the walk has a constant probability of resetting its dependence on the past and jumping to a random arm. This ``forgetting" property guarantees that the walk covers all $n$ arms efficiently.

This symmetry-based approach is not applicable to the heterogeneous case, $\Psi^{\mathrm{ER}}_{\mathrm{het}}(n,\{p_{ij}\})$, where the expected kernel $\bar{W}$ is generally not symmetric due to the non-uniform edge probabilities. Instead, our analysis (Appendix \ref{app:er-heterogeneous}) utilizes a uniform typicality argument. We prove that, with high probability, the structural properties of the realized graph $G_t$ at every round are sharply concentrated around their expected values. This concentration ensures that even without global symmetry, the realized walk mixes efficiently across the action space to maintain the uniform coverage required for identification. These two distinct analyses culminate in a unified regret structure. This result accounts for both the exploration cost of the lazy walk and the navigation cost of the policy defined in Section \ref{sec:algorithm}. We formalize these findings in the following theorem.


\vspace{-0.5em}
\begin{theorem}[Regret for ER Graphs]
\label{thm:main-regret-er}
Under the FMAB model with i.i.d. ER graph evolution, for any $\delta \in (0,1)$, there exist constants $c_i > 0$ such that for $T_0 \ge c_1 n \log(nT/\delta) + c_2 n \log(n/\delta)/\Gap^2$, the regret $R(T)$ satisfies, with probability $\ge 1-\delta$:
\vspace{-1em}
\begin{equation*}
R(T) \le \underbrace{c_1 n \log(nT/\delta) + c_2 \frac{n \log(n/\delta)}{\Gap^2}}_{\text{Exploration Cost}} + \underbrace{c_3 n \log(n/\delta)}_{\text{Navigation Cost}}.
\end{equation*}
\end{theorem}


\begin{proof}(Sketch)
The proof (detailed in Appendices \ref{sec:app-er-homo} and \ref{app:er-heterogeneous}) budgets the failure probability $\delta$ across three key steps. First, the exploration coverage analysis ensures $T_0$ is sufficient to gather $s_0 \triangleq O(\log(n/\delta)/\Gap^2)$ samples from each arm. As mentioned above, the homogeneous case (Appendix \ref{sec:app-er-homo}) uses a regeneration argument on the expected kernel $\bar{W}$, while the general heterogeneous case (Appendix \ref{app:er-heterogeneous}) uses a uniform typicality argument on the realized kernels $W_t$. Both methods show that an exploration time scaling linearly with $n$ is sufficient.

Second, since the number of samples for every arm are atleast $s_0$, the identification step is a standard application of Hoeffding's inequality and a union bound (Lemma \ref{lem:er-identification}), guaranteeing $a^*$ is correctly identified with high probability.

Third, the navigation cost $T_{\mathrm{nav}}$ (the $c_3 n \log(n/\delta)$ term) bounds the time to reach $a^*$ using the \emph{lazy random walk} policy from Section \ref{sec:algorithm}. This ``mixing route" (Lemma \ref{lem:hetero-navigation}) bounds the time it takes for the walk to mix across the graph and hit the target $a^*$, which is $O(n \log(n/\delta))$ and succeeds as long as the expected graph is connected.
\end{proof}
\vspace{-0.5em}
\begin{remark}[General Analysis and the Homogeneous Case]
\label{rem:er-hetero-recovery}
The clean $O(n)$ scaling in Theorem \ref{thm:main-regret-er} is an asymptotic result of the analysis for the heterogenous ER case (Appendix \ref{app:er-heterogeneous}), where the exploration cost scales with the ``effective graph size" $\nEff$. For homogeneous graphs, $\nEff$ specializes to $\Theta(n)$ (Remark \ref{rem:het-consistency-check}), recovering the linear dependence on the number of arms.
\end{remark}

\subsection{Regret for FMAB with Edge-Markovian Graphs}
\label{sec:markovian-results}
We next analyze the algorithm under Edge-Markovian evolution, where the graph $G_t$ evolves from its preceding configuration $G_{t-1}$. This setting introduces a fundamental analytical challenge: the agent's walk and the environment's evolution are coupled stochastic processes. Because the graph evolves as the learner moves, the target stationary distribution $\pi_t$ becomes a moving target.


\paragraph{The ``Two Speeds" Challenge.} Our analysis (detailed in Appendix \ref{sec:app-markovian}) reveals that successful learning is possible only if the learner's walk mixes faster than the graph's structure drifts. We characterize this requirement through two competing quantities: the (i) walk speed $\gamma_t$, which admits a uniform lower bound $\gamma_0 = \Omega(1)$ under typicality (Lemma \ref{lem:cheeger-gamma}), and the (ii) drift speed $\varepsilon_{\max}$ (Lemma \ref{lem:tv-drift}), which quantifies the maximum per-step variation in the stationary law $\pi_t$. To ensure the agent can visit all $n$ arms before the topology re-wires, the graph must be sufficiently sticky. This intuition is formalized by our stickiness condition (Corollary \ref{cor:slow-flip-zeta}), which reveals that the edge disappearance rate must scale with the graph size such that $\beta \le O(1/n)$.

\paragraph{Burn-in and Exploration Phases.}
This analysis introduces two distinct requirements for the exploration phase of our algorithm.
First, the analysis requires the graph process to be in its statistically stationary state (the homogeneous ER model $\Psi^{\mathrm{ER}}_{\hom}(n, p_\infty)$ where $p_\infty = \alpha / (\alpha + \beta)$). This requires an initial burn-in period of $T_{\mathrm{burn}} = O(\log(nT/\delta) / (\alpha + \beta))$ steps (Lemma \ref{lem:edge-stationary}).
Second, after burn-in, the learner must continue exploring for an additional sample-gathering period, $T_{\mathrm{exp}}$, to collect sufficient samples (Theorem \ref{thm:Texp-lower-bound}).
Therefore, the two-phase algorithm from Section \ref{sec:algorithm} is adapted for this setting: the total exploration phase length $T_0$ must cover both periods, i.e., $T_0 \ge T_{\mathrm{burn}} + T_{\mathrm{exp}}$. The regret bound in Theorem \ref{thm:main-regret-markovian} reflects this structure, with the total regret being the sum of costs from these distinct phases.
\vspace{-0.5em}
\begin{theorem}[Regret for Edge-Markovian Graphs]
\label{thm:main-regret-markovian}
Under the typicality (Lemma \ref{lem:typicality}) and stickiness ($\beta \le O(1/n)$) conditions, for any $\delta \in (0,1)$, there exist constants $c_i > 0$ such that the two-phase algorithm with total exploration duration $T_0 = T_{\mathrm{burn}} + T_{\mathrm{exp}}$ satisfies, with probability at least $1-\delta$:
\[
R(T) \le \underbrace{c_0 \frac{\log(nT/\delta)}{\alpha+\beta}}_{T_{\mathrm{burn}} \text{ (Burn-in)}} + \underbrace{c_1 n \log(nT/\delta) + c_2 \frac{n \log(n/\delta)}{\Gap^2}}_{T_{\mathrm{exp}} \text{ (Exploration)}} + \underbrace{c_3 n \log(n/\delta)}_{T_{\mathrm{nav}} \text{ (Navigation)}}.
\]
\end{theorem}


\begin{proof}(Sketch)
The total regret (Theorem \ref{thm:final-regret-hp}) is the sum of the regret from the three distinct phases of the agent's timeline: burn-in, exploration, and navigation.

First, the agent incurs regret for $T_{\mathrm{burn}}$ steps while the graph process mixes to its stationary measure (Lemma \ref{lem:edge-stationary}).

Second, the agent incurs regret during the $T_{\mathrm{exp}}$ steps required for sample gathering. The stickiness condition is crucial here. It ensures the graph's drift $\varepsilon_{\max}$ is slow enough relative to the walk's mixing speed $\gamma_0$ to guarantee a positive \emph{effective visitation probability} $\pi_{\mathrm{eff}} \triangleq \pi_0 - 2\varepsilon_{\max}/\gamma_0 = \Omega(1/n)$ (Corollary \ref{cor:slow-flip-zeta}). This positive rate, combined with a martingale concentration bound (Lemma \ref{lem:uniform-visitation}), ensures $T_{\mathrm{exp}}$ is sufficient to gather $s_0 \triangleq O(\log(n/\delta)/\Gap^2)$ samples for correct identification.

Third, the \emph{navigation cost} $T_{\mathrm{nav}}$ (the $c_3$ term) bounds the time to reach $a^*$ using the lazy random walk policy. On the typical graph, the walk has a constant gap $\gamma_0$ and $\pi_t(a^*) = \Omega(1/n)$. The ``mixing route" analysis (Lemma \ref{lem:navigation-delta}) shows the total time to hit $a^*$ is bounded by $O(n \log(n/\delta))$.
\end{proof}
\vspace{-1em}
\subsection{Expected Regret for FMAB}
\label{sec:expected-regret}
Our high-probability bounds in Theorems \ref{thm:main-regret-er} and \ref{thm:main-regret-markovian} lead to a unified result for the expected regret. By converting the high-probability bound to an expectation, we obtain a clean asymptotic guarantee that is independent of the specific graph parameters.

\begin{corollary}[Expected Regret Bound]
\label{cor:fmab-expected-regret}
Under the assumptions of Theorem \ref{thm:main-regret-er} (for i.i.d. ER) or Theorem \ref{thm:main-regret-markovian} (for edge-Markovian), the expected cumulative regret $\E{R(T)}$ of the two-phase algorithm is bounded by:
\[
\E{R(T)} = O\left( \frac{n \log(nT)}{\Gap^2} \right).
\]
\end{corollary}
\begin{proof}(Sketch)
We use the law of total expectation: $\E{R(T)} = \E{R(T) | \mathcal{E}_{\text{success}}} (1-\delta) + \E{R(T) | \mathcal{E}_{\text{fail}}} \delta$. The regret on the failure event $\mathcal{E}_{\text{fail}}$ is at most $T$. A standard technique in regret analysis is to set the failure probability $\delta = 1/T$. This choice optimally balances the regret on the success event (which grows with $\log(1/\delta) = \log(T)$) against the maximum possible regret $T$ incurred on the failure event.

With this setting, the regret contribution from the failure event is bounded by $\E{R(T) | \mathcal{E}_{\text{fail}}}\delta \le T \cdot (1/T) = 1$. The regret on the success event is bounded by the high-probability costs from our theorems. When $\delta=1/T$, all mixing, navigation, and burn-in costs (which scale as $O(n \log(nT))$ or $O(\log(nT)/(\alpha+\beta))$) become asymptotically smaller than the identification cost component of $T_0$. The final bound is therefore dominated by this identification cost, which is $O(n \log(nT)/\Gap^2)$ (see Corollaries \ref{cor:er-expected-regret} and \ref{cor:expected-regret}).
\end{proof}

\begin{remark}[Near-Optimality of Exploration]
\label{rem:near-optimality}
The expected regret is dominated by the identification cost, as the additive terms for burn-in and navigation are lower-order. For any $\Delta_{\min} \in (0,1]$, the total regret simplifies to:
\[
\mathbb{E}[R(T)] \le O\left(\frac{n \log(nT)}{\Delta_{\min}^2}\right) + O(n \log(nT)) = O\left(\frac{n \log(nT)}{\Delta_{\min}^2}\right).
\]
This performance matches the fundamental limits of the problem class. In Appendix \ref{app:intrinsic-hardness}, we establish an information-theoretic lower bound (Theorem \ref{thm:lb-identification-time-app}) showing that for equal-gap instances, the expected identification time $T_{\mathrm{ID}}(\delta)$ must satisfy:
\[
\mathbb{E}[T_{\mathrm{ID}}(\delta)] \ge \Omega\left( \frac{(n-1) \mathrm{kl}(1-\delta, \delta)}{\Delta_{\min}^2} \right).
\]
By setting $\delta = 1/T$, we recover a fundamental hardness of $\Omega(n \log(T) / \Delta_{\min}^2)$. Thus, our two-phase approach achieves a \textbf{near-optimal exploration cost}, matching the lower bound \textbf{up to logarithmic factors}. Furthermore, the $\Omega(n)$ worst-case traversal time (Lemma \ref{lem:traversal-app}) is fully subsumed by this statistical limit and does not represent an additional bottleneck.
\end{remark}

This result brings the analysis full circle. While our high-probability analysis necessarily depends on the internal graph parameters used to model the availability constraints, the final \emph{expected} regret depends only on the problem's minimal parameters: the number of arms $n$, the time $T$, and the reward gap $\Gap$.

\section{Numerical Simulations}
\label{sec:num-simulations}

We validate our theoretical results through controlled experiments and a simulated robotic deployment scenario. 

\paragraph{Theoretical Validation.} We conduct two sets of experiments to verify our regret and navigation cost bounds. First, we simulate both i.i.d. ER and Edge-Markovian models with $T=10,000$ and $p=0.5$ for graphs of size $n \in \{10, 50, 100\}$. By scaling $T_0$ with $n$ and $\Delta_{\min}$ as theoretically prescribed, Figures~\ref{fig:exa}(a) and (b) show the average regret $R(t)/t$ converging toward zero, empirically confirming our sublinear guarantees. Second, we isolate the Phase II navigation cost in Figure~\ref{fig:exa}(c) by fixing $T_0=5,000$ to decouple movement from identification. Plotting navigation time $T_{\mathrm{nav}}$ against sparsity $p$ for the i.i.d. ER model reveals a clear monotonic trend where median values rise from $\approx 21$ at $p=0.8$ to $\approx 120$ at $p=0.01$. These results validate the $O(1/p)$ navigation cost component of our regret bound.

\begin{figure}[t]
  \centering
  \begin{minipage}[t]{0.33\linewidth}
    \centering
    \includegraphics[width=\linewidth]{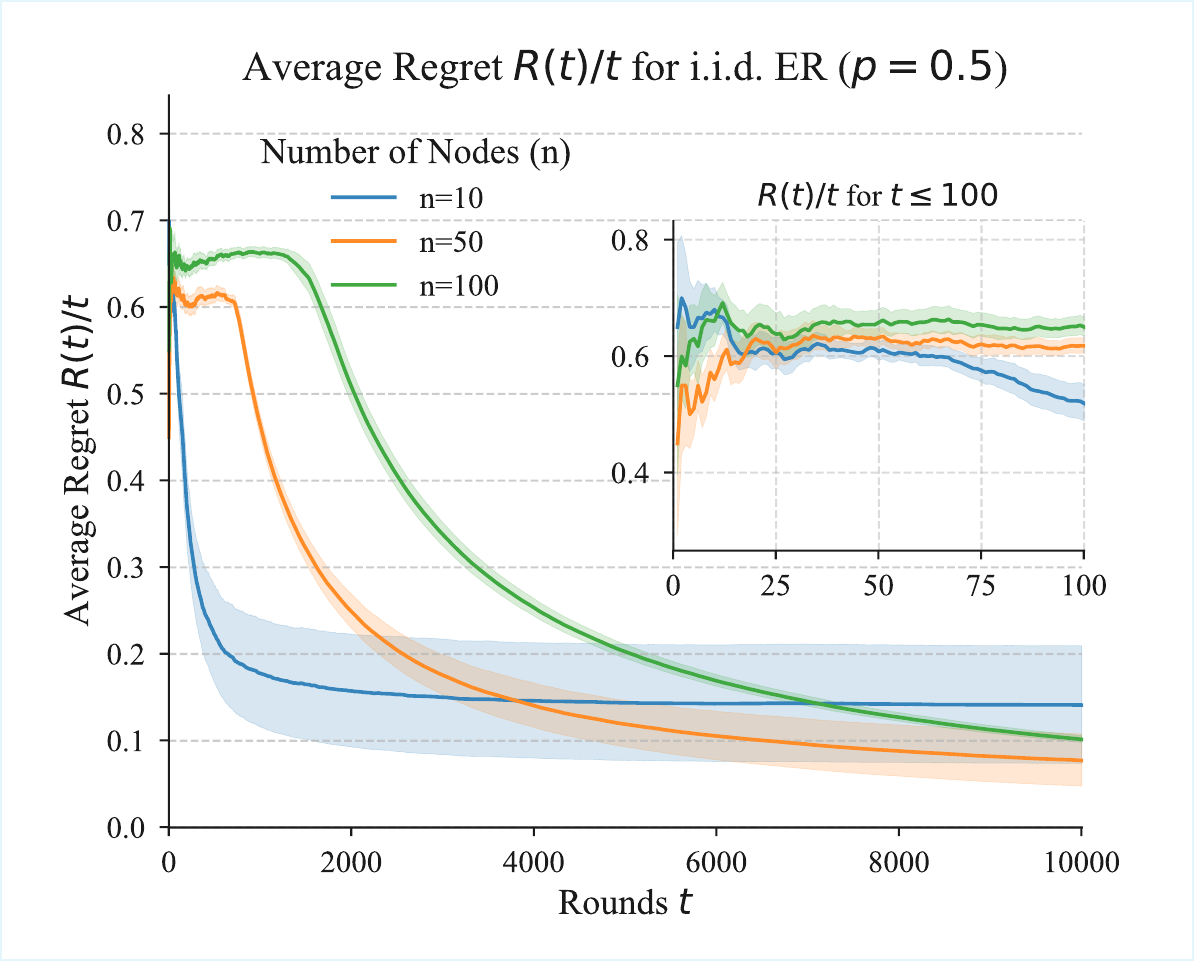}\\[2pt]
    (a)
    \label{fig:exa:a}
  \end{minipage}\hfill
  \begin{minipage}[t]{0.33\linewidth}
    \centering
    \includegraphics[width=\linewidth]{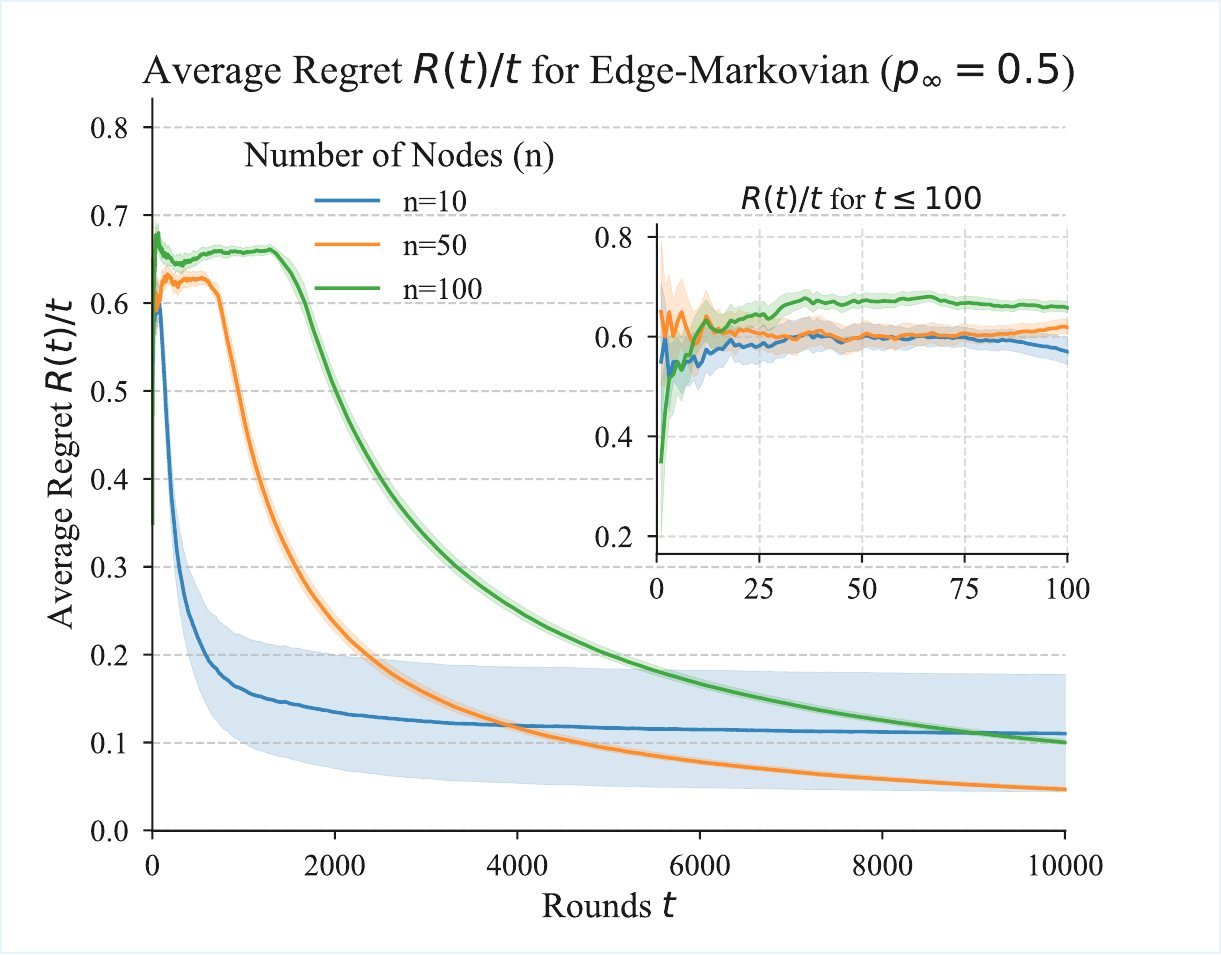}\\[2pt]
    (b) 
    \label{fig:exa:b}
  \end{minipage}
  \begin{minipage}[t]{0.33\linewidth}
    \centering
    \includegraphics[width=\linewidth]{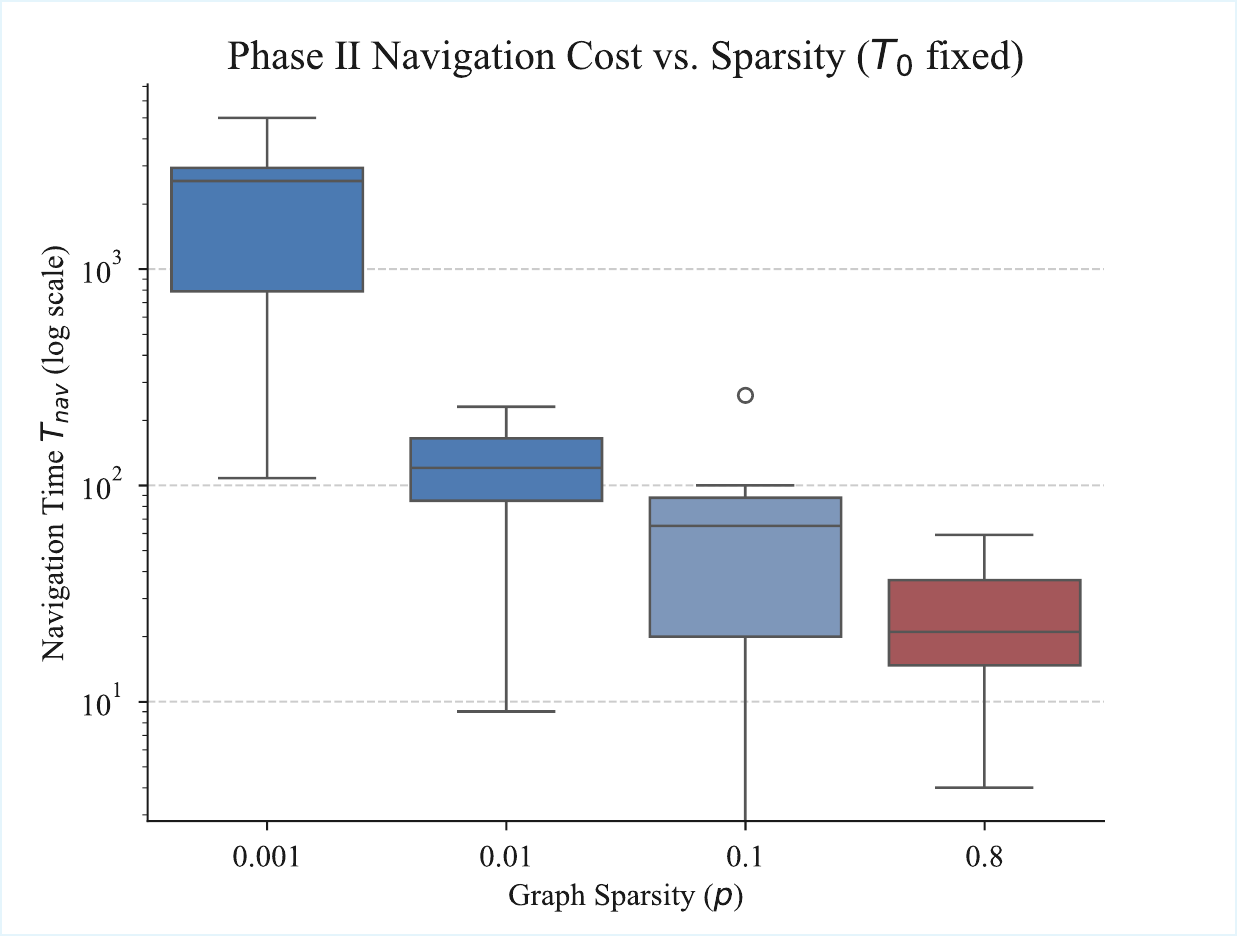}\\[2pt]
    (c)
    \label{fig:exa:b}
  \end{minipage}
  \caption{(a) Average cumulative regret \( R(t)/t \) over time for i.i.d ER case. The main plot and the inset (showing the initial $t \le 200$ rounds) demonstrate the algorithm's rapid convergence. (b) Same plot for the Edge-Markovian Case. (c) Box plot for the navigation cost for ER case with respect to the sparsity.}
  \label{fig:exa}
\end{figure}

\subsection{Application: Disaster Response Scenario}
The goal is to locate and repeatedly assess high-utility candidate locations that offer strong and reliable signal coverage. These locations are unknown, their pathways are unstable, and they must be discovered through repeated interaction with the environment. The vehicle operates without centralized coordination, relying only on local sensing and movement. Repeated visitation is essential for building a reliable estimate of signal quality and stability before committing to a final deployment location.

The physical region spans five square kilometers and is discretized into \( n = 500 \) candidate locations (e.g., street intersections or clearings). This discretization is an approximate representation of potential deployment sites in an urban map. At any given time-stamp $t$, the discretized sites (or locations) defines the node set $A$ of a random graph $G_t$, where the navigable streets between adjacent locations correspond to the edges $E_t$. These connections evolve stochastically over time due to debris, structural damage, or diverted traffic.

Each location \( a \in A \) has an unknown reward \( \mu(a) \in [0,1] \), representing the quality and stability of communication coverage the site could provide. When the vehicle visits a location \( a_t \), it takes a measurement and receives a binary reward \( r_t(a_t) \sim \mathrm{Bernoulli}(\mu(a_t)) \) (e.g., `signal strength beyond a threshold' or not), with no feedback from neighboring zones. The environment (i.e., pathway access) evolves independently of the vehicle’s actions. The reward structure is sparse and clustered. A single location is designated as the hotspot with \( \mu(a^*) = 0.95 \) (i.e., a location with the best coverage), ten others have moderate rewards drawn from \( [0.45, 0.65] \) (i.e., acceptable but partially obstructed sites), and the remaining 489 lie in \( [0.1, 0.4] \) (i.e., in low-lying areas or ``urban canyons"). These values reflect deployment scenarios, where only a few locations are suitable for a high-bandwidth relay, while most are non-viable.

The vehicle runs the natural lazy walk (Section \ref{sec:algorithm}), fully embedded in its onboard control stack. A discretized map, node structure, and initial connectivity are preloaded before deployment. During operation, the vehicle localizes itself using GPS/LIDAR and maps real-time observations, such as blocked streets or accessible pathways, to its current neighborhood in the evolving graph. All decisions are made locally, without access to a global map or external planner. The algorithm runs continuously over the entire duration of a two-week mission (\( T = 80{,}640 \) rounds). A decision interval of 15 seconds (e.g., the time to move between intersections or take a stable signal measurement) allowing for adaptive, long-term learning and prioritization of high-value areas in real time.

\paragraph{Performance Results.}
We simulate our two-phase algorithm under the Edge-Markovian graph evolution, with parameters $(\alpha = 0.01,\beta=0.03)$ set to match the scenario's stationary density. Since our analysis shows that the algorithm's long-term convergence is independent of the initial graph $G_0$ (see Lemma \ref{lem:edge-stationary}), we start with a randomly generated arbitrary graph.

  

\begin{figure}[t]
  \centering
  \begin{minipage}[b]{0.23\linewidth}
    \centering
    \includegraphics[width=\linewidth]{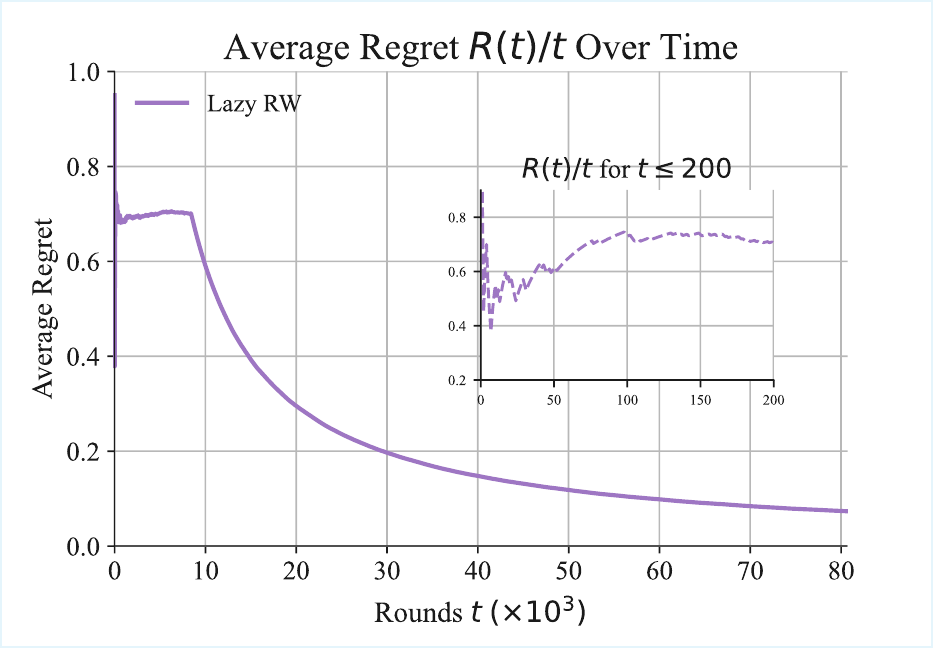}\\
    \vspace{6pt}
    (a)
  \end{minipage}
  \begin{minipage}[b]{0.75\linewidth}
    \centering
    \includegraphics[width=\linewidth]{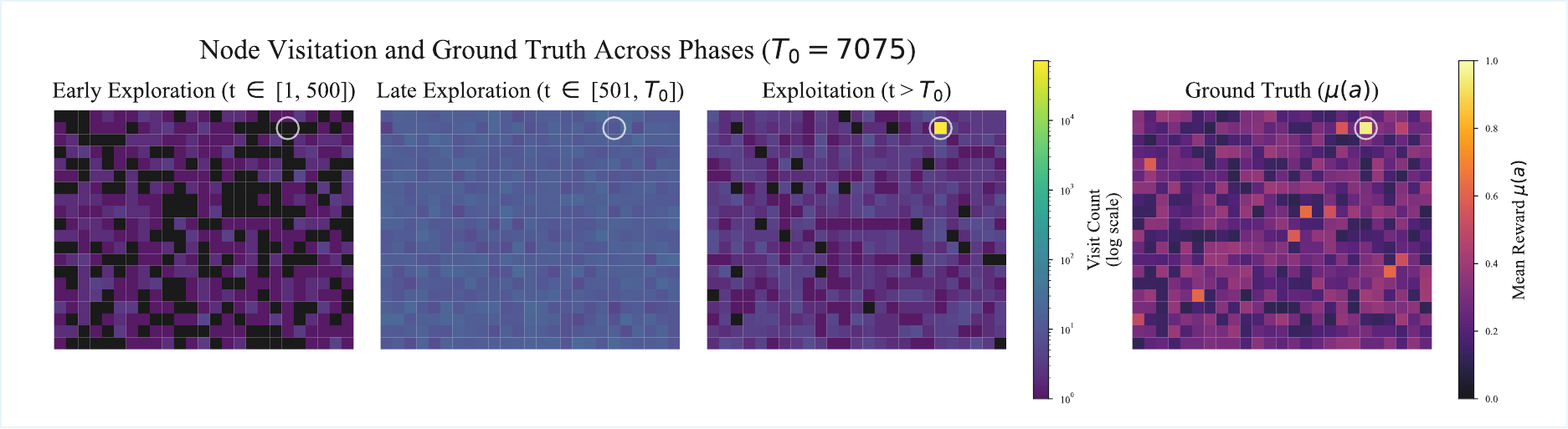}\\
    (b)
  \end{minipage}
  \caption{(a) Average cumulative regret $R(t)/t$ for the disaster response scenario. (b) Spatial visitation density (log-scale) illustrating the behavioral transition from the exploration and exploitation rounds.}
  \label{fig:combined_sim}
\end{figure}


Figure \ref{fig:combined_sim}(a) plots the average regret $R(t)/t$. The algorithm converges rapidly, demonstrating its ability to efficiently manage the exploration–exploitation trade-off under evolving connectivity and local movement. Figure \ref{fig:combined_sim}(b) presents spatial visitation heatmaps as separate histograms for each mission phase rather than cumulative sums. This captures the agent's behavioral evolution: initial sparse wandering (Rounds 1–500) transitions to a diffuse “blue cloud” during Late Exploration (Rounds 501–$T_0$). This uniform coverage confirms the lazy random walk has mixed sufficiently to gather the unbiased samples required to compute $\hat a^*$. Post-identification ($t > T_0$), the visitation becomes sparse as the vehicle executes the navigation path and repeatedly pulls $\hat a^*$. The “Ground Truth” panel visually confirms that this committed arm matches the true optimal arm $a^*$.

\section{Conclusion}
\label{sec:conclusion}

We introduce the Flickering Multi-Armed Bandit (FMAB) framework, a novel sequential decision-making model where action accessibility is strictly constrained by the localized, time-varying neighborhood structure of the action space. We characterize this dynamic accessibility through two canonical graph processes: i.i.d. Erdős–Rényi and Edge-Markovian evolution. Our theoretical analysis demonstrates that a two-phase strategy utilizing lazy random walks is inherently robust to these connectivity constraints and achieves near-optimal exploration performance across diverse regimes. Future work involves developing adaptive, anytime algorithms and extending the FMAB framework to non-stationary rewards or multi-agent systems. Ultimately, this work establishes a rigorous foundation for analyzing the fundamental tension between information acquisition and physical mobility in dynamic environments.

\bibliography{refs}

\clearpage
\appendix
\addcontentsline{toc}{section}{Appendix}
\vspace{-1.2em}
\begingroup
\setcounter{tocdepth}{3} 
\renewcommand{\contentsname}{Appendix Contents}
\begin{center}
{\small
\begin{minipage}{0.95\linewidth}
\tableofcontents
\end{minipage}
}
\end{center}
\endgroup
\vspace{-1.5em}


\clearpage
\section{Regret Analysis: Homogeneous Erd\H{o}s--R\'enyi (ER) Graphs}
\label{sec:app-er-homo}
\subsection{Model, Notations and Graph Dynamics}\label{sec:er-gating}

\subsubsection{Problem Setup and i.i.d. Erd\H{o}s--R\'enyi Graph Process}
\label{sec:er-problem-setup}

We begin by specializing the problem formulation from Section \ref{sec:problem-formulation} to the i.i.d. Erd\H{o}s--R\'enyi (ER) environment. This section establishes the core notation for this model and formalizes its basic probabilistic properties, which serve as a foundation for the subsequent analysis.

\paragraph{i.i.d. Erd\H{o}s--R\'enyi Graph Process.}
In this model, the environment's availability structure is encoded by a sequence of graphs $\{G_t\}_{t=1}^T$. At the beginning of each round $t$, a new undirected graph $G_t=(A,E_t)$ is drawn independently from the homogenous ER distribution $\Psi^{\mathrm{ER}}_0(n, p)$, independent of all past graphs and actions.

The learner, residing at the arm $a_{t-1}$ chosen in the previous round, observes the set of available arms $L_t(a_{t-1})$. This set is defined by the local neighborhood in the newly drawn graph $G_t$:
$$
L_t(a_{t-1}) \triangleq \{\,a_{t-1}\,\}\cup \{\,j\in A\setminus\{a_{t-1}\} : \{a_{t-1},j\}\in E_t\,\}.
$$
The learner then selects a new arm $a_t \in L_t(a_{t-1})$ and receives a reward $r_t(a_t) \sim \mathcal{D}(a_t)$.

For any arm $a \in A$, its degree in $G_t$ is a random variable $d_t(a) \sim \mathrm{Bin}(n-1, p)$. The size of the available set from $a$ is thus also a random variable, $|L_t(a)| = 1 + d_t(a)$, whose distribution is $1+\mathrm{Bin}(n-1,p)$ and whose expected size is $\mathbb{E}[|L_t(a)|] = 1 + (n-1)p$. Because $G_t$ is drawn independently at each step, this information is independent of the past history $\mathcal{F}_{t-1}$.

The next lemma formalizes the probability that any specific arm $j$ becomes available to the learner, a property that holds uniformly over all (possibly adaptive) learning strategies.

\begin{lemma}[ER Availability]
\label{lem:er-availability}
Fix any arm $j\in A$. For every round $t\ge 1$ and any history $\mathcal{F}_{t-1}$ up to time $t-1$, the conditional probability that $j$ is available at round $t$ (given the learner is at $a_{t-1}$) satisfies
$$
\mathbb{P}\!\left(j\in L_t(a_{t-1}) \,\middle|\, \mathcal{F}_{t-1}\right)
= \mathbf{1}\{a_{t-1}=j\} + p\,\mathbf{1}\{a_{t-1}\neq j\}. \text{}
$$
Consequently, the expected number of times $j$ is available over $T$ rounds is bounded below:
$$
\mathbb{E}\Big[\sum_{t=1}^T \mathbf{1}\{j\in L_t(a_{t-1})\}\Big] \ge pT \qquad \text{for all }T\ge 1. \text{}
$$
\end{lemma}
\begin{proof}
Fix any arm $j \in A$ and round $t \ge 1$. We seek to compute the conditional probability $\mathbb{P}(j \in L_t(a_{t-1}) \mid \mathcal{F}_{t-1})$. Conditioned on $\mathcal{F}_{t-1}$, the learner's position $a_{t-1}$ is a fixed and known arm.

We analyze this probability by cases based on the learner's position $a_{t-1}$:
\begin{itemize}
    \item \textbf{Case 1: $a_{t-1} = j$.} By the definition $L_t(j) \triangleq \{j\} \cup \{\dots\}$, the learner can always choose to stay at their current arm. Thus, $j \in L_t(j)$ deterministically, and $\mathbb{P}(j \in L_t(a_{t-1}) \mid \mathcal{F}_{t-1}) = 1$.
    
    \item \textbf{Case 2: $a_{t-1} \neq j$.} In this case, the arm $j$ is available, $j \in L_t(a_{t-1})$, if and only if the edge $\{a_{t-1}, j\}$ is present in the graph $G_t$. Since $G_t$ is drawn independently from $G(n,p)$ and is independent of $\mathcal{F}_{t-1}$, this event occurs with probability $p$.
\end{itemize}
Combining these two mutually exclusive cases yields the first statement in the lemma:
\[
\mathbb{P}(j \in L_t(a_{t-1}) \mid \mathcal{F}_{t-1}) = \mathbf{1}\{a_{t-1}=j\} \cdot 1 + \mathbf{1}\{a_{t-1}\neq j\} \cdot p.
\]
For the second statement, we apply the law of total expectation to the sum:
\begin{align*}
\mathbb{E}\Big[\sum_{t=1}^T \mathbf{1}\{j\in L_t(a_{t-1})\}\Big]
&= \sum_{t=1}^T \mathbb{E}\Big[\mathbb{E}\big[\mathbf{1}\{j\in L_t(a_{t-1})\} \mid \mathcal{F}_{t-1}\big]\Big] \\
&= \sum_{t=1}^T \mathbb{E}\Big[\mathbb{P}(j\in L_t(a_{t-1}) \mid \mathcal{F}_{t-1})\Big].
\end{align*}
From the first part, the inner conditional probability is a random variable that always takes a value in $\{1, p\}$. Since $p \in (0,1)$, this value is always greater than or equal to $p$. Therefore:
\[
\mathbb{E}\Big[\mathbb{P}(j\in L_t(a_{t-1}) \mid \mathcal{F}_{t-1})\Big] \ge \mathbb{E}[p] = p.
\]
Summing this over $T$ rounds gives the claimed bound: $\sum_{t=1}^T p = pT$.

Finally, for the last part of the lemma, the degree $d_t(a)$ for any arm $a$ is the sum of $n-1$ independent $\mathrm{Bernoulli}(p)$ trials (one for each other arm). Thus, $d_t(a) \sim \mathrm{Bin}(n-1,p)$, and the available set size is $|L_t(a)| = 1 + d_t(a)$. As $G_t$ is independent of $\mathcal{F}_{t-1}$, so are $d_t(a)$ and $|L_t(a)|$. This completes the proof.
\end{proof}

The preceding lemma's i.i.d. availability property directly implies a simple bound on the waiting time to access a specific arm from a fixed position. This quantity will be critical for bounding the cost of navigating to the optimal arm after identification.

\begin{corollary}[Waiting Time to Availability]
\label{cor:er-geometric}
Fix two distinct arms $a, j \in A$. Suppose the learner adopts a ``wait-at-$a$" strategy, holding its position at arm $a$ until arm $j$ appears in the available set $L_t(a)$. The waiting time $\tau_{a \to j}$ until $j$ is first available is a geometric random variable with success parameter $p$, satisfying $\mathbb{E}[\tau_{a \to j}] = 1/p$.
\end{corollary}

\begin{proof}
Fix distinct arms $a, j \in A$. Under the ``wait-at-$a$" strategy, the learner's position is fixed at $a_{t-1} = a$ for all rounds $t$ prior to the move. By Lemma \ref{lem:er-availability}, since $a \neq j$, the event $j \in L_t(a)$ occurs at each round $t$ with probability $p$. Because the graphs $G_t$ are drawn i.i.d. in each round, these availability events are independent $\mathrm{Bernoulli}(p)$ trials across $t$. The waiting time $\tau_{a \to j}$ is therefore the time of the first success in this sequence, which is, by definition, a geometric random variable with parameter $p$. Its expectation is $\mathbb{E}[\tau_{a \to j}] = 1/p$.
\end{proof}

\subsubsection{Uniform-Move Policy: Dynamics and Properties}
\label{sec:er-walk-dynamics}

As part of the algorithm design (see Sections \ref{sec:algorithm} and \ref{sec:performance-guarantees}), the learner's movement during the exploration phase is modeled as a \emph{uniform-move policy}. At each exploration round $t$, given the learner is at $a_{t-1}$, a new graph $G_t \sim \Psi^{\mathrm{ER}}_0(n, p)$ is revealed, and the learner selects the next arm $a_t$ uniformly at random from the available set $L_t(a_{t-1})$.

We now analyze the properties of this uniform-move process, as this analysis is central to quantifying the rate at which the learner visits all arms during exploration. Because the graph $G_t$ is drawn i.i.d. at each step, the transition probability $\mathbb{P}(a_t = j \mid a_{t-1} = a)$ depends only on $a$ and $j$, not on the time $t$. The learner's location sequence $(a_t)_{t \ge 0}$ under this policy is therefore a \emph{time-homogeneous Markov chain}. This stands in contrast to the time-inhomogeneous chain induced by the temporally correlated Markovian model. We denote the transition matrix of this chain by $W$ and summarize its properties in the following theorem.

\begin{theorem}[Properties of the Uniform-Move Kernel]
\label{thm:uniform-move-summary}
Consider the time-homogeneous Markov chain on $A$ induced by the uniform-move policy, where at each round $t$, the environment samples $G_t \sim \Psi^{\mathrm{ER}}_0(n, p)$ and the learner transitions from $a_{t-1}$ to $a_t$ by sampling uniformly from $L_t(a_{t-1})$. Let $W$ be its transition matrix. The chain and its properties are as follows:
\begin{enumerate}[label=(\roman*), leftmargin=2.3em]
    \item \textbf{Transition Probabilities:} The random-walk matrix $W$ is symmetric with entries
    \[
    W(a,a)=\mathbb{E}\!\left[\frac{1}{1+d_t(a)}\right]=\frac{1-(1-p)^n}{np}
    \quad\text{and}\quad
    W(a,j)=\frac{np-1+(1-p)^n}{np(n-1)}\quad \text{for }a\neq j,
    \]
    where the expectation is over $d_t(a) \sim \mathrm{Bin}(n-1,p)$.
    
    \item \textbf{Ergodicity:} The chain is irreducible and aperiodic.
    
    \item \textbf{Stationarity:} The chain is reversible with respect to the uniform distribution $\pi(a) = 1/n$, which is its unique stationary distribution.
    
    \item \textbf{Spectrum and Spectral Gap:} The eigenvalues of $W$ are $\lambda_1 = 1$ (with multiplicity one) and $\lambda_2=\cdots=\lambda_n = \frac{1-(1-p)^n-p}{p(n-1)}$. The spectral gap is $\gamma \triangleq 1-\lambda_2 > 0$.
    
    \item \textbf{Mixing Time:} For every $\varepsilon\in(0,1)$, the total-variation mixing time satisfies $t_{\mathrm{mix}}(\varepsilon)\le \frac{\log(n/\varepsilon)}{\gamma}$.
    
    \item \textbf{Uniform Gap Bound:} In the dense regime ($p\in(0,1)$ fixed) and the standard sparse regime ($p=c/n$ for $c>0$), the spectral gap $\gamma$ is bounded below by a positive absolute constant $\gamma_0 \triangleq \Omega(1)$, and the mixing time is $t_{\mathrm{mix}}(\varepsilon)=O(\log n)$.
\end{enumerate}
\end{theorem}

\begin{proof}
We analyze the properties of the time-homogeneous Markov chain $W$ in a sequence of steps.

\paragraph{Step 1. Transition Probabilities.}
First, consider the diagonal entries $W(a,a) = \mathbb{P}(a_t = a \mid a_{t-1} = a)$. Given $a_{t-1}=a$, the environment draws $G_t \sim \Psi^{\mathrm{ER}}_0(n, p)$. The learner is at arm $a$, and the available set is $L_t(a)$, which has size $|L_t(a)| = 1 + d_t(a)$, where $d_t(a) \sim \mathrm{Bin}(n-1, p)$. The learner stays at $a$ by selecting it uniformly from $L_t(a)$, which occurs with probability $1 / |L_t(a)|$. To find $W(a,a)$, we take the expectation over the randomness of $G_t$:
\[
W(a,a) = \mathbb{E}_{G_t \sim \Psi^{\mathrm{ER}}_0(n, p)}\!\left[\frac{1}{1+d_t(a)}\right].
\]
Let $m=n-1$ and $D \sim \mathrm{Bin}(m,p)$. We compute this expectation in closed form using the identity $\frac{1}{k+1}\binom{m}{k} = \frac{1}{m+1}\binom{m+1}{k+1}$:
\begin{align*}
\mathbb{E}\!\left[\frac{1}{1+D}\right]
&= \sum_{k=0}^{m} \frac{1}{k+1}\binom{m}{k}p^k(1-p)^{m-k}
= \frac{1}{m+1}\sum_{k=0}^{m} \binom{m+1}{k+1}p^k(1-p)^{m-k} \\
&= \frac{1}{(m+1)p}\sum_{l=1}^{m+1} \binom{m+1}{l}p^l(1-p)^{(m+1)-l} \quad (\text{letting } l=k+1) \\
&= \frac{1}{(m+1)p} \left[ \sum_{l=0}^{m+1} \binom{m+1}{l}p^l(1-p)^{(m+1)-l} - (1-p)^{m+1} \right] \\
&= \frac{1-(1-p)^{m+1}}{(m+1)p} = \frac{1-(1-p)^n}{np}.
\end{align*}
Next, consider the off-diagonal entries $W(a,j) = \mathbb{P}(a_t = j \mid a_{t-1} = a)$ for $a \neq j$. By the symmetry of the $G(n,p)$ model, this probability is identical for all $j \neq a$. Let this common value be $C$. Since $W$ is a row-stochastic matrix, the row sums must be 1:
\[
W(a,a) + \sum_{j \neq a} W(a,j) = 1 \implies W(a,a) + (n-1)C = 1.
\]
Solving for $C$ and substituting the expression for $W(a,a)$ (which we denote $D$ for simplicity in this step) gives:
\[
W(a,j) = C = \frac{1-W(a,a)}{n-1} = \frac{1 - \frac{1-(1-p)^n}{np}}{n-1} = \frac{np-1+(1-p)^n}{np(n-1)}.
\]

\paragraph{Step 2. Irreducibility and Aperiodicity.}
We show that $W(a,j) > 0$ for all $a, j \in A$.
The diagonal entry $W(a,a) = \frac{1-(1-p)^n}{np} > 0$ since $p \in (0,1)$.
For the off-diagonal entry $W(a,j)$ (where $a \neq j$), we analyze the numerator $f(p) = np-1+(1-p)^n$ for $p \in (0,1]$. Note $f(0)=0$. The derivative is
\[
f'(p) = n - n(1-p)^{n-1} = n\big(1-(1-p)^{n-1}\big).
\]
Since $n \ge 2$ and $p \in (0,1)$, we have $0 < (1-p)^{n-1} < 1$, which implies $f'(p) > 0$. Thus, $f(p)$ is strictly increasing on $[0,1]$ from $f(0)=0$, and $f(p) > 0$ for all $p \in (0,1]$. This proves $W(a,j) > 0$ for all $a \neq j$.
Since all states have positive self-loops ($W(a,a) > 0$), the chain is aperiodic. Since all off-diagonal entries are positive ($W(a,j) > 0$), the chain is irreducible (and in fact, fully connected).

\paragraph{Step 3. Stationarity and Reversibility.}
From the formulas derived in Step 1, we observe that $W(a,a) = W(j,j)$ and $W(a,j) = W(j,a)$ for all $a, j \in A$. The transition matrix $W$ is therefore symmetric ($W = W^T$). A symmetric, row-stochastic matrix is necessarily doubly stochastic. This implies that the uniform distribution $\pi(a) = 1/n$ for all $a \in A$ is the unique stationary distribution, since $\pi W = \pi$ is satisfied.
Furthermore, the chain is reversible with respect to $\pi$. We verify the detailed balance equations:
\[
\pi(a)W(a,j) = \frac{1}{n} W(a,j) = \frac{1}{n} W(j,a) = \pi(j)W(j,a),
\]
which holds precisely because $W$ is symmetric.

\paragraph{Step 4. Spectrum and Spectral Gap.}
The transition matrix $W$ has the form $W = C J + (D-C)I$, where $D=W(a,a)$, $C=W(a,j)$ ($a \neq j$), $I$ is the identity matrix, and $J$ is the all-ones matrix. The matrix $J$ has a well-known spectrum: one eigenvalue equal to $n$ (for the all-ones eigenvector) and $n-1$ eigenvalues equal to $0$.
The eigenvalues of $W$ are thus:
\begin{itemize}
    \item For the all-ones eigenvector: $\lambda_1 = C \cdot n + (D-C) \cdot 1 = 1$. This corresponds to the stationary distribution.
    \item For the $n-1$ other eigenvectors (which are orthogonal to the all-ones vector):
    $\lambda_k = C \cdot 0 + (D-C) \cdot 1 = D-C$ for $k=2, \dots, n$.
\end{itemize}
The second-largest eigenvalue (in magnitude) is $\lambda_2 = D-C$. We compute this value:
\[
\lambda_2 = D-C = \frac{1-(1-p)^n}{np} - \frac{np-1+(1-p)^n}{np(n-1)} = \frac{(n-1)(1-(1-p)^n) - (np-1+(1-p)^n)}{np(n-1)}.
\]
A simpler form is $\lambda_2 = \frac{1-(1-p)^n-p}{p(n-1)}$. Since $0 < \lambda_2 < 1$ for $p \in (0,1)$ and $n \ge 2$, the spectral gap is $\gamma \triangleq 1-\lambda_2 > 0$.

\paragraph{Step 5. Mixing Time Bound.}
Because $W$ is reversible with respect to $\pi$ (Step 3), it is self-adjoint in the $L_2(\pi)$ norm. For any initial distribution $\nu$, the $L_2(\pi)$-norm of the deviation from stationarity contracts by $\lambda_2$ at each step. This leads to the standard bound relating total variation (TV) distance to the spectral gap $\gamma = 1 - \lambda_2$:
\[
\|\nu W^t - \pi\|_{\mathrm{TV}} \le \frac{1}{2}\sqrt{n}\,\lambda_2^t.
\]
To ensure $\|\nu W^t - \pi\|_{\mathrm{TV}} \le \varepsilon$, it suffices to choose $t$ such that $t \ge \frac{\log(\sqrt{n}/(2\varepsilon))}{-\log \lambda_2}$. Using the inequality $-\log(1-\gamma) \ge \gamma$, we get the final bound:
\[
t_{\mathrm{mix}}(\varepsilon) \le \frac{\log(n/\varepsilon)}{-\log(1-\gamma)} \le \frac{\log(n/\varepsilon)}{\gamma}.
\]

\paragraph{Step 6. Asymptotics and Uniform Gap Bound.}
We analyze the gap $\gamma$ in two limiting regimes as $n \to \infty$:
\begin{itemize}
    \item \textbf{Dense regime ($p \in (0,1)$ fixed):} As $n \to \infty$, $(1-p)^n \to 0$. The second eigenvalue is $\lambda_2 = \frac{1 - (1-p)^n - p}{p(n-1)} \to 0$. Therefore, the spectral gap $\gamma = 1 - \lambda_2 \to 1$.
    \item \textbf{Sparse regime ($p = c/n$ for $c > 0$ fixed):} As $n \to \infty$, $(1-p)^n = (1-c/n)^n \to e^{-c}$. The numerator of $\lambda_2$ approaches $1 - e^{-c} - c/n \to 1 - e^{-c}$, while the denominator $p(n-1) = (c/n)(n-1) \to c$. Thus, $\lambda_2 \to \frac{1-e^{-c}}{c}$. The limiting spectral gap is $\gamma \to 1 - \frac{1-e^{-c}}{c} = \frac{c - 1 + e^{-c}}{c}$, which is strictly positive for $c > 0$.
\end{itemize}
In both standard regimes (fixed $p$ or $p=c/n$), the spectral gap $\gamma$ is bounded below by a positive absolute constant, $\gamma_0 \triangleq \Omega(1)$. This allows $\gamma$ to be treated as a constant in the final regret analysis, where factors of $1/\gamma$ can be absorbed into the $C_i$ constants.
\end{proof}

\begin{remark}[Implications of the Uniform-Move Kernel]
Theorem \ref{thm:uniform-move-summary} provides a crucial guarantee for the analysis of the i.i.d. ER environment. It shows that the simple uniform-move policy induces an exploration walk that mixes in $O(\log n)$ time, independent of the edge probability $p$ (so long as $p$ is not $o(1/n)$). This rapid mixing ensures that the agent can quickly achieve uniform coverage of all arms, forming the basis of our sample complexity bounds for the exploration phase.
\end{remark}

\subsection{Algorithm and Analysis Overview}
\label{sec:er-algo-overview}

We now formally define the two-phase algorithm analyzed for the i.i.d. ER environment. The algorithm operates in an exploration phase, followed by an exploitation phase. The analysis of this algorithm relies on the properties of the uniform-move kernel $W$ established in Theorem \ref{thm:uniform-move-summary}.

\paragraph{The Algorithm.}
The algorithm proceeds in two distinct phases:

\begin{itemize}[leftmargin=2.3em]
    \item \textbf{Phase I: Exploration.} For a predetermined exploration length of $T_{\exp}$ rounds ($t=1, \dots, T_{\exp}$), the learner follows the \emph{Uniform-Move Policy} (see Theorem \ref{thm:uniform-move-summary}). At each round $t$, the learner, residing at $a_{t-1}$, observes the newly drawn graph $G_t \sim \Psi^{\mathrm{ER}}_0(n, p)$ and selects its next action $a_t$ uniformly at random from the available set $L_t(a_{t-1})$.
    
    During this phase, the learner updates the running empirical mean, $\hat{\mu}_t(a)$, for each arm $a$ it plays. Let $\phi_{T_{\exp}}(a) \triangleq \sum_{t=1}^{T_{\exp}} \mathbf{1}\{a_t = a\}$ be the total visitation count for arm $a$ by the end of the phase. The final, frozen empirical mean used for exploitation is $\hat{\mu}(a) \triangleq \hat{\mu}_{T_{\exp}}(a)$, which is computed based on these $\phi_{T_{\exp}}(a)$ samples.

    \item \textbf{Phase II: Exploitation.} For all subsequent rounds $t > T_{\exp}$, the learner follows a greedy policy with respect to the \emph{frozen} estimates $\hat{\mu}(a)$ computed during exploration.
    
    At each round $t$, the learner, residing at $a_{t-1}$, observes the new graph $G_t \sim \Psi^{\mathrm{ER}}_0(n, p)$ and selects the arm with the highest empirical mean from the available set:
    \[
    a_t \in \arg\max_{j \in L_t(a_{t-1})} \hat{\mu}(j).
    \]
    We note that once the learner reaches the empirically best arm $a^*$ (assuming successful identification), it will remain there, as $a^*$ is always available from itself ($a^* \in L_t(a^*)$).
\end{itemize}

\paragraph{Analysis Roadmap and Probabilistic Framework.}
Our objective is to derive a high-probability, instance-dependent bound on the cumulative regret, $R(T)$. The analysis relies on bounding the cost of each phase separately. The total regret is bounded by the sum of the regret from the exploration phase (which is at most $T_{\exp}$) and the regret from the exploitation phase. The exploitation regret is dominated by the number of steps $T_{\text{nav}}$ required to navigate to and stay at the optimal arm $a^*$.

Our final bound will hold with probability at least $1-\delta$, for a total failure budget $\delta \in (0,1)$. We explicitly budget this failure probability across the three key probabilistic steps of the analysis:
\begin{enumerate}[label=(\roman*), leftmargin=2.3em]
    \item \textbf{Exploration Coverage:} Bounding the time $T_{\exp}$ required to ensure $\phi_{T_{\exp}}(a) \ge s_0$ (a sufficient sample count) for all arms $a \in A$. We budget a failure probability of $\delta_{\text{cov}} = \delta/3$.
    \item \textbf{Identification:} Ensuring that the empirical means $\hat{\mu}(a)$ correctly identify the optimal arm $a^*$ given the coverage event. We budget $\delta_{\text{id}} = \delta/3$.
    \item \textbf{Navigation:} Bounding the navigation time $T_{\text{nav}}$ during the exploitation phase, conditional on successful identification. We budget $\delta_{\text{nav}} = \delta/3$.
\end{enumerate}
By a union bound, the total failure probability of our analysis is at most $\delta_{\text{cov}} + \delta_{\text{id}} + \delta_{\text{nav}} = \delta$.

\subsection{Analysis of the Exploration Phase}
\label{sec:er-exploration-analysis}

We now analyze the exploration phase defined in Section \ref{sec:er-algo-overview}. The goal is to determine a sufficient exploration time $T_{\exp}$ to ensure the learner collects enough samples from every arm for a high-probability guarantee of correct identification.

\subsubsection{Coverage Guarantees via Regeneration}
\label{sec:er-coverage}

The first step is to find a $T_{\exp}$ that guarantees every arm $a \in A$ is visited a sufficient number of times, $\phi_{T_{\exp}}(a) \ge s_0$. Our analysis relies on the regenerative properties of the uniform-move kernel $W$.

From Theorem \ref{thm:uniform-move-summary}, we know $W$ is a time-homogeneous, reversible, and irreducible Markov chain with a uniform stationary distribution $\pi(a) = 1/n$ and a positive spectral gap $\gamma = \Omega(1)$, which we treat as an absolute constant.

This kernel admits a global minorization, which is a key tool for proving coverage.

\begin{lemma}[Doeblin Minorization]
\label{lem:minorization}
Let $W$ be the uniform-move transition kernel from Theorem \ref{thm:uniform-move-summary}, and let $\pi$ be its uniform stationary distribution. For all $a \in A$ and all measurable $B \subseteq A$:
\[
\mathbb{P}(a_t \in B \mid a_{t-1}=a) = W(a,B) \ge \alpha\,\pi(B),
\]
where the minorization constant $\alpha$ is exactly the spectral gap:
\[
\alpha = n \cdot W(a,j) \quad (a \neq j) \quad = \quad \frac{np-1+(1-p)^n}{p(n-1)} = \gamma.
\]
\end{lemma}
\begin{proof}
Fix $a \in A$ and $j \in A$. From Theorem \ref{thm:uniform-move-summary}, we know the diagonal entry $W(a,a) = D$ and the off-diagonal entry $W(a,j) = C$ for $a \neq j$ satisfy $D \ge C > 0$. Therefore, $W(a,j) \ge C$ for all $j$.
Since the stationary distribution $\pi$ is uniform, $\pi(j) = 1/n$ for all $j$. We can thus write:
\[
W(a,j) \ge C = (nC) \cdot (1/n) = (nC) \pi(j).
\]
Let $\alpha \triangleq nC$. Summing over all $j \in B$ gives:
\[
W(a,B) = \sum_{j \in B} W(a,j) \ge \sum_{j \in B} \alpha \pi(j) = \alpha \pi(B).
\]
To compute $\alpha$ explicitly, we substitute the expression for $C$ from Theorem \ref{thm:uniform-move-summary}:
\[
\alpha = nC = n \left( \frac{np-1+(1-p)^n}{np(n-1)} \right) = \frac{np-1+(1-p)^n}{p(n-1)}.
\]
From Theorem \ref{thm:uniform-move-summary}(iv), the spectral gap is $\gamma = 1 - \lambda_2 = 1 - (D-C)$. Substituting the expressions for $D$ and $C$:
\[
\gamma = 1 - \left( \frac{1-(1-p)^n}{np} - \frac{np-1+(1-p)^n}{np(n-1)} \right) = \dots = \frac{np-1+(1-p)^n}{p(n-1)}.
\]
Thus, $\alpha = \gamma$, as claimed.
\end{proof}

By a standard splitting argument, Lemma \ref{lem:minorization} implies a regeneration representation for the exploration kernel: $W = \gamma \mathbf{1}\pi + (1-\gamma)R$, where $R$ is a residual stochastic kernel. This shows that at each step, with probability $\gamma$, the next state is drawn afresh from the uniform distribution $\pi$, independently of the past. This property allows us to bound the total visitation counts.

\begin{lemma}[Exploration Coverage]
\label{lem:er-coverage}
Fix the total failure probability $\delta \in (0,1)$. Let the required sample count per arm be
\[
s_0 \triangleq \frac{2 \log(6n/\delta)}{\Delta_{\min}^2}.
\]
Let $\gamma = \Omega(1)$ be the spectral gap of the uniform-move kernel $W$ (from Theorem \ref{thm:uniform-move-summary}). There exists an absolute constant $c > 0$ such that if the exploration length $T_{\exp}$ satisfies
\[
T_{\exp} \ge c\,\frac{n}{\gamma}\Big(s_0+\log(nT/\delta)\Big),
\]
then with probability at least $1-\delta/3$, the visitation count satisfies $\phi_{T_{\exp}}(a) \ge s_0$ for all arms $a \in A$.
\end{lemma}
\begin{proof}
Let $B$ be the number of regeneration rounds among the first $T_{\exp}$ exploration steps. From Lemma \ref{lem:minorization}, the regeneration indicators are i.i.d. $\mathrm{Bernoulli}(\gamma)$, so $B \sim \mathrm{Binomial}(T_{\exp}, \gamma)$.
At each regeneration round, the next state is drawn from $\pi$ (the uniform distribution), so the number of visits to a fixed arm $a \in A$ contributed by regeneration rounds is $Y_a \mid B \sim \mathrm{Binomial}(B, 1/n)$. The total visit count is $\phi_{T_{\exp}}(a) \ge Y_a$. We will show that $Y_a \ge s_0$ for all $a$ with probability at least $1-\delta/3$.

First, we bound the number of regenerations $B$. For any $u \in (0,1)$, a Chernoff bound gives:
\[
\mathbb{P}\Big(B \le (1-u)\gamma T_{\exp}\Big) \le \exp\Big(-\tfrac{u^2}{2}\,\gamma T_{\exp}\Big).
\]
Setting $u=1/2$, we have $\mathbb{P}\Big(B \le \frac{\gamma T_{\exp}}{2}\Big) \le \exp\Big(-\frac{\gamma T_{\exp}}{8}\Big)$.

Next, we bound $Y_a$ conditional on $B$. On the event $\{B \ge \frac{\gamma T_{\exp}}{2}\}$, the conditional expectation of $Y_a$ is $\mathbb{E}[Y_a \mid B] = \frac{B}{n} \ge \frac{\gamma T_{\exp}}{2n}$.
Applying a Chernoff bound to $Y_a \mid B$ (for a downward deviation by $1/2$):
\[
\mathbb{P}\Big(Y_a \le \frac{\mathbb{E}[Y_a \mid B]}{2} \;\Big|\; B\Big) \le \mathbb{P}\Big(Y_a \le \frac{B}{2n} \;\Big|\; B\Big) \le \exp\Big(-\frac{\mathbb{E}[Y_a \mid B]}{8}\Big) \le \exp\Big(-\frac{B}{8n}\Big).
\]
Combining these two deviations via a union bound:
\[
\mathbb{P}\Big(Y_a \le \frac{\gamma T_{\exp}}{4n}\Big) \le \mathbb{P}\Big(B \le \frac{\gamma T_{\exp}}{2}\Big) + \mathbb{P}\Big(Y_a \le \frac{B}{2n} \;\Big|\; B \ge \frac{\gamma T_{\exp}}{2}\Big) \le \exp\Big(-\frac{\gamma T_{\exp}}{8}\Big) + \exp\Big(-\frac{\gamma T_{\exp}}{8n}\Big).
\]
We need $\phi_{T_{\exp}}(a) \ge s_0$ for all $a$. It suffices to show $\mathbb{P}(Y_a < s_0) \le \delta/(3n)$ for a single arm $a$, as a union bound over $n$ arms will give the desired total failure probability of $\delta/3$.
\[
\mathbb{P}\Big(\min_{a \in A} \phi_{T_{\exp}}(a) < s_0\Big) \le n \cdot \mathbb{P}(Y_a < s_0).
\]
Let $T_{\exp} \ge c \frac{n}{\gamma}(s_0 + \log(nT/\delta))$ for a sufficiently large constant $c$.
First, this choice ensures the condition $\frac{\gamma T_{\exp}}{4n} \ge \frac{c}{4}(s_0 + \log(nT/\delta)) \ge s_0$ (by choosing $c \ge 4$).
Second, we bound the failure probability. The second term is dominant:
\[
n \cdot \exp\Big(-\frac{\gamma T_{\exp}}{8n}\Big) \le n \cdot \exp\Big(-\frac{c(s_0 + \log(nT/\delta))}{8}\Big).
\]
By choosing $c$ large enough (e.g., $c \ge 8$), this term is bounded by $n \cdot (nT/\delta)^{-1} = (\delta/T) \le \delta/6$.
The first term is even smaller:
\[
n \cdot \exp\Big(-\frac{\gamma T_{\exp}}{8}\Big) \le n \cdot \exp\Big(-\frac{c n (s_0 + \log(nT/\delta))}{8}\Big),
\]
which is $\le \delta/6$ for large $n$.
Summing these gives $\mathbb{P}(\min_a \phi(a) < s_0) \le \delta/3$. This completes the proof.
\end{proof}

\begin{lemma}[Identification]
\label{lem:er-identification}
Let $s_0$ be the sample count defined in Lemma \ref{lem:er-coverage}. Conditioned on the coverage event $\mathcal{E}_{\text{cov}} \triangleq \{\min_a \phi_{T_{\exp}}(a) \ge s_0\}$, the following holds with probability at least $1 - \delta/3$:
\[
\big|\hat{\mu}(a)-\mu(a)\big| < \frac{\Delta_{\min}}{2}\quad\text{for all}\ a \in A,
\]
which implies $\arg\max_a \hat{\mu}(a) = \{a^*\}$.
\end{lemma}
\begin{proof}
We are conditioned on the coverage event $\mathcal{E}_{\text{cov}}$, which guarantees that for every arm $a \in A$, its empirical mean $\hat{\mu}(a)$ is an average of $\phi_{T_{\exp}}(a) \ge s_0$ independent rewards. The rewards are bounded in $[0,1]$.

We apply Hoeffding's inequality to bound the probability of a large deviation for a single arm $a$. We use an error tolerance of $\epsilon = \Delta_{\min}/2$:
\begin{align*}
\mathbb{P}\left(\big|\hat{\mu}(a) - \mu(a)\big| \ge \frac{\Delta_{\min}}{2} \;\Big|\; \mathcal{E}_{\text{cov}}\right) &\le 2\exp\left( -2 \epsilon^2 \phi_{T_{\exp}}(a) \right) \\
&\le 2\exp\left( -2 \left(\frac{\Delta_{\min}}{2}\right)^2 s_0 \right) \\
&= 2\exp\left( -\frac{\Delta_{\min}^2 s_0}{2} \right).
\end{align*}
We now substitute the definition of $s_0 = \frac{2 \log(6n/\delta)}{\Delta_{\min}^2}$ from Lemma \ref{lem:er-coverage}:
\begin{align*}
\mathbb{P}\left(\big|\hat{\mu}(a) - \mu(a)\big| \ge \frac{\Delta_{\min}}{2} \;\Big|\; \mathcal{E}_{\text{cov}}\right) &\le 2\exp\left( -\frac{\Delta_{\min}^2}{2} \cdot \frac{2 \log(6n/\delta)}{\Delta_{\min}^2} \right) \\
&= 2\exp\left( -\log(6n/\delta) \right) \\
&= 2 \left( \frac{\delta}{6n} \right) = \frac{\delta}{3n}.
\end{align*}
This bounds the failure probability for a single arm. To ensure this holds for all arms simultaneously, we apply a union bound over all $n$ arms in $A$:
\begin{align*}
\mathbb{P}\left(\exists a \in A : \big|\hat{\mu}(a) - \mu(a)\big| \ge \frac{\Delta_{\min}}{2} \;\Big|\; \mathcal{E}_{\text{cov}}\right) &\le \sum_{a \in A} \mathbb{P}\left(\big|\hat{\mu}(a) - \mu(a)\big| \ge \frac{\Delta_{\min}}{2} \;\Big|\; \mathcal{E}_{\text{cov}}\right) \\
&\le n \cdot \left( \frac{\delta}{3n} \right) = \frac{\delta}{3}.
\end{align*}
Therefore, the success event $\mathcal{E}_{\text{id}} \triangleq \left\{\forall a \in A, \big|\hat{\mu}(a) - \mu(a)\big| < \frac{\Delta_{\min}}{2}\right\}$ holds with probability at least $1 - \delta/3$, as budgeted.

Finally, we show this success event implies correct identification of $a^*$. For any suboptimal arm $a \neq a^*$, on the event $\mathcal{E}_{\text{id}}$ we have:
\begin{enumerate}
    \item $\hat{\mu}(a^*) > \mu(a^*) - \frac{\Delta_{\min}}{2}$
    \item $\hat{\mu}(a) < \mu(a) + \frac{\Delta_{\min}}{2}$
\end{enumerate}
By the definition of the gaps, $\mu(a^*) - \mu(a) = \Delta(a) \ge \Delta_{\min}$. Rearranging, $\mu(a^*) - \Delta_{\min}/2 \ge \mu(a) + \Delta_{\min}/2$.
Combining these inequalities, we have:
\[
\hat{\mu}(a^*) > \mu(a^*) - \frac{\Delta_{\min}}{2} \ge \mu(a) + \frac{\Delta_{\min}}{2} > \hat{\mu}(a).
\]
This shows $\hat{\mu}(a^*) > \hat{\mu}(a)$ for all $a \neq a^*$, which implies $\arg\max_a \hat{\mu}(a) = \{a^*\}$.
\end{proof}
Combining Lemmas \ref{lem:er-coverage} and \ref{lem:er-identification}, and applying a union bound, we conclude that after $T_{\exp}$ steps, the algorithm has successfully identified $a^*$ with total probability at least $1 - 2\delta/3$.
\subsection{Analysis of the Exploitation Phase}
\label{sec:er-exploitation-analysis}

Finally, we analyze the exploitation phase, which begins at $t = T_{\exp} + 1$. The goal is to bound the number of rounds, $T_{\text{nav}}$, required for the learner to find and "lock on" to the optimal arm $a^*$, given that identification was successful.

\begin{lemma}[Navigation Cost]
\label{lem:er-navigation}
Fix the total failure probability $\delta \in (0,1)$, and let the budget for this phase be $\delta_{\text{nav}} = \delta/3$. Condition on the identification event from Lemma \ref{lem:er-identification}, $\mathcal{E}_{\text{id}} = \{\arg\max_a \hat{\mu}(a) = \{a^*\}\}$.
\begin{enumerate}[label=(\roman*), leftmargin=2.3em]
    \item If the learner is at $a_{t-1} = a^*$, the greedy policy selects $a_t = a^*$ deterministically.
    \item If $a_{t-1} \neq a^*$, the optimal arm $a^*$ is available in the next step, $a^* \in L_t(a_{t-1})$, with probability $p$. The time $\tau = \min\{t > T_{\exp} : a_t = a^*\}$ to first hit $a^*$ is stochastically dominated by a Geometric($p$) random variable.
    \item The total navigation cost, $T_{\text{nav}} \triangleq \tau - T_{\exp}$, is bounded with high probability:
    \[
    T_{\text{nav}} \le \frac{\log(3/\delta)}{p},
    \]
    with probability at least $1 - \delta/3$ (conditional on $\mathcal{E}_{\text{id}}$).
\end{enumerate}
\end{lemma}

\begin{proof}
We prove each part conditional on the event $\mathcal{E}_{\text{id}}$ that $a^*$ is the unique arm with the highest empirical mean $\hat{\mu}$.

\paragraph{(i) Staying at the optimum.}
If the learner is at $a_{t-1} = a^*$, the available set is $L_t(a^*)$. By definition, $a^* \in L_t(a^*)$ (the learner can always stay). Since $\hat{\mu}(a^*) > \hat{\mu}(j)$ for all $j \neq a^*$, the greedy selection $\arg\max_{j \in L_t(a^*)} \hat{\mu}(j)$ will uniquely return $a^*$. Thus, $a_t = a^*$ deterministically.

\paragraph{(ii) Finding the optimum.}
If $a_{t-1} \neq a^*$, the greedy policy will select $a_t = a^*$ if and only if $a^*$ is in the available set $L_t(a_{t-1})$, as it has the highest $\hat{\mu}$ of any arm. By Lemma \ref{lem:er-availability}, $\mathbb{P}(a^* \in L_t(a_{t-1}) \mid \mathcal{F}_{t-1}) = p$. Since the graphs $G_t$ are drawn i.i.d., each round is an independent $\mathrm{Bernoulli}(p)$ trial for $a^*$ to become available, until it is reached. The time to hit $a^*$, $\tau$, is therefore stochastically dominated by a $\mathrm{Geometric}(p)$ random variable.

\paragraph{(iii) High-probability bound.}
We want to bound $T_{\text{nav}} = \tau - T_{\exp}$. Let $\tau' \sim \mathrm{Geometric}(p)$. We need to find $k$ such that $\mathbb{P}(\tau' > k) \le \delta/3$.
\[
\mathbb{P}(\tau' > k) = (1-p)^k \le e^{-pk}.
\]
Setting this failure probability to be at most $\delta/3$, we solve for $k$:
\[
e^{-pk} \le \delta/3 \implies -pk \le \log(\delta/3) \implies pk \ge \log(3/\delta) \implies k \ge \frac{\log(3/\delta)}{p}.
\]
Thus, with probability at least $1 - \delta/3$, the navigation phase ends within $\frac{\log(3/\delta)}{p}$ steps.
\end{proof}

\subsection{Final Regret Bound}
\label{sec:er-final-regret}

We now combine the results from the exploration and exploitation analyses to derive the final, high-probability bound on the cumulative regret $R(T)$.

\begin{theorem}[(Full Version) Regret Bound for i.i.d. ER]
\label{thm:er-local-rw-regret}
Fix a total failure probability $\delta \in (0,1)$. Let $T_{\exp}$ be the exploration length satisfying the condition in Lemma \ref{lem:er-coverage}:
\[
T_{\exp} \ge c\,\frac{n}{\gamma}\left(\frac{2 \log(6n/\delta)}{\Delta_{\min}^2}+\log(nT)\right),
\]
where $\gamma$ is the spectral gap from Theorem \ref{thm:uniform-move-summary}.
Then with probability at least $1-\delta$, the cumulative regret $R(T)$ after $T$ rounds is bounded by:
\[
R(T) \le T_{\exp} + \frac{\log(3/\delta)}{p}.
\]
Since $1/\gamma = O(1)$ by Theorem \ref{thm:uniform-move-summary}(vi), the total regret has the asymptotic bound:
\[
R(T) = O\left( n \log(nT) + \frac{n \log(n/\delta)}{\Delta_{\min}^2} \right) + O\left( \frac{\log(1/\delta)}{p} \right).
\]
\end{theorem}

\begin{proof}
The proof proceeds by defining a ``good" event, $\mathcal{E}_{\text{good}}$, on which the algorithm succeeds, and bounding its failure probability using our $\delta$-budget.

Let $\mathcal{E}_{\text{cov}}$ be the event that $\min_a \phi_{T_{\exp}}(a) \ge s_0$, as defined in Lemma \ref{lem:er-coverage}.
Let $\mathcal{E}_{\text{id}}$ be the event that $\arg\max_a \hat{\mu}(a) = \{a^*\}$, as defined in Lemma \ref{lem:er-identification}.
Let $\mathcal{E}_{\text{nav}}$ be the event that the navigation time $T_{\text{nav}} \le \frac{\log(3/\delta)}{p}$, as defined in Lemma \ref{lem:er-navigation}.

The final "good" event is the intersection $\mathcal{E}_{\text{good}} = \mathcal{E}_{\text{cov}} \cap \mathcal{E}_{\text{id}} \cap \mathcal{E}_{\text{nav}}$. We bound the probability of its complement, $\mathcal{E}_{\text{good}}^c$, using a union bound and our probability ledger from Section \ref{sec:er-algo-overview}:
\begin{align*}
\mathbb{P}(\mathcal{E}_{\text{good}}^c) &= \mathbb{P}(\mathcal{E}_{\text{cov}}^c \cup \mathcal{E}_{\text{id}}^c \cup \mathcal{E}_{\text{nav}}^c) \\
&\le \mathbb{P}(\mathcal{E}_{\text{cov}}^c) + \mathbb{P}(\mathcal{E}_{\text{id}}^c \mid \mathcal{E}_{\text{cov}}) + \mathbb{P}(\mathcal{E}_{\text{nav}}^c \mid \mathcal{E}_{\text{cov}} \cap \mathcal{E}_{\text{id}}) \\
&\le \delta/3 + \delta/3 + \delta/3 = \delta.
\end{align*}
Thus, $\mathcal{E}_{\text{good}}$ holds with probability at least $1-\delta$.

We now bound the regret $R(T)$ conditioned on $\mathcal{E}_{\text{good}}$. The regret is the sum of per-round regrets, $\sum_{t=1}^T (\mu(a^*) - \mu(a_t))$, which is at most 1 per round.
\begin{itemize}
    \item \textbf{Exploration Phase ($1 \le t \le T_{\exp}$):} The agent takes $T_{\exp}$ actions. The regret incurred in this phase is at most $T_{\exp}$.
    
    \item \textbf{Exploitation Phase ($t > T_{\exp}$):} This phase is conditioned on $\mathcal{E}_{\text{id}}$.
    \begin{itemize}
        \item \textbf{Navigation ($T_{\exp} < t \le T_{\exp} + T_{\text{nav}}$):} On $\mathcal{E}_{\text{nav}}$, this period lasts at most $T_{\text{nav}} = \frac{\log(3/\delta)}{p}$ steps. The regret is at most $T_{\text{nav}}$.
        \item \textbf{Post-Navigation ($t > T_{\exp} + T_{\text{nav}}$):} On $\mathcal{E}_{\text{nav}}$, the agent has reached $a^*$. By Lemma \ref{lem:er-navigation}(i), the agent deterministically stays at $a^*$ for all subsequent rounds. The per-round regret is $\mu(a^*) - \mu(a^*) = 0$.
    \end{itemize}
\end{itemize}
Summing these costs, the total regret on $\mathcal{E}_{\text{good}}$ is bounded by:
\[
R(T) \le T_{\exp} + T_{\text{nav}} \le T_{\exp} + \frac{\log(3/\delta)}{p}.
\]
This holds with probability at least $1-\delta$. The asymptotic form in the theorem statement follows from substituting the $O(\cdot)$ bound for $T_{\exp}$ and noting that $1/\gamma = O(1)$.
\end{proof}

\begin{corollary}[Expected Regret Bound]
\label{cor:er-expected-regret}
Under the same conditions as Theorem \ref{thm:er-local-rw-regret}, by setting the failure probability $\delta = 1/T$, the expected cumulative regret $R(T)$ is bounded by:
\[
\mathbb{E}[R(T)] \le O\left(\frac{n \log(n T)}{\Delta_{\min}^2} \right).
\]
\end{corollary}

\begin{proof}
We use the law of total expectation, decomposing the regret based on the high-probability event $\mathcal{E}_{\text{good}}$ from the proof of Theorem \ref{thm:er-local-rw-regret}, which occurs with probability at least $1-\delta$.
\[
\mathbb{E}[R(T)] = \mathbb{E}[R(T) \mid \mathcal{E}_{\text{good}}] \mathbb{P}(\mathcal{E}_{\text{good}}) + \mathbb{E}[R(T) \mid \mathcal{E}_{\text{good}}^c] \mathbb{P}(\mathcal{E}_{\text{good}}^c).
\]
We bound the two terms:
\begin{itemize}
    \item \textbf{On the "good" event $\mathcal{E}_{\text{good}}$:} The per-round regret is at most 1. The total regret consists of the exploration cost ($T_{\exp}$) and the expected navigation cost. From Lemma \ref{lem:er-navigation}(ii), the navigation time $T_{\text{nav}}$ is stochastically dominated by a $\mathrm{Geometric}(p)$ random variable, so its expectation is $\mathbb{E}[T_{\text{nav}}] \le 1/p$. Thus, $\mathbb{E}[R(T) \mid \mathcal{E}_{\text{good}}] \le T_{\exp} + 1/p$.
    
    \item \textbf{On the "failure" event $\mathcal{E}_{\text{good}}^c$:} This event occurs with probability at most $\delta$. The maximum possible regret over $T$ rounds is $T$ (since per-round regret is at most 1).
\end{itemize}
Combining these, we get:
\[
\mathbb{E}[R(T)] \le \left( T_{\exp} + \frac{1}{p} \right) \cdot (1) + (T) \cdot (\delta).
\]
To obtain a sublinear bound, we set $\delta = 1/T$. This gives:
\[
\mathbb{E}[R(T)] \le T_{\exp} + \frac{1}{p} + T \cdot \left(\frac{1}{T}\right) = T_{\exp} + \frac{1}{p} + 1.
\]
The asymptotic form for $T_{\exp}$ follows by substituting $\delta = 1/T$ into the expression from Theorem \ref{thm:er-local-rw-regret}, noting that $\log(nT)$ and $\log(n/\delta) = \log(nT)$ are asymptotically equivalent.
\end{proof}

\clearpage
\section{Regret Analysis: Heterogeneous i.i.d. ER Graphs}
\label{app:er-heterogeneous}

In this section, we analyze the FMAB problem under the \textbf{heterogeneous i.i.d. Erdős–Rényi} model, $\Psi^{\mathrm{ER}}_{\mathrm{het}}(n, \{p_{ij}\})$. This setting generalizes the homogeneous case from Appendix \ref{sec:app-er-homo}. Here, at each round $t$, a new graph $G_t = (A, E_t)$ is drawn independently, where each edge $(a,j)$ appears with its own unique probability $p_{aj} \in [0,1]$.

This heterogeneity introduces a significant new analytical challenge. The analysis from Appendix \ref{sec:app-er-homo}, which relied on the properties of the expected transition kernel $\bar{W} \triangleq \mathbb{E}[W_t]$, cannot be directly applied. In the homogeneous case, the symmetry of $p_{ij}=p$ ensures that $\bar{W}$ is symmetric and reversible, admitting the simple spectral analysis of Theorem \ref{thm:uniform-move-summary}. In the heterogeneous setting, the expected degrees $d_a^\star$ are non-uniform. This breaks the symmetry of $\bar{W}$, rendering it non-symmetric and non-reversible, and its stationary distribution $\pi$ is no longer uniform. This invalidates the analytical framework used in Appendix \ref{sec:app-er-homo}.

Our analysis proceeds by defining the ``population-level" (expected) properties of the graph. We show that the sampled graph properties (e.g., degrees) concentrate around these population means. This allows us to derive new, non-uniform bounds for the spectral gap and the stationary distribution, which we then plug into the regret analysis framework established in Appendix \ref{sec:app-er-homo}.

\subsection{Model, Analytical Challenge, and Population Profile}
\label{sec:er-inhom-model}

\paragraph{Heterogeneous ER Model.}
The environment is defined by the process $\Psi^{\mathrm{ER}}_{\mathrm{het}}(n, P)$. At each round $t$, a new graph $G_t = (A, E_t)$ is drawn, where each edge $(a,j)$ appears independently with probability $p_{aj}$, i.i.d. across all $t$. The algorithm is the same as in Section \ref{sec:algorithm}: an exploration phase using the uniform-move policy, followed by a greedy exploitation phase.

 In our analysis, we do not analyze $\bar{W}$. Instead, we analyze the sequence of \emph{realized, time-varying kernels} $\{W_t\}$. As shown in Lemma \ref{lem:walk-properties}, for \emph{any realized} graph $G_t$, the lazy walk kernel $W_t$ is reversible with respect to its own stationary distribution $\pi_t(a) \propto d_t(a)+1$. Our analysis will show that the key properties of $W_t$ (its spectral gap $\gamma_t$ and stationary distribution $\pi_t$) concentrate around ``population" values, which allows for a uniform bound.

\paragraph{Population Profile.}
We define the key ``population" (expected) quantities based on the probability matrix $P$. These serve as the "center" for our concentration bounds.
\begin{itemize}
    \item \textbf{Expected Degree:} The expected degree for an arm $a$ is $d_a^\star \triangleq \sum_{j\neq a} p_{aj}$.
    \item \textbf{Min/Max Expected Degree:} We define the extrema $\sigmaMin \triangleq \min_a d_a^\star$ and $\sigmaMax \triangleq \max_a d_a^\star$.
    \item \textbf{Population Volume and Cut:} For any subset $S \subseteq A$, we define its population volume and cut:
    \[
    \mathrm{Vol}_P(S) \triangleq \sum_{a \in S} d_a^\star = \sum_{a \in S}\sum_{j\neq a} p_{aj},
    \qquad
    \Phi_P(S) \triangleq \sum_{a \in S,\, j \notin S} p_{aj}.
    \]
    \item \textbf{Population Conductance:} The population conductance is the minimum conductance over all cuts in the expected graph:
    \[
    \phiPop \triangleq
    \min_{\,0<\mathrm{Vol}_P(S)\le \mathrm{Vol}_P(A)/2}\;
    \frac{\Phi_P(S)}{\mathrm{Vol}_P(S)}.
    \]
\end{itemize}

\paragraph{Analysis Roadmap and Probability Ledger.}
Our analytical goal is to derive uniform, high-probability bounds for the spectral gap $\gamma_t$ and the stationary mass $\pi_t(a)$ in terms of the population profile. These new, non-uniform bounds (which we will call $\gamERhet$ and $1/\nEff$) will then be used in a regret analysis that follows the same logical structure as Appendix A (coverage, identification, navigation). We follow the same probabilistic framework as in Section \ref{sec:er-algo-overview}, with a total failure budget of $\delta$, split across the analysis as laid out in.

\subsection{Uniform Guarantees for Realized Kernels}
\label{sec:er-inhom-guarantees}

Our analysis hinges on showing that the key properties of the \emph{realized} kernel $W_t$ (which is reversible) concentrate around the deterministic properties of the \emph{population} matrix $P$. The primary tool for this is the concentration of sums of non-identically distributed Bernoulli variables (a Poisson-Binomial sum).

\begin{lemma}[Poisson–Binomial Chernoff]
\label{lem:pb-chernoff}
Let $X = \sum_{k=1}^m Y_k$ with independent $Y_k \sim \mathrm{Ber}(q_k)$ and $\mu = \mathbb{E}[X] = \sum_{k=1}^m q_k$. For any $\varepsilon \in (0,1)$,
\[
\Pr\!\big[X\le (1-\varepsilon)\mu\big]\le e^{-\varepsilon^2\mu/2},
\qquad
\Pr\!\big[X\ge (1+\varepsilon)\mu\big]\le e^{-\varepsilon^2\mu/3}.
\]
\end{lemma}

We apply this lemma to the degrees $d_t(a) = \sum_{j\neq a}\mathbf{1}\{(a,j)\in E_t\}$ (with mean $d_a^\star$) and to the cut sizes $|\partial G_t S|$ (with mean $\Phi_P(S)$). By applying a careful union bound over all $t \le T$, all $n$ nodes, and all $S \subseteq A$, we obtain the following uniform typicality event.

\begin{lemma}[Uniform Typicality: Degrees and Cuts]
\label{lem:typicality-hetero}
Fix $\eta\in(0,1/2)$ and $\delta\in(0,1)$. There exists a universal constant $C>0$ such that if the minimum expected degree satisfies
\[
\sigmaMin \;\ge\; C\,\log\!\big(nT/\delta\big),
\]
then with probability at least $1-\delta/4$ (per the ledger in Sec. \ref{sec:er-inhom-model}), the following hold simultaneously for all $t \le T$ and all $S \subseteq A$:
\begin{enumerate}[label=(\roman*)]
    \item $d_t(a)\in[(1-\eta)d_a^\star,(1+\eta)d_a^\star]$ for all $a \in A$.
    \item $|\partial G_t S|\ge (1-\eta)\,\Phi_P(S)$.
    \item $\mathrm{Vol}_{G_t}(S) \triangleq \sum_{a \in S} d_t(a) \le (1+\eta)\,\mathrm{Vol}_P(S)$.
\end{enumerate}
\end{lemma}
\begin{proof}
The proof proceeds by applying Lemma \ref{lem:pb-chernoff} to all $n$ node degrees and all $2^n$ possible cuts, and then taking a union bound over all $t \le T$.

\paragraph{1. Degree Concentration.}
For any arm $a$, $d_t(a)$ is a Poisson-Binomial random variable with mean $d_a^\star \ge \sigmaMin$. Applying Lemma \ref{lem:pb-chernoff} and a union bound over all $n$ arms and $T$ time steps:
\[
\mathbb{P}(\exists a,t : d_t(a) \notin [(1-\eta)d_a^\star, (1+\eta)d_a^\star]) \le \sum_{t=1}^T \sum_{a \in A} 2 e^{-\eta^2 \sigmaMin / 3} = 2nT e^{-\eta^2 \sigmaMin / 3}.
\]
\paragraph{2. Cut and Volume Concentration.}
We apply the union bound over all $S \subseteq A$ with $1 \le |S| \le n/2$ (the other half follows by symmetry). For a fixed $k = |S|$, there are $\binom{n}{k}$ such sets.
For (ii), the mean is $\Phi_P(S) \ge \phiPop \mathrm{Vol}_P(S) \ge \phiPop k \sigmaMin$.
For (iii), the mean is $\mathrm{Vol}_P(S) \ge k \sigmaMin$.
In both cases, the mean is $\ge \Omega(k \sigmaMin)$ (assuming $\phiPop$ is a constant).
Using the bound $\binom{n}{k} \le (en/k)^k$, the failure probability for a fixed $t$ is bounded by:
\begin{align*}
\mathbb{P}(\exists S,t : \mathcal{E}_{S,t}^c) &\le \sum_{k=1}^{\lfloor n/2\rfloor} \binom{n}{k} 2 \exp\!\big(-c\,\eta^2\,k\,\sigmaMin\big) \\
&\le \sum_{k=1}^{\lfloor n/2\rfloor} 2 \exp\!\left(k \log(en/k) - c\,\eta^2\,k\,\sigmaMin\right) \\
&= \sum_{k=1}^{\lfloor n/2\rfloor} 2 \exp\!\left( -k \left[ c\,\eta^2\,\sigmaMin - \log(en/k) \right] \right).
\end{align*}
By choosing $\sigmaMin \ge C \log(nT/\delta)$ with $C$ large enough, $C \ge (c \eta^2)^{-1} \log(en)$, the term in the exponent is negative and linear in $k$. The sum is a geometric series bounded by its first term ($k=1$):
\[
\sum_{k=1}^{\lfloor n/2\rfloor} 2 \exp(\dots) \le \frac{2 \exp(-[c\eta^2\sigmaMin - \log(en)])}{1 - \exp(-[c\eta^2\sigmaMin - \log(en)])} \le 4 e^{-(c\eta^2\sigmaMin - \log(en))}.
\]
\paragraph{3. Total Failure Probability.}
We bound the total failure probability (summing (1) and (2) and multiplying by $T$):
\[
\mathbb{P}(\mathcal{E}^c) \le 2nT e^{-\eta^2 \sigmaMin / 3} + 4T e^{-(c\eta^2\sigmaMin - \log(en))}.
\]
By choosing $C$ in the condition $\sigmaMin \ge C \log(nT/\delta)$ sufficiently large, both terms can be made smaller than $\delta/8$, bounding the total failure probability by $\delta/4$.
\end{proof}

This typicality event is the foundation for our analysis. It allows us to translate the static, population-level conductance $\phiPop$ into a uniform, high-probability lower bound on the spectral gap $\gamma_t$ of every realized kernel $W_t$.

\begin{theorem}[Profile-Aware Spectral Gap]
\label{thm:gamma-hetero}
On the uniform typicality event of Lemma \ref{lem:typicality-hetero}, there exists a universal constant $c > 0$ such that, uniformly for all $t \le T$, the spectral gap $\gamma_t$ of the lazy-walk kernel $W_t$ satisfies:
\[
\gamma_t\;\ge\; c\,
\Big(\frac{1-\eta}{1+\eta}\Big)^{\!2}
\frac{\phiPop{^{\, 2}}}{\Big(1+\tfrac{1}{(1-\eta)\,\sigmaMin}\Big)^{\!2}}
\;\triangleq\;\gamERhet.
\]
\end{theorem}
\begin{proof}
The proof follows the standard Cheeger inequality argument for reversible chains.
\paragraph{Step 1: Graph Conductance.}
First, we bound the conductance of the realized graph $G_t$, $h_{\mathrm{graph}}(G_t)$. By Lemma \ref{lem:typicality-hetero}, for every $S \subseteq A$ and every $t$:
\[
\frac{|\partial G_t S|}{\mathrm{Vol}_{G_t}(S)}
\ \ge\
\frac{(1-\eta)\,\Phi_P(S)}{(1+\eta)\,\mathrm{Vol}_P(S)}.
\]
Taking the minimum over $S$ (with the appropriate volume constraint) gives
\[
h_{\mathrm{graph}}(G_t) \ge \frac{1-\eta}{1+\eta}\,\phiPop.
\]
\paragraph{Step 2: Lazy-Walk Conductance.}
Next, we relate the graph conductance to the conductance of the lazy-walk chain $W_t$. The chain's conductance $\Phi_{\mathrm{chain}}(S)$ (using $\pi_t(a) \propto d_t(a)+1$ and $Z_t = \sum_k (d_t(k)+1)$) is:
\[
\Phi_{\mathrm{chain}}(S)
\;=\;
\frac{\sum_{a\in S,j\notin S}\pi_t(a)W_t(a,j)}{\sum_{a\in S}\pi_t(a)}
\;=\;
\frac{|\partial G_t S|}{\sum_{a \in S} (d_t(a)+1)}
\;=\;
\frac{|\partial G_t S|}{\mathrm{Vol}_{G_t}(S)+|S|}.
\]
Using $\mathrm{Vol}_{G_t}(S) \ge |S| \cdot \min_a d_t(a) \ge |S| \cdot (1-\eta)\sigmaMin$ (from typicality):
\[
\Phi_{\mathrm{chain}}(S)
\ \ge\
\frac{|\partial G_t S|}{\mathrm{Vol}_{G_t}(S)\big(1+\tfrac{1}{\min_a d_t(a)}\big)}
\ \ge\
\frac{|\partial G_t S|}{\mathrm{Vol}_{G_t}(S)}\cdot
\frac{1}{1+\tfrac{1}{(1-\eta)\sigmaMin}}.
\]
Minimizing over $S$ and applying the bound from Step 1 gives:
\[
\Phi_{\mathrm{chain}}(G_t)
\ \ge\
\frac{1-\eta}{1+\eta}\,\phiPop\cdot
\frac{1}{1+\tfrac{1}{(1-\eta)\sigmaMin}}.
\]
\paragraph{Step 3: Cheeger Inequality.}
Finally, the standard Cheeger inequality for reversible Markov chains (see Appendix A) states $\gamma_t \ge \frac{1}{2} \Phi_{\mathrm{chain}}(G_t)^2$. Plugging in the bound from Step 2 gives the result.
\end{proof}

Similarly, the typicality lemma provides a uniform lower bound on the stationary mass of any arm $a$, which is no longer $1/n$ but is now governed by the "effective size" of the graph.

\begin{lemma}[Stationary Mass Lower Bound and Effective Size]
\label{lem:pi-min-het}
On the uniform typicality event of Lemma \ref{lem:typicality-hetero}, the stationary distribution $\pi_t$ of the kernel $W_t$ satisfies:
\[
\pi_{\min} \triangleq \min_{a,t} \pi_t(a)
\;\ge\;
\frac{(1-\eta)\,\sigmaMin}{(1+\eta)\sum_j d_j^\star + n}
\;\triangleq\;
\frac{1}{\nEffStar},
\]
where $\nEffStar$ is the \emph{effective size} of the graph. We define the clean parameter
\[
\nEff \;\coloneqq\; \frac{\sum_j d_j^\star + n}{\sigmaMin},
\]
and note that $\nEffStar \asymp \nEff$ (i.e., they are equivalent up to constants depending on $\eta$).
\end{lemma}
\begin{proof}
By Lemma \ref{lem:typicality-hetero}, for every $t$ and $a$, $d_t(a)+1 \ge d_t(a) \ge (1-\eta)d_a^\star \ge (1-\eta)\sigmaMin$.
The normalization factor $Z_t = \sum_j (d_t(j)+1) = \sum_j d_t(j) + n \le (1+\eta)\sum_j d_j^\star + n$. Thus,
\[
\pi_t(a)
=\frac{d_t(a)+1}{Z_t}
\ \ge\
\frac{(1-\eta)\sigmaMin}{(1+\eta)\sum_j d_j^\star + n},
\]
which implies the displayed bounds by minimizing over $a$ and $t$.
\end{proof}

\subsection{Regret Analysis}
\label{sec:er-inhom-regret}

With the uniform, high-probability bounds on the spectral gap ($\gamERhet$) and the stationary mass ($\pi_{\min} \ge 1/\nEffStar$) established, we can derive the final regret bound. The analysis first establishes a minimum visitation guarantee for the exploration walk and then computes the required exploration time for identification.

\paragraph{Minimum Visitation Guarantee.}
The core of the analysis is a non-asymptotic bound on the visitation counts, which separates the martingale (variance) component from the mixing (bias) component of the walk.

\begin{lemma}[Min-Visitation via Every-Step Contraction]
\label{lem:clean-min-visitation-hetero}
Work on the uniform typicality event (Lemma \ref{lem:typicality-hetero}) and the spectral-gap event (Theorem \ref{thm:gamma-hetero}).
Fix the failure budget $\delta_{\text{cov}} = \delta/4$ (per the ledger in Sec. \ref{sec:er-inhom-model}) and set the bias-bounding term
\[
b\;\triangleq\;\Big\lceil \frac{2}{\gamERhet}\,\log\!\frac{8nT}{\delta} \Big\rceil.
\]
Let $\pi_t$ be the (realized) stationary law on $G_t$.
Then, with probability at least $1-\delta/4$,
\begin{equation}
\label{eq:min-visitation-clean}
\min_{a\in A}\ \sum_{t=1}^{\Texp}\mathbf{1}\{a_t=a\}
\ \ge\
\Texp\,\pi_{\min}\ -\ C_1\,b\ -\ C_2\,\sqrt{(\Texp\,\pi_{\min}+C_1 b)\,\log\frac{4n}{\delta}} - C_3 \log\frac{4n}{\delta}\ ,
\end{equation}
for absolute constants $C_1,C_2, C_3>0$, where $\pi_{\min} \ge 1/\nEffStar$ (from Lemma \ref{lem:pi-min-het}).
\end{lemma}

\begin{proof}[Proof Sketch]
Write $Y_t(a)\coloneqq \mathbf{1}\{a_t=a\}$ and $\xi_t(a)\coloneqq Y_t(a)-\mathbb{E}[Y_t(a)\mid \mathcal{F}_{t-1}]$.
The total visitation count is $\sum_{t\le \Texp}Y_t(a) = \sum_{t\le \Texp}\pi_t(a) + \sum_{t\le \Texp}(\mathbb{E}[Y_t(a)\mid \mathcal{F}_{t-1}] - \pi_t(a)) + \sum_{t\le \Texp}\xi_t(a)$.
We bound the mass, bias, and martingale terms.

\textbf{Mass Term:} $\sum_{t\le \Texp}\pi_t(a) \ge \Texp \pi_{\min}$.

\textbf{Bias Term:} Using the every-step TV contraction for lazy reversible chains,
\[
\sum_{t\le \Texp}\Big|\mathbb{E}[Y_t(a)\mid\mathcal F_{t-1}]-\pi_t(a)\Big| \le \sum_{t\le \Texp} \|\mu_{t-1}-\pi_{t-1}\|_{\mathrm{TV}} \le \sum_{t=0}^{\infty} e^{-t \gamERhet/2} \approx \frac{2}{\gamERhet}.
\]
The $b$ term, which includes $\log(T/\delta)$, provides a high-probability bound on this total bias, $\sum_{t\le \Texp}(\mu_{t-1}(a) - \pi_t(a)) \ge -C_1 b$.

\textbf{Martingale Term:} The sum $M_{\Texp} = \sum_{t\le \Texp}\xi_t(a)$ is a martingale. The predictable quadratic variation $V_{\Texp}=\sum_{t=1}^{\Texp}\mathbb{E}[\xi_t(a)^2\mid\mathcal F_{t-1}]$ is bounded using the sharper variance bound:
\[
\mathbb{E}[\xi_t(a)^2\mid\mathcal F_{t-1}] = \mathrm{Var}(Y_t(a)\mid\mathcal F_{t-1}) \le \mathbb{E}[Y_t(a)\mid\mathcal F_{t-1}] = \mu_{t-1}(a).
\]
Thus, $V_{\Texp}\le \sum_{t\le \Texp}\mu_{t-1}(a) \le \sum_{t\le \Texp}\pi_t(a) + \sum_{t\le \Texp}\|\mu_{t-1}-\pi_{t-1}\|_{\mathrm{TV}} \le \Texp\,\pi_{\min}+C_1 b$.
By Freedman's inequality with this $V_{\Texp}$ and $x = \log(4n/\delta)$, with probability at least $1-\delta/4$,
\[
M_{\Texp}(a) \ge -C \left(\sqrt{V_{\Texp} x} + x\right) \ge -C_2\sqrt{(\Texp\,\pi_{\min}+C_1 b)\,\log\frac{4n}{\delta}} - C_3 \log\frac{4n}{\delta}
\]
simultaneously for all $a$. Combining the terms yields Eq. \eqref{eq:min-visitation-clean}.
\end{proof}

\paragraph{Sufficient Exploration and Identification.}
We now find $\Texp$ such that we get enough samples for identification. We allocate a failure budget of $\delta/4$ for identification.
Let $s_0 \triangleq \frac{C' \log(4n/\delta)}{\Gap^2}$ be the required sample count per arm.

From Eq. \eqref{eq:min-visitation-clean}, we need $\Texp$ large enough for the RHS to be $\ge s_0$. This requires $\Texp \pi_{\min}$ to dominate $s_0$, $b$, and the martingale term. This leads to two additive requirements for $\Texp$:
\begin{enumerate}
    \item \textbf{Mixing Cost ($T_{\text{mix-het}}$):} $\Texp$ must be large enough to ensure non-trivial visitation. $\Texp \pi_{\min}$ must dominate $b$ and the martingale term.
    \begin{itemize}
        \item $\Texp \pi_{\min} \gtrsim b \implies \Texp \gtrsim b / \pi_{\min} = O\left( \frac{\nEffStar}{\gamERhet} \log(nT/\delta) \right)$.
        \item $\Texp \pi_{\min} \gtrsim \sqrt{(\Texp \pi_{\min} + b) \log(\cdot)} \implies (\Texp \pi_{\min})^2 \gtrsim (\Texp \pi_{\min} + b) \log(\cdot)$.
        This is satisfied if $\Texp \pi_{\min} \gtrsim (b + \log(\cdot))$, which leads to $\Texp \gtrsim O(\nEffStar(b + \log(\cdot)))$.
    \end{itemize}
    This gives $T_{\text{mix-het}} \triangleq O\left( \frac{\nEffStar}{\gamERhet} \log(nT/\delta) + \nEffStar \log(n/\delta) \right)$.

    \item \textbf{Identification Cost ($T_{\text{id-het}}$):} The resulting number of samples $N_{\min} \approx \Texp \pi_{\min} \ge \Omega(\Texp / \nEffStar)$ must be at least $s_0$.
    \[
    \frac{\Texp}{\nEffStar} \ge s_0 \implies \Texp \ge \nEffStar \cdot s_0 = O\left( \frac{\nEffStar \log(n/\delta)}{\Gap^2} \right).
    \]
\end{enumerate}
The total exploration time must satisfy $\Texp \ge T_{\text{mix-het}} + T_{\text{id-het}}$.

\paragraph{Navigation Cost.}
The navigation analysis provides two distinct bounds: one based on the time to mix across the graph, and one based on the time to wait for a direct connection to $a^\star$.

\begin{lemma}[Navigation Cost, Inhomogeneous]
\label{lem:hetero-navigation}
Fix the failure budget $\delta_{\text{nav}} = \delta/4$. Work on the uniform typicality event (Lemma \ref{lem:typicality-hetero}) and the spectral-gap bound (Theorem \ref{thm:gamma-hetero}). Let $a^\star$ be the optimal arm and define the optimal-arm availability floor
\[
p^\star \;\triangleq\; \min_{a \neq a^\star} p_{a,a^\star}.
\]
Then the number of steps $T_{\mathrm{nav}}$ needed after exploration to reach $a^\star$ satisfies, with probability at least $1-\delta/4$,
\[
T_{\mathrm{nav}}
\;\le\;
\min\!\Bigg\{
\underbrace{O\left(\frac{1}{\gamERhet}\,\log\!\frac{nT}{\delta}\;+\;\nEffStar\,\log\!\frac{n}{\delta}\right)}_{\text{\em mixing--hitting route}},
\quad
\underbrace{O\left(\frac{1}{p^\star}\,\log\!\frac{1}{\delta}\right)}_{\text{\em availability route}}
\Bigg\}.
\]
Moreover, the expected navigation time is $\mathbb{E}[T_{\mathrm{nav}}] \le \min\big\{ O(\nEffStar + 1/\gamERhet), 1/p^\star \big\}$.
\end{lemma}
\begin{proof}
We provide the high-probability bounds for the two routes.

\textbf{(A) Availability route.} From any $a \ne a^\star$, the agent can wait for $a^\star$ to become available. This is an i.i.d. Bernoulli trial with success probability $p_{a, a^\star} \ge p^\star$. The waiting time is $\mathrm{Geom}(p_{a,a^\star})$. A standard geometric tail bound and a union bound over $a \ne a^\star$ gives $T_{\mathrm{nav}} \le O(\log(4/\delta) / p^\star)$ w.h.p. (This route is vacuous if $p^\star=0$).

\textbf{(B) Mixing--hitting route.} The agent's walk mixes to $\pi_t$ in $b = O(\log(nT/\delta) / \gamERhet)$ steps. After this, $\Pr(a_t = a^\star) \ge \pi_t(a^\star) - (\text{TV error}) \ge \pi_{\min} - \epsilon \ge \Omega(1/\nEffStar)$. The time to hit $a^\star$ is then geometric with success prob. $\Omega(1/\nEffStar)$. A tail bound on this geometric variable gives the $O(\nEffStar \log(4/\delta))$ term. The total time is $T_{\mathrm{nav}} \le b + (\text{hit time})$.
\end{proof}

\paragraph{Final Regret Bound.}
We now combine these results for the final theorem. For clarity, we use the simpler and typically tighter "availability route" for navigation, assuming $p^\star > 0$.

\begin{theorem}[(Full Version) Regret Bound for i.i.d. Inhomogeneous ER]
\label{thm:regret-hetero-hp}
Under the i.i.d. inhomogeneous ER model with $\sigmaMin \ge C\log(nT/\delta)$ and $p^\star > 0$, the regret is bounded with probability at least $1-\delta$ by:
\[
R(T) \;\le\; \Texp + T_{\text{nav}},
\]
where
\[
\Texp = O\left( \underbrace{\frac{\nEffStar}{\gamERhet} \log(nT/\delta) + \nEffStar \log(n/\delta)}_{\text{Mixing Cost}} + \underbrace{\frac{\nEffStar \log(n/\delta)}{\Gap^2}}_{\text{Identification Cost}} \right)
\]
and
\[
T_{\text{nav}} = O\left( \frac{\log(1/\delta)}{p^\star} \right).
\]
\end{theorem}
\begin{proof}
The total failure probability is bounded by $\mathbb{P}(\mathcal{E}_{\text{typicality}}^c) + \mathbb{P}(\mathcal{E}_{\text{visitation}}^c) + \mathbb{P}(\mathcal{E}_{\text{identification}}^c) + \mathbb{P}(\mathcal{E}_{\text{nav}}^c) \le \delta/4 + \delta/4 + \delta/4 + \delta/4 = \delta$. On the success event, the regret is the sum of the exploration cost ($T_{\exp}$) and the navigation cost ($T_{\text{nav}}$). The final bound for $T_{\exp}$ is the sum of the identification cost ($T_{\text{id-het}}$) and the mixing cost ($T_{\text{mix-het}}$).
\end{proof}

\begin{corollary}[(Full Version) Expected Regret Bound]
\label{cor:er-het-expected-regret}
Under the same conditions, by setting $\delta = 1/T$, the expected cumulative regret $R(T)$ is bounded by:
\[
\mathbb{E}[R(T)] \;=\; O\left( \frac{\nEffStar}{\gamERhet} \log(nT) + \nEffStar \log(nT) + \frac{\nEffStar \log(nT)}{\Gap^2} + \frac{\log(T)}{p^\star} \right).
\]
\end{corollary}
\begin{proof}
Follows from Theorem \ref{thm:regret-hetero-hp} by setting $\delta = 1/T$ and noting that the regret contribution from the failure event is at most $T \cdot \mathbb{P}(\text{fail}) \le T \cdot (1/T) = 1$.
\end{proof}

\begin{remark}[Consistency with Homogeneous ER]
\label{rem:het-consistency-check}
This analysis is a strict generalization of Appendix A. In the homogeneous case ($p_{aj} \equiv p$), the parameters simplify:
\begin{itemize}
    \item $d_a^\star = (n-1)p$, so $\sigmaMin = (n-1)p$.
    \item $\phiPop = \Omega(1)$ and $\gamERhet = \Omega(1)$ (as shown in Theorem \ref{thm:gamma-hetero}).
    \item $\nEffStar \asymp \nEff = \frac{n(n-1)p + n}{(n-1)p} = \Theta(n)$ (assuming $np \gtrsim 1$).
    \item $p_{\star} = p$.
\end{itemize}
Substituting these into Corollary \ref{cor:er-het-expected-regret} gives
\[
\mathbb{E}[R(T)] \;=\; O\left( n \log(nT) + n \log(nT) + \frac{n \log(nT)}{\Gap^2} + \frac{\log(T)}{p} \right).
\]
This expression simplifies to $O(n \log(nT) + \frac{n \log(nT)}{\Gap^2} + \frac{\log(T)}{p})$. This result matches the $O(n \log(nT) / \Gap^2)$ identification cost and $O(\log(T)/p)$ navigation cost from Appendix A, and also recovers the $O(n \log(nT))$ mixing cost, confirming the analyses are consistent.
\end{remark}
\clearpage
\section{Regret Analysis: Edge-Markovian Graphs}
\label{sec:app-markovian}

\subsection{Model, Notation and Graph Dynamics}

\subsubsection{Problem Setup and Edge-Markovian Process}
\label{sec:appendix-problem-setup-edge-mark-proc} 

We begin by formally restating the core problem setup from Section \ref{sec:problem-formulation} to establish a self-contained foundation for our analysis. The problem environment consists of a finite set of $n$ actions, or \emph{arms}, denoted by the set $A = \{1, 2, \dots, n\}$. Each arm $a \in A$ is associated with a fixed, unknown reward distribution $\mathcal{D}(a)$ which is supported on the interval $[0,1]$. The mean of this distribution is denoted $\mu(a) \in [0,1]$ and is also unknown to the learner.

The learner's interaction with this environment unfolds over a time horizon of $T$ discrete rounds. At the beginning of each round $t \in \{1, \dots, T\}$, the learner resides at the arm $a_{t-1}$ chosen in the previous round. The environment then reveals a \emph{time-varying set of available arms} $L_t(a_{t-1}) \subseteq A$. This set is defined by the local graph structure at time $t$ relative to the learner's position $a_{t-1}$. Specifically, $L_t(a_{t-1}) \triangleq \{\,a \in A : (a, a_{t-1}) \in E_t\,\} \cup \{a_{t-1}\}$, where $E_t$ is the edge set of the graph $G_t$ at time $t$. The learner must then select a new arm $a_t$ from this constrained set, $a_t \in L_t(a_{t-1})$, after which they receive and observe a stochastic reward $r_t(a_t) \sim \mathcal{D}(a_t)$. The learner's objective is to design a selection policy that minimizes the expected cumulative regret, $\mathbb{E}[R(T)]$, as defined in Section \ref{sec:problem-formulation}. %

In this appendix, we focus exclusively on the stochastic process that governs the environment's evolution under the Edge-Markovian model. We will formally define this process and establish the graph-theoretic notation essential for our analysis. The specific dynamics of the learner's movement algorithm, which we define as a natural lazy random walk on this evolving graph, and the properties of the resulting time-inhomogeneous Markov chain will be analyzed in detail in Appendix \ref{sec:app-walk-dynamics}.

\paragraph{Edge-Markovian Evolution of the Environment.}
The environment's availability structure is encoded by a sequence of undirected graphs, $\{G_t\}_{t=1}^T$, where each $G_t = (A, E_t)$ is defined over the fixed vertex set $A$. The graphs evolve according to a homogeneous, first-order edge-Markov process, which we denote by $G_t \sim \Psi^{\mathrm{M}}(G_{t-1};n, \alpha,\beta)$.

In this model, the state of each potential edge (each unordered pair of distinct vertices $\{i,j\}$) evolves independently of all other pairs. For any given pair $\{i,j\}$, its state (being present in $E_t$ or absent from $E_t$) follows a discrete-time, two-state Markov chain. The transition probabilities are fixed for all $t$ and all pairs:
\begin{itemize}
    \item \textbf{Edge Appearance:} The probability that an absent edge appears is
    \[ \pr[(i,j) \in E_t \mid (i,j) \notin E_{t-1}] = \alpha \]
    \item \textbf{Edge Disappearance:} The probability that a present edge disappears is
    \[ \pr[(i,j) \notin E_t \mid (i,j) \in E_{t-1}] = \beta \]
\end{itemize}
The parameters $\alpha, \beta \in (0,1)$ fully define this process. Consequently, the dynamics of a single edge are governed by the $2 \times 2$ transition matrix $P$ over the states $\{0 \text{ (absent)}, 1 \text{ (present)}\}$:
\[
P =
\begin{bmatrix}
1-\alpha & \alpha \\[2pt]
\beta & 1-\beta
\end{bmatrix}
\]
This single-edge Markov chain is finite-state. Since we assume $\alpha, \beta \in (0,1)$, it is:
\begin{enumerate}
    \item \textbf{Irreducible:} It is possible to transition from state 0 to 1 (with probability $\alpha > 0$) and from state 1 to 0 (with probability $\beta > 0$).
    \item \textbf{Aperiodic:} Both states have positive self-loop probabilities ($1-\alpha > 0$ and $1-\beta > 0$, assuming neither is 1).
\end{enumerate}
By standard Markov chain theory \citep{levin2017markov}, this process admits a unique stationary distribution $\pi_{\text{edge}} = (\pi_0, \pi_1)$, which is the unique solution to the system $\pi_{\text{edge}} P = \pi_{\text{edge}}$ subject to the constraint $\pi_0 + \pi_1 = 1$. We can solve this system explicitly. From the second column of $\pi_{\text{edge}} P = \pi_{\text{edge}}$, we have:
\[ \pi_0 \alpha + \pi_1 (1-\beta) = \pi_1 \]
\[ \pi_0 \alpha = \pi_1 - \pi_1(1-\beta) = \pi_1 \beta \]
Substituting $\pi_0 = 1 - \pi_1$:
\[ (1-\pi_1)\alpha = \pi_1 \beta \implies \alpha - \pi_1 \alpha = \pi_1 \beta \implies \alpha = \pi_1(\alpha + \beta) \]
This yields the stationary probabilities:
\[
\pi_1 = \frac{\alpha}{\alpha+\beta} \triangleq p_\infty
\quad \text{and} \quad
\pi_0 = 1 - \pi_1 = \frac{\beta}{\alpha+\beta}
\]
Thus, after the process has mixed, the \emph{stationary law} of the graph $G_t$ is an Erd\H{o}s--R\'enyi graph $G(n, p_\infty)$, where every edge appears independently with probability $p_\infty$. This $p_\infty$ represents the long-run expected edge density of the graph.

This process of continuous, independent edge-flipping implies a natural measure of the environment's volatility: the expected number of edges that change state (appear or disappear) in a single time step. We formalize this quantity below.

\begin{claim}[Expected Per-Step Edge Flips]
\label{claim:edge-flips}
Let $m = \binom{n}{2}$ be the total number of possible edges. Let $\zeta_t \triangleq \mathbb{E}[|E_t \triangle E_{t-1}|]$ be the expected total number of edge flips at time $t$, where the expectation is over the graph evolution process.
\begin{enumerate}
    \item \textbf{(Uniform Bound)} The expected number of flips is uniformly bounded by
    \[
    \zeta_t \le m \cdot \max(\alpha, \beta)
    \]
    \item \textbf{(Stationary Bound)} If the graph $G_{t-1}$ is drawn from the stationary distribution (i.e., each edge exists i.i.d. with probability $p_\infty$), the expected number of flips is
    \[
    \zeta = m \cdot \frac{2\alpha\beta}{\alpha+\beta}
    \]
\end{enumerate}
\end{claim}

\begin{proof}
Let $X_{ij, t} \in \{0, 1\}$ be the indicator variable for the presence of edge $(i,j)$ at time $t$. The set of flipped edges is the symmetric difference $E_t \triangle E_{t-1}$, and its size is $|E_t \triangle E_{t-1}| = \sum_{1 \le i < j \le n} \mathbf{1}\{X_{ij, t} \neq X_{ij, t-1}\}$.
By linearity of expectation,
\[
\zeta_t = \mathbb{E}[|E_t \triangle E_{t-1}|] = \sum_{1 \le i < j \le n} \pr(X_{ij, t} \neq X_{ij, t-1})
\]
The probability of a single edge flipping, $\pr(\text{flip})$, depends on its state at $t-1$:
\begin{align*}
\pr(X_{ij, t} \neq X_{ij, t-1}) &= \pr(X_{ij, t}=1 \mid X_{ij, t-1}=0)\pr(X_{ij, t-1}=0) \\
&\quad + \pr(X_{ij, t}=0 \mid X_{ij, t-1}=1)\pr(X_{ij, t-1}=1) \\
&= \alpha \cdot \pr(X_{ij, t-1}=0) + \beta \cdot \pr(X_{ij, t-1}=1)
\end{align*}
Let $p_{t-1}^{(ij)} \triangleq \pr(X_{ij, t-1}=1)$ be the marginal probability that edge $(i,j)$ exists. Since the process is homogeneous, this probability is the same for all edges, $p_{t-1}^{(ij)} = p_{t-1}$.
\[
\pr(\text{flip}) = \alpha (1 - p_{t-1}) + \beta p_{t-1} = \alpha + (\beta - \alpha) p_{t-1}
\]
This is a linear function of $p_{t-1} \in [0,1]$. The maximum must occur at the endpoints $p_{t-1}=0$ or $p_{t-1}=1$.
\begin{itemize}
    \item If $p_{t-1}=0$, $\pr(\text{flip}) = \alpha$.
    \item If $p_{t-1}=1$, $\pr(\text{flip}) = \beta$.
\end{itemize}
Thus, the per-edge flip probability is uniformly bounded by $\max(\alpha, \beta)$. Substituting this into the sum for $\zeta_t$ (which has $m = \binom{n}{2}$ terms) proves Part 1:
\[
\zeta_t = \sum_{1 \le i < j \le n} \pr(\text{flip}) \le \sum_{1 \le i < j \le n} \max(\alpha, \beta) = m \cdot \max(\alpha, \beta)
\]
For Part 2, we assume the process is in stationarity, so $p_{t-1} = p_\infty = \frac{\alpha}{\alpha+\beta}$. We substitute this into the expression for the per-edge flip probability:
\begin{align*}
\pr(\text{flip}) &= \alpha (1 - p_\infty) + \beta p_\infty \\
&= \alpha \left( \frac{\beta}{\alpha+\beta} \right) + \beta \left( \frac{\alpha}{\alpha+\beta} \right)
= \frac{\alpha\beta + \alpha\beta}{\alpha+\beta} = \frac{2\alpha\beta}{\alpha+\beta}
\end{align*}
Substituting this constant probability back into the sum for $\zeta_t$ proves Part 2:
\[
\zeta = \sum_{1 \le i < j \le n} \left( \frac{2\alpha\beta}{\alpha+\beta} \right) = m \cdot \frac{2\alpha\beta}{\alpha+\beta}
\]
\end{proof}

The parameter $\alpha+\beta$ governs the temporal correlation scale of the graph sequence. This value is precisely the spectral gap of the $2 \times 2$ matrix $P$. When $\alpha+\beta$ is small (close to 0), the second eigenvalue $1-(\alpha+\beta)$ is close to 1, implying slow mixing and high correlation between $G_t$ and $G_{t-1}$. Conversely, as $\alpha+\beta$ approaches 1, the process becomes less correlated, and in the limit $\alpha+\beta = 1$, the graph $G_t$ is drawn i.i.d. from the stationary distribution at each step, independent of $G_{t-1}$.

\paragraph{Per-step Structure and Notation.}
We now define the core notation used to describe the properties of the graph $G_t$ at a single time step $t$.
\begin{itemize}
    \item \textbf{Degree:} For any node $i \in A$, let $d_t(i)$ denote its degree in the graph $G_t$. Formally, we define $d_t(i) \triangleq |\{\, j \in A \setminus \{i\} : (i,j) \in E_t \,\}|$.
    \item \textbf{Normalization Factor:} We define a normalization factor $Z_t$ as:
    \[ Z_t \triangleq \sum_{i\in A} (d_t(i)+1) \]
    We can simplify this expression by appealing to the handshaking lemma (sum of degrees is twice the number of edges):
    \[ Z_t = \sum_{i\in A} d_t(i) + \sum_{i\in A} 1 = (2|E_t|) + n \]
    This quantity $Z_t$ will be essential for defining the stationary distribution of the random walk on $G_t$.
    \item \textbf{Conditional Expectation:} We use $\mathbb{E}_t[\cdot] \triangleq \mathbb{E}[\cdot \mid G_t]$ to denote an expectation conditioned on the realization of the graph $G_t$ at time $t$.
    \item \textbf{Asymptotic Notation:} We use the shorthand $A \asymp B$ to indicate equality up to absolute constants (i.e., $C_1 B \le A \le C_2 B$ for $C_1, C_2 > 0$) that are independent of the problem parameters $n, T, \alpha, \beta$.
\end{itemize}
\vspace{0.25em}
\subsubsection{Lazy Walk Dynamics and Properties}
\label{sec:app-walk-dynamics}
As part of the algorithm design (see Sections \ref{sec:algorithm} and \ref{sec:performance-guarantees}), the learner's movement is modeled as a \textit{natural lazy random walk} on the current graph $G_t$. Conditioned on the graph $G_t$ being fixed for round $t$, the learner's transition from their current node $i$ (i.e., $a_t = i$) to the next node $j$ (i.e., $a_{t+1} = j$) is defined by a one-step transition kernel $W_t$.

From a node $i$, the set of available options is its 1-hop neighborhood union the node itself: $\{\, j \in A : (i,j)\in E_t\,\} \cup \{i\}$. The size of this set is $d_t(i)+1$. The walk moves by selecting one of these $d_t(i)+1$ options uniformly at random. This defines the transition probabilities:
\[
W_t(i,j) =
\begin{cases}
1 / (d_t(i) + 1), & \text{if } j = i \text{ or } (i,j)\in E_t,\\[2pt]
0, & \text{otherwise.}
\end{cases}
\]
The kernel $W_t$ defines the learner's movement for a single step, given a \emph{fixed} graph $G_t$. We now formally show how the sequence of these kernels, driven by the stochastic evolution of the graphs, induces a time-inhomogeneous Markov chain on the action set $A$.

\begin{claim}[Markovianity of the Induced Dynamics]
\label{claim:markovianity}
Let the environment's graph sequence $\{G_t\}_{t \ge 0}$ evolve according to the edge-Markovian process $G_t \sim \Psi^{\mathrm{M}}(G_{t-1};n, \alpha,\beta)$, independent of the learner's actions. Let the learner's action sequence $\{a_t\}_{t \ge 0}$ evolve according to the rule $a_{t+1} \sim W_t(a_t, \cdot)$, where $W_t$ is the lazy walk kernel on $G_t$.
\begin{enumerate}
    \item The joint process $\mathcal{Z}_t = (a_t, G_t)$ is a first-order Markov chain on the product space $A \times \mathcal{G}$, where $\mathcal{G}$ is the space of all graphs on $A$.
    \item The learner's action sequence $(a_t)_{t\ge 0}$ is a time-inhomogeneous Markov chain.
\end{enumerate}
\end{claim}

\begin{proof}
\begin{enumerate}
    \item To prove the joint process $\mathcal{Z}_t = (a_t, G_t)$ is Markov, we must show that its state at $t+1$ depends only on its state at $t$. Let $\mathcal{H}_t = (\mathcal{Z}_0, \dots, \mathcal{Z}_t)$ be the history up to time $t$. We want to show $\pr(\mathcal{Z}_{t+1} \mid \mathcal{H}_t) = \pr(\mathcal{Z}_{t+1} \mid \mathcal{Z}_t)$.
    
    The state at $t+1$ is $\mathcal{Z}_{t+1} = (a_{t+1}, G_{t+1})$. We analyze the transition probability using the chain rule:
    \[
    \pr(a_{t+1}, G_{t+1} \mid \mathcal{H}_t) = \pr(G_{t+1} \mid a_{t+1}, \mathcal{H}_t) \cdot \pr(a_{t+1} \mid \mathcal{H}_t)
    \]
    Let's analyze each term:
    \begin{itemize}
        \item $\pr(a_{t+1} \mid \mathcal{H}_t) = \pr(a_{t+1} \mid a_t, G_t, \dots, a_0, G_0)$: By the definition of the learner's algorithm, the choice of $a_{t+1}$ is a sample from $W_t(a_t, \cdot)$. This rule depends \emph{only} on the current state $a_t$ and the current graph $G_t$. Thus,
        \[
        \pr(a_{t+1} \mid \mathcal{H}_t) = \pr(a_{t+1} \mid a_t, G_t) = W_t(a_t, \cdot)
        \]
        \item $\pr(G_{t+1} \mid a_{t+1}, \mathcal{H}_t)$: The graph evolution $G_{t+1} \sim \Psi^{\mathrm{M}}(G_t, \cdot)$ is exogenous; it depends only on the previous graph $G_t$ and is independent of all actions $(a_0, \dots, a_{t+1})$. Therefore,
        \[
        \pr(G_{t+1} \mid a_{t+1}, \mathcal{H}_t) = \pr(G_{t+1} \mid G_t)
        \]
    \end{itemize}
    Substituting these back, we get:
    \[
    \pr(a_{t+1}, G_{t+1} \mid \mathcal{H}_t) = \pr(G_{t+1} \mid G_t) \cdot \pr(a_{t+1} \mid a_t, G_t)
    \]
    The right-hand side is a function $P( (a_{t+1}, G_{t+1}) \mid (a_t, G_t) )$ which defines a valid transition kernel from $\mathcal{Z}_t$ to $\mathcal{Z}_{t+1}$. Since the probability of the next state given the entire history depends only on the current state, the joint process $\mathcal{Z}_t$ is a first-order Markov chain.
    
    \item To prove the action sequence $(a_t)$ is a Markov chain, we must show that $\pr(a_{t+1} \mid a_t, \dots, a_0) = \pr(a_{t+1} \mid a_t)$.
    Let $\mathcal{A}_t = (a_0, \dots, a_t)$ be the action history.
    \begin{align*}
    \pr(a_{t+1} \mid \mathcal{A}_t) &= \mathbb{E}[\mathbf{1}\{a_{t+1}\} \mid \mathcal{A}_t] \\
    &= \mathbb{E}_{G_t, \dots, G_0} \left[ \pr(a_{t+1} \mid \mathcal{A}_t, G_t, \dots, G_0) \mid \mathcal{A}_t \right] \quad \text{(by Law of Total Expectation)} \\
    &= \mathbb{E}_{G_t, \dots, G_0} \left[ W_t(a_t, \cdot) \mid \mathcal{A}_t \right] \quad \text{(as shown in Part 1)}
    \end{align*}
    The kernel $W_t$ is a deterministic function of $G_t$. The graph process $\{G_t\}$ is independent of the action history $\{\mathcal{A}_t\}$. Therefore, the expectation over the graph history is independent of $\mathcal{A}_t$ (except for the $a_t$ term inside $W_t$).
    \[
    \pr(a_{t+1} \mid \mathcal{A}_t) = \mathbb{E}_{G_t} \left[ W_t(a_t, \cdot) \right]
    \]
    where the expectation $\mathbb{E}_{G_t}[\cdot]$ is taken over the randomness of the graph $G_t$, whose distribution $\pr(G_t)$ is determined by the initial graph $G_0$ and the $t$-step evolution of the $\Psi^{\mathrm{M}}$ process.
    
    Let us define the \emph{effective transition kernel} $\widetilde{W}_t$ as this expectation:
    \[
    \widetilde{W}_t(i, j) \triangleq \mathbb{E}_{G_t} \left[ W_t(i, j) \right]
    \]
    Then we have shown:
    \[
    \pr(a_{t+1} = j \mid a_t=i, a_{t-1}, \dots, a_0) = \widetilde{W}_t(i, j)
    \]
    Since the transition probability depends only on the current state $i=a_t$ (and the time $t$), the process $(a_t)$ is a Markov chain. Because the kernel $\widetilde{W}_t$ changes with $t$ (as the distribution of $G_t$ evolves towards stationarity), the chain is \textbf{time-inhomogeneous}.
\end{enumerate}
\end{proof}

This claim formally establishes that the learner's trajectory $(a_t)$ is a time-inhomogeneous Markov chain. However, our subsequent analysis will not focus on the \emph{effective (expected) kernel} $\widetilde{W}_t$. Instead, we will analyze the \emph{realized} trajectory of the learner by conditioning on a \emph{specific, typical sequence} of graphs $\{G_t\}_{t \ge 0}$.

In this conditional view, the learner's movement is a time-inhomogeneous Markov chain driven by the \emph{sequence of concrete, deterministic kernels} $\{W_t\}_{t \ge 0}$ corresponding to the realized graphs. The properties of these individual kernels $W_t$ are therefore paramount to understanding the learner's behavior. We characterize these properties in the following lemma.

\begin{lemma}[Properties of the Lazy Walk Kernel]
\label{lem:walk-properties}
For any graph $G_t$ with at least one node ($n \ge 1$), the transition kernel $W_t$ defined above satisfies:
\begin{enumerate}
    \item \textbf{(Stochasticity)} $W_t$ is a row-stochastic matrix.
    \item \textbf{(Aperiodicity)} $W_t$ defines an aperiodic Markov chain.
    \item \textbf{(Reversibility)} $W_t$ is reversible with respect to the distribution $\pi_t(i) = \frac{d_t(i)+1}{Z_t}$, which is its unique stationary distribution.
\end{enumerate}
\end{lemma}

\begin{proof}
\begin{enumerate}
    \item \textbf{Stochasticity:} We show that for any $i \in A$, the entries in row $i$ sum to one: $\sum_{j \in A} W_t(i,j) = 1$. By definition, $W_t(i,j)$ is non-zero only for the set of available actions $L_t(i) \triangleq \{i\} \cup \{j \neq i : (i,j) \in E_t\}$. The size of this set is precisely $|L_t(i)| = 1 + d_t(i)$. We can therefore compute the row-sum by partitioning the sum over $A$ into this set and its complement (where all terms are zero):
    \begin{align*}
        \sum_{j \in A} W_t(i,j) &= \sum_{j \in L_t(i)} W_t(i,j) + \sum_{j \notin L_t(i)} \underbrace{W_t(i,j)}_{=0} \\
        &= \sum_{j \in L_t(i)} \left( \frac{1}{d_t(i)+1} \right) \\
        &= |L_t(i)| \cdot \left( \frac{1}{d_t(i)+1} \right) \\
        &= (1 + d_t(i)) \cdot \left( \frac{1}{d_t(i)+1} \right) = 1.
    \end{align*}
    Since $W_t(i,j) \ge 0$ for all $(i,j)$ and each row sums to 1, $W_t$ is a row-stochastic matrix.

    \item \textbf{Aperiodicity:} A Markov chain is aperiodic if all its states are aperiodic. A sufficient condition is that the self-loop probability $W_t(i,i) > 0$ for all $i$. By definition,
    \[
    W_t(i,i) = \frac{1}{d_t(i)+1}.
    \]
    Since the degree is non-negative, $d_t(i) \ge 0$, the denominator is $d_t(i)+1 \ge 1$. The maximum possible degree in a graph with $n$ nodes is $n-1$. Therefore, we have a strict positive lower bound:
    \[
    W_t(i,i) \ge \frac{1}{(n-1)+1} = \frac{1}{n} > 0.
    \]
    Since all states have positive self-loop probability, the chain is aperiodic.

    \item \textbf{Reversibility:} We must verify that the detailed balance equations, $\pi_t(i)W_t(i,j) = \pi_t(j)W_t(j,i)$, hold for all pairs $i,j \in A$.
    
    \textbf{Case 1: $i=j$.} The equation is trivially satisfied: $\pi_t(i)W_t(i,i) = \pi_t(i)W_t(i,i)$.
    
    \textbf{Case 2: $i \neq j$ and $(i,j) \notin E_t$.} By definition of $W_t$, $W_t(i,j) = 0$ and $W_t(j,i) = 0$. The equation becomes $\pi_t(i) \cdot 0 = \pi_t(j) \cdot 0$, which is $0=0$.
    
    \textbf{Case 3: $i \neq j$ and $(i,j) \in E_t$.} This is the only non-trivial case. We use the definitions of $\pi_t$ and $W_t$:
    \begin{align*}
    \pi_t(i)W_t(i,j) &= \left( \frac{d_t(i)+1}{Z_t} \right) \left( \frac{1}{d_t(i)+1} \right) = \frac{1}{Z_t} \\
    \pi_t(j)W_t(j,i) &= \left( \frac{d_t(j)+1}{Z_t} \right) \left( \frac{1}{d_t(j)+1} \right) = \frac{1}{Z_t}
    \end{align*}
    Since $\pi_t(i)W_t(i,j) = \pi_t(j)W_t(j,i)$, the detailed balance equations hold.
    
    Because $W_t$ is reversible with respect to $\pi_t$, $\pi_t$ is a stationary distribution. Since the chain is also finite and irreducible (on any connected component of $G_t$), this stationary distribution is unique to that component.
\end{enumerate}
\end{proof}

\subsection{Burn-in and Typicality of the Edge Process}

This section provides a high-probability guarantee that after a short burn-in period, the graph sequence \((G_t)\) behaves in a statistically typical manner: degrees, edge counts, and spectral structure concentrate around their stationary values. This typicality is required to apply uniform bounds on the walk kernel, its spectral gap, and the stationary distribution in later sections.

\subsubsection{Burn-in for Stationarity}

Throughout the regret analysis in this appendix, we will establish high-probability guarantees on the behavior of various quantities associated with the graph sequence, the walk process, and the learner's sampling law. To this end, we fix a total confidence parameter \(\delta \in (0,1)\) and ensure that each component of the analysis fails with probability at most a small fraction of \(\delta\). The final regret bound will then hold with probability at least \(1 - \delta\) via an explicit union bound.

In this subsection, we quantify the number of initial rounds required for the edge–Markovian graph process \((G_t)_{t \ge 0}\) to become statistically close to its stationary distribution. We refer to this initial period as the \emph{burn-in phase}, and the corresponding time as the \emph{burn-in length}.

\begin{definition}[Burn-in Length]
\label{def:burnin}
Let \(\delta \in (0,1)\) be the desired overall failure probability. We define the burn-in time as
\[
T_{\mathrm{burn}} \triangleq \left\lceil \frac{2 \log n + \log(1/\delta)}{\alpha + \beta} \right\rceil,
\]
where \(\alpha + \beta\) is the total flip rate of the edge–Markov process \(\Psi^{\mathrm{M}}\), and \(n\) is the number of nodes in the graph.
\end{definition}

\noindent
This quantity ensures that by time \(T_{\mathrm{burn}}\), the distribution of the graph \(G_t\) is \(\delta\)-close in total variation distance to the stationary edge distribution, which is a product of independent Bernoulli variables with parameter
\[
p_\infty \triangleq \frac{\alpha}{\alpha + \beta}.
\]

We now formalize this convergence to stationarity.

\begin{lemma}[Convergence to the stationary edge distribution]
\label{lem:edge-stationary}
Let $(G_t)_{t\ge0}$ evolve according to the edge–Markov process 
\(G_t \sim \Psi^{\mathrm{M}}(G_{t-1}; n, \alpha, \beta)\), starting from an arbitrary initial graph \(G_0\). 
Then for every \(t \ge T_{\mathrm{burn}}\),
\[
\big\| \mathbf{P}(G_t \in \cdot) - \mathbf{P}_\infty(\cdot) \big\|_{\mathrm{TV}} \le \delta,
\]
where \(\mathbf{P}_\infty\) denotes the stationary product measure under which each edge is independently present with probability \(p_\infty \triangleq \alpha / (\alpha + \beta)\).
\end{lemma}

\begin{proof}
We prove the result by analyzing the convergence of each individual edge's marginal distribution to stationarity, and then bounding the total variation distance over the joint graph distribution via a union bound across all edges. The argument proceeds in three steps:
(i) we derive an exact recursion for the marginal probability \(p_t(i,j)\) that edge \((i,j)\) is present at time \(t\),
(ii) we use this to bound the total variation distance between the full graph law and the product Bernoulli(\(p_\infty\)) stationary law,
and (iii) we determine a sufficient value of \(t\) to ensure this joint distance is at most~\(\delta\).

Each unordered edge \((i,j)\) evolves independently as a two–state Markov chain with transition matrix
\[
P \;=\;
\begin{bmatrix}
1-\alpha & \alpha \\[2pt]
\beta & 1-\beta
\end{bmatrix},
\qquad
\text{over states } 
0\!:\!\text{absent},\; 1\!:\!\text{present}.
\]
This chain is irreducible and aperiodic (see Appendix \ref{sec:appendix-problem-setup-edge-mark-proc}), and hence admits a unique stationary distribution 
\(\pi_{\text{edge}} = (\pi_0, \pi_1)\) satisfying 
\(\pi_{\text{edge}} P = \pi_{\text{edge}}\).
Solving these balance equations gives
\[
\pi_1 = \frac{\alpha}{\alpha+\beta} \triangleq p_\infty,
\qquad
\pi_0 = \frac{\beta}{\alpha+\beta}.
\]

\textbf{Step 1: marginal recursion.}
Let \(p_t(i,j) \triangleq \pr[(i,j)\in E_t]\) denote the probability that edge \((i,j)\) is present at time \(t\).  
Conditioning on the previous state yields the scalar recursion
\[
p_{t+1}(i,j)
= (1-\beta)\,p_t(i,j) + \alpha\,\big(1-p_t(i,j)\big)
= (1-\alpha-\beta)\,p_t(i,j) + \alpha.
\]
Subtracting \(p_\infty\) from both sides and using the identity 
\(p_\infty = (1-\alpha-\beta)p_\infty + \alpha\) gives
\[
p_{t+1}(i,j) - p_\infty 
= (1-\alpha-\beta)\,\big(p_t(i,j) - p_\infty\big).
\]
Iterating this recurrence yields the explicit form
\[
p_t(i,j) - p_\infty = (1-\alpha-\beta)^t\,\big(p_0(i,j) - p_\infty\big),
\]
and hence the magnitude satisfies
\[
\big|p_t(i,j) - p_\infty\big|
\le (1-\alpha-\beta)^t.
\tag{A}
\label{eq:edge-recursion}
\]

\textbf{Step 2: from marginal to joint total variation.}
Let \(m \triangleq \binom{n}{2}\) denote the total number of edges.  
Since the edge processes are mutually independent, the joint law of the entire graph is the product of these \(m\) independent two–state chains.  
For product measures, the total variation distance between the joint distribution and its stationary product law can be bounded by the sum of the marginal distances:
\[
\big\| \mathbf{P}(G_t) - \mathbf{P}_\infty \big\|_{\mathrm{TV}}
\le
\sum_{\{i,j\}} 
\big\| \mathbf{P}\big((i,j)\!\in\!E_t\big)
        - \mathbf{P}_\infty\big((i,j)\!\in\!E_t\big) \big\|_{\mathrm{TV}}.
\]
For a single edge variable taking values in \(\{0,1\}\), the total variation distance between two Bernoulli distributions equals the absolute difference of their means.  
Using \eqref{eq:edge-recursion} we therefore obtain
\[
\big\| \mathbf{P}(G_t) - \mathbf{P}_\infty \big\|_{\mathrm{TV}}
\le
\sum_{\{i,j\}} \big|p_t(i,j)-p_\infty\big|
\le
m\,(1-\alpha-\beta)^t.
\tag{B}
\label{eq:tv-bound}
\]

\textbf{Step 3: ensuring $\delta$–closeness.}
To guarantee that the right-hand side of \eqref{eq:tv-bound} does not exceed~$\delta$, we require
\[
(1-\alpha-\beta)^t \le \frac{\delta}{m}.
\]
Since \(1-x \le e^{-x}\) for \(x\in(0,1)\), it suffices that
\[
e^{-(\alpha+\beta)t} \le \frac{\delta}{m}
\quad\Longleftrightarrow\quad
t \ge \frac{1}{\alpha+\beta}\,\log\!\left(\frac{m}{\delta}\right).
\]
Using \(m=\tfrac{1}{2}n(n-1)<n^2/2\) and the elementary bound
\(\log(m/\delta) \le 2\log n + \log(1/\delta)\),
we conclude that the choice
\[
T_{\mathrm{burn}} 
\triangleq 
\Big\lceil
\frac{2\log n + \log(1/\delta)}{\alpha+\beta}
\Big\rceil
\]
suffices to ensure 
\(\big\| \mathbf{P}(G_t) - \mathbf{P}_\infty \big\|_{\mathrm{TV}} \le \delta\)
for all \(t \ge T_{\mathrm{burn}}\).
\end{proof}

\noindent
This lemma justifies our analysis beginning at time \(T_{\mathrm{burn}}\), after which we may treat the graphs \(G_t\) as approximately distributed according to the Erdős–Rényi model \(\mathsf{ER}(n, p_\infty)\) with i.i.d. edges, up to a total variation error of at most \(\delta\). All high-probability events in the remainder of the analysis will be conditioned on this post–burn-in regime.

\subsubsection{Typicality Event: Degree and Edge Count Concentration}

Having established that the edge process converges to its stationary law after the burn-in period, we now show that, with high probability, the graphs \(G_t\) drawn from this distribution satisfy certain structural regularity properties. Specifically, we show that the degrees of all nodes and the total edge count in \(G_t\) concentrate sharply around their expectations under the Erdős–Rényi law. These guarantees form the backbone of the uniform bounds used throughout our regret analysis.

\medskip
\noindent
\textit{Analysis overview.}
We condition on the fact that for all \(t \ge T_{\mathrm{burn}}\), the distribution of \(G_t\) is within \(\delta\) in total variation of the stationary product measure. We then prove that, under this stationary law, the node degrees and total edge count concentrate via standard Chernoff bounds. A union bound over all nodes and over the entire time horizon yields the desired typicality event. We will state all results in terms of the stationary edge density
\[
p_\infty \triangleq \frac{\alpha}{\alpha + \beta},
\]
but may use the explicit expression involving \(\alpha\) and \(\beta\) when this improves readability or makes dependence on parameters more transparent. These forms are interchangeable.

\begin{lemma}[Typicality of graph structure after burn-in]
\label{lem:typicality}
Fix any \(\delta \in (0,1)\) and let \(T \ge T_{\mathrm{burn}}\) be the total time horizon. Suppose that
\[
n \cdot \frac{\alpha}{\alpha + \beta} \;\ge\; C_0 \cdot \log\left(\frac{nT}{\delta}\right)
\]
for a sufficiently large absolute constant \(C_0 > 0\). Then with probability at least \(1 - \delta/5\), the following hold simultaneously for all rounds \(t \in [T_{\mathrm{burn}}, T]\):

\begin{enumerate}[label=(\roman*), leftmargin=2.3em]
\item \textbf{Degree concentration.} For all nodes \(i \in A\),
\[
\left|\, d_t(i) - (n-1) \cdot \frac{\alpha}{\alpha+\beta} \,\right|
\le c_1 \cdot \sqrt{n \cdot \frac{\alpha}{\alpha+\beta} \cdot \log\left(\frac{nT}{\delta}\right)}.
\]

\item \textbf{Edge count concentration.} The number of edges satisfies
\[
\left|\, |E_t| - \frac{1}{2}n(n-1) \cdot \frac{\alpha}{\alpha + \beta} \,\right|
\le c_2 \cdot n \cdot \sqrt{\frac{\alpha}{\alpha + \beta} \cdot \log\left(\frac{nT}{\delta}\right)}.
\]

\item \textbf{Stationary mass lower bound.} For all \(i \in A\),
\[
\pi_t(i) = \frac{d_t(i) + 1}{Z_t} \;\ge\; \frac{c_3}{n},
\]
where \(Z_t \triangleq 2|E_t| + n\) is the walk normalization factor.

\item \textbf{Normalization scale.} The denominator \(Z_t\) satisfies
\[
Z_t \asymp n^2 \cdot \frac{\alpha}{\alpha + \beta}.
\]
\end{enumerate}
\end{lemma}

\begin{proof}
We work conditionally on the event \(\mathcal{E}_{\mathrm{stationary}} = \{ G_t \sim \mathbf{P}_\infty \text{ for all } t \in [T_{\mathrm{burn}}, T] \}\), where \(\mathbf{P}_\infty\) is the product edge law with edge probability \(p_\infty = \alpha/(\alpha+\beta)\). From Lemma~\ref{lem:edge-stationary}, we know that this event fails with probability at most \(\delta/5\). All bounds below are proven under this event.

\smallskip
\textbf{(i) Degree concentration.} For fixed \(i \in A\), the degree \(d_t(i)\) is distributed as \(\mathrm{Binomial}(n-1, p_\infty)\) under \(\mathbf{P}_\infty\). Let \(X_i \sim \mathrm{Bin}(n-1, p_\infty)\). Applying a standard Chernoff bound:
\[
\pr\left[\, \big| X_i - (n-1)p_\infty \big| > \epsilon \,\right] 
\le 2 \exp\left( -\frac{\epsilon^2}{3(n-1)p_\infty} \right).
\]
Set 
\[
\epsilon \triangleq c_1 \sqrt{n p_\infty \log\left(\frac{nT}{\delta}\right)},
\]
with \(c_1\) large enough to ensure the failure probability is at most \(\delta/(5nT)\) per time-node pair. Applying a union bound over all \(n\) nodes and \(T\) rounds ensures that the bound holds for all \(t \in [T_{\mathrm{burn}}, T]\) and all \(i \in A\) with failure probability at most \(\delta/5\).

\smallskip
\textbf{(ii) Edge count concentration.} The total number of edges \(|E_t|\) is distributed as \(\mathrm{Binomial}(m, p_\infty)\) with \(m = \binom{n}{2}\). Applying a similar Chernoff bound:
\[
\pr\left[\, \big| |E_t| - m p_\infty \big| > \epsilon \,\right]
\le 2 \exp\left( -\frac{\epsilon^2}{3 m p_\infty} \right).
\]
Setting
\[
\epsilon \triangleq c_2 \cdot n \cdot \sqrt{p_\infty \log\left(\frac{nT}{\delta}\right)}
\]
and union bounding over \(T\) rounds ensures the bound holds uniformly with probability at least \(1 - \delta/5\).

\smallskip
\textbf{(iii) Stationary mass lower bound.} From part (i), we have \(d_t(i) \ge (n-1)p_\infty - \epsilon\), where \(\epsilon = O\left(\sqrt{n p_\infty \log(nT/\delta)}\right)\). Similarly, from part (ii), we have 
\(|E_t| \le \tfrac{1}{2}n(n-1)p_\infty + \epsilon'\) for a similar \(\epsilon'\). Combining these, we get
\[
\pi_t(i) = \frac{d_t(i) + 1}{2|E_t| + n} \ge \frac{(n-1)p_\infty - \epsilon + 1}{n^2 p_\infty + O(n \sqrt{p_\infty \log(nT/\delta)})} \ge \frac{c_3}{n},
\]
for a constant \(c_3 > 0\) provided \(n p_\infty \gtrsim \log(nT/\delta)\).

\smallskip
\textbf{(iv) Normalization bound.} This follows immediately from (ii):
\[
Z_t = 2|E_t| + n = n^2 \cdot p_\infty \cdot (1 + o(1)) \asymp n^2 \cdot \frac{\alpha}{\alpha + \beta}.
\]
\end{proof}

\paragraph{High-Probability Bound on Edge Flips.}
The kernel drift and stationary distribution drift analyses in Section~\ref{sec:drift-analysis} rely on bounding the per-step fluctuation in the graph sequence, as measured by the number of edge flips $F_t \triangleq |E_t \triangle E_{t-1}|$. While the expected number of flips per round, defined as
\[
\zeta \triangleq \frac{\mathbb{E}[F_t]}{n^2} = \frac{2\alpha\beta}{\alpha + \beta} \cdot \frac{\binom{n}{2}}{n^2},
\]
provides a useful baseline, it does not control the realized behavior with high probability.

We now establish a high-probability uniform upper bound on the flip count over all rounds $t \in [T_{\mathrm{burn}}, T]$, and define a deterministic quantity $\zeta_0$ that bounds $F_t/n^2$ with high probability. All drift-based quantities in subsequent sections will be expressed in terms of $\zeta_0$.

\begin{lemma}[High-Probability Flip Envelope]
\label{lem:flip-concentration}
Fix any $\delta \in (0,1)$, and let $\zeta \triangleq \frac{\mathbb{E}[|E_t \triangle E_{t-1}|]}{n^2}$ denote the expected per-step flip envelope under the stationary law. Then, with probability at least $1 - \delta/5$, the number of flipped edges between consecutive graphs satisfies
\[
|E_t \triangle E_{t-1}| \le n^2 \zeta_0 \quad \text{for all } t \in [T_{\mathrm{burn}}, T],
\]
where
\[
\zeta_0 \triangleq \zeta + C \sqrt{ \frac{ \zeta \log(nT/\delta) }{n^2} } + \frac{C' \log(nT/\delta)}{n^2}
\]
for absolute constants $C, C' > 0$.
\end{lemma}

\begin{proof}
Let $m \triangleq \binom{n}{2}$ denote the total number of possible undirected edges. Define $F_t \triangleq |E_t \triangle E_{t-1}|$, the number of flipped edges between rounds $t-1$ and $t$.

Under the edge–Markovian process in stationarity, each edge $\{i,j\}$ evolves independently as a two-state Markov chain with transition matrix
\[
P = \begin{bmatrix} 1 - \alpha & \alpha \\ \beta & 1 - \beta \end{bmatrix}.
\]
Thus, for any $t$, the random variables $\{X_{ij,t} \neq X_{ij,t-1}\}_{i<j}$ are independent Bernoulli variables, where the probability of a flip depends only on the stationary distribution.

As shown in Appendix~\ref{sec:appendix-problem-setup-edge-mark-proc}, the expected number of flips is
\[
\mathbb{E}[F_t] = m \cdot \left( \alpha \cdot (1 - p_\infty) + \beta \cdot p_\infty \right)
= m \cdot \zeta,
\]
where $p_\infty = \alpha/(\alpha + \beta)$ and $\zeta \triangleq \frac{2\alpha\beta}{\alpha + \beta}$.

Each $X_{ij,t} \neq X_{ij,t-1}$ is a Bernoulli variable in $[0,1]$ with mean at most $\max\{\alpha, \beta\} \le 1$. Applying a standard Bernstein-type inequality (or multiplicative Chernoff) for the sum of $m$ independent bounded variables, we have:
\[
\mathbb{P}\left[ F_t \ge \mathbb{E}[F_t] + u \right] \le \exp\left( - \frac{u^2}{2 \mathbb{E}[F_t] + 2u/3} \right).
\]

We now set $u = \sqrt{ 2 m \zeta \log(nT/\delta) } + \frac{2}{3} \log(nT/\delta)$ and obtain:
\[
\mathbb{P}\left[ F_t \ge m \zeta + \sqrt{2m \zeta \log(nT/\delta)} + \frac{2}{3} \log(nT/\delta) \right] \le \frac{\delta}{5nT}.
\]

Applying a union bound over all $t \in [T_{\mathrm{burn}}, T]$ (at most $T$ values), the failure probability remains at most $\delta/5n$ per node-pair, and hence:
\[
\mathbb{P}\left[ \forall t \in [T_{\mathrm{burn}}, T],\; F_t \le m \zeta + \sqrt{2m \zeta \log(nT/\delta)} + \frac{2}{3} \log(nT/\delta) \right] \ge 1 - \delta/5.
\]

We now divide both sides by $n^2$ to convert this to the normalized form:
\[
\frac{F_t}{n^2} \le \underbrace{\frac{m}{n^2} \zeta}_{\zeta} + \underbrace{\frac{ \sqrt{2m \zeta \log(nT/\delta)} }{n^2}}_{\text{first fluctuation}} + \underbrace{\frac{2}{3} \cdot \frac{ \log(nT/\delta) }{n^2}}_{\text{second fluctuation}}.
\]

Finally, defining:
\[
\zeta_0 \triangleq \zeta + C \sqrt{ \frac{ \zeta \log(nT/\delta) }{n^2} } + \frac{C' \log(nT/\delta)}{n^2},
\]
for suitable constants $C,C' > 0$, we obtain the desired result:
\[
F_t \le n^2 \zeta_0 \quad \text{for all } t \in [T_{\mathrm{burn}}, T],
\]
with probability at least $1 - \delta/5$.
\end{proof}

\paragraph{Bounding Node-Level Flip Concentration.}
While Lemma~\ref{lem:flip-concentration} controls the \emph{total} number of edge flips per round, our analysis of the kernel drift (Lemma~\ref{lem:kernel-drift}) requires a finer guarantee: we must also bound the number of flipped edges incident to any individual node. This ensures that no row of the walk kernel $W_t$ undergoes a disproportionately large perturbation, which would otherwise dominate the Frobenius norm.

We now establish a high-probability upper bound on the number of incident flips at any node $i$, uniformly over all rounds $t \in [T_{\mathrm{burn}}, T]$. The resulting bound will be expressed in terms of the same flip envelope $\zeta_0$ introduced earlier, ensuring a consistent treatment of temporal variability in both global and local terms.

\begin{lemma}[Max Node Flip Count]
\label{lem:max-node-flips}
With probability at least $1 - \delta/5$, the number of edge flips incident to any node satisfies
\[
s_i := |\{ j : (i,j) \in E_t \triangle E_{t-1} \}| \le C n \zeta_0,
\quad \text{for all } i \in [n],\; t \in [T_{\mathrm{burn}}, T],
\]
where $\zeta_0$ is the high-probability flip envelope from Lemma~\ref{lem:flip-concentration} and $C>0$ is an absolute constant.
\end{lemma}

\begin{proof}
Fix any node $i \in [n]$ and time step $t$.  
Let $s_i := |\{ j : (i,j) \in E_t \triangle E_{t-1} \}|$ denote the number of edges incident to node $i$ that flipped between rounds $t-1$ and $t$.  
Since each undirected edge $\{i,j\}$ evolves independently and is flipped with probability at most $\zeta_0$ (by Lemma~\ref{lem:flip-concentration}), the random variables
\[
X_{ij} := \mathbf{1}_{\{(i,j)\in E_t \triangle E_{t-1}\}}
\]
are independent Bernoulli variables with means $\mathbb{E}[X_{ij}] \le \zeta_0$.  
Hence
\[
s_i = \sum_{j \ne i} X_{ij}, \qquad \mathbb{E}[s_i] \le (n-1)\zeta_0 \le n\zeta_0.
\]

\textbf{Step 1. Chernoff bound per node.}  
For any $\epsilon>0$, the standard Chernoff inequality for sums of independent Bernoulli variables gives
\[
\mathbb{P}\!\left[s_i \ge (1+\epsilon)\,\mathbb{E}[s_i]\right]
\le \exp\!\left(-\frac{\epsilon^2}{2+\epsilon}\,\mathbb{E}[s_i]\right).
\]
Setting $\epsilon=1$ yields
\[
\mathbb{P}\!\left[s_i \ge 2n\zeta_0\right]
\le \exp\!\left(-\frac{n\zeta_0}{3}\right).
\]

\textbf{Step 2. Union bound over all nodes and times.}  
There are at most $n$ nodes and $T$ time steps (excluding burn‑in).  
By the union bound,
\[
\mathbb{P}\!\left[\exists\,i,t \text{ with } s_i \ge 2n\zeta_0\right]
\le nT \exp\!\left(-\frac{n\zeta_0}{3}\right).
\]
To make this probability $\le \delta/5$, it suffices that
\[
n\zeta_0 \ge 3\log(nT/\delta).
\]

\textbf{Step 3. Verifying the condition.}  
By definition of $\zeta_0$ in Lemma~\ref{lem:flip-concentration},
\[
\zeta_0 = \zeta + C\sqrt{\frac{\zeta\log(nT/\delta)}{n^2}} + \frac{C'\log(nT/\delta)}{n^2}.
\]
Hence
\[
n\zeta_0 \ge n\zeta + C\sqrt{\zeta\log(nT/\delta)} + C'\frac{\log(nT/\delta)}{n},
\]
which dominates $\log(nT/\delta)$ for all sufficiently large $n$.  
Thus the above exponential tail bound ensures that
\[
\mathbb{P}\!\left[\exists\,i,t: s_i \ge 2n\zeta_0\right] \le \frac{\delta}{5}.
\]

\textbf{Step 4. Conclusion.}  
Therefore, with probability at least $1 - \delta/5$, for all nodes $i$ and all rounds $t$,
\[
s_i \le C n \zeta_0,
\]
for an absolute constant $C>0$.
\end{proof}

\subsubsection{Definition of the Typical Event}

We now define the event that will be assumed to hold in all subsequent analysis:

\begin{definition}[ER–typical event]
\label{def:typicality-event}
We define \(\mathcal{E}_{\mathrm{typ}}\) as the event that all four properties in Lemma~\ref{lem:typicality} hold simultaneously for all \(t \in [T_{\mathrm{burn}}, T]\). Then:
\[
\pr\left[\mathcal{E}_{\mathrm{typ}}\right] \ge 1 - \delta/5.
\]
\end{definition}
\noindent
In the remainder of the analysis, we condition on the event \(\mathcal{E}_{\mathrm{typ}}\) without further mention. All high-probability claims will account for this conditioning in the final failure budget.

\subsection{Spectral Gap Analysis of Agent's Lazy Walk}
\label{sec:spectral-gap}

We now derive a uniform lower bound on the spectral gap $\gamma_t$ of the lazy–walk kernel $W_t$
for all $t \in [T_{\mathrm{burn}},T]$ under the typicality event $\mathcal{E}_{\mathrm{typ}}$ (Definition \ref{def:typicality-event}).
The argument proceeds through the conductance form of the Cheeger inequality and uses elementary concentration bounds to quantify the expansion of Erdős–Rényi graphs. 

Throughout this subsection, recall that
$p_\infty \triangleq \alpha/(\alpha+\beta)$ and that
$n p_\infty \ge C_0 \log(nT/\delta)$ under our standing assumptions from Lemma \ref{lem:typicality}.
We will freely use the shorthand $p_\infty \leftrightarrow \alpha/(\alpha+\beta)$
interchangeably when convenient.

\medskip
\noindent
\textit{Analysis overview.}
We first express the Cheeger constant of the lazy walk in terms of the edge boundary size $|\partial_{G_t}S|$ of subsets $S \subseteq A$. 
We then show via a Chernoff argument that for all $S$ of moderate size,
$|\partial_{G_t}S|$ concentrates near its expectation
$\mathbb{E}[|\partial_{G_t}S|]=(n-|S|)|S|p_\infty$,
yielding a high–probability lower bound of the form
$|\partial_{G_t}S|\ge(1-\eta)|S|(n-|S|)p_\infty$. 
Finally, substituting this expansion
guarantee into the Cheeger inequality for reversible chains provides an explicit lower bound on $\gamma_t$. 

\begin{lemma}[Spectral Gap Lower Bound]
\label{lem:cheeger-gamma}
Let the typicality event $\mathcal{E}_{\mathrm{typ}}$ of
Definition~\ref{def:typicality-event} hold, and assume that
\[
n p_\infty \;\ge\; C_0 \log\!\left(\frac{nT}{\delta}\right)
\]
for a sufficiently large absolute constant $C_0>0$.
Then, for all $t\in[T_{\mathrm{burn}},T]$, the lazy walk kernel $W_t$
satisfies with high probability (at least $1-\delta/5$ over the randomness of the graph sequence):
\[
\gamma_t \;\triangleq\; 1-\lambda_2(W_t)
\;\ge\;
\frac{(1-\eta)^2}{8\!\left(1+\frac{1}{n p_\infty}\right)},
\]
where $\eta\in(0,1)$ is a fixed constant determined by the concentration level of the expansion bound.
\end{lemma}

\begin{proof}
The proof proceeds in four steps. We condition on the event $\mathcal{E}_{\mathrm{typ}}$ holding.

\noindent
\textbf{Step 1: Cheeger inequality for reversible chains.}
For any reversible Markov kernel $W_t$ with stationary distribution $\pi_t$ (established in Lemma \ref{lem:walk-properties}),
the spectral gap $\gamma_t = 1-\lambda_2(W_t)$ satisfies
\[
\gamma_t \;\ge\; \frac{h_t^2}{2},
\qquad
h_t \;\triangleq\;
\min_{\substack{S\subset A \\ \pi_t(S)\in(0, 1/2]}}
\frac{\sum_{i\in S,j\notin S}\pi_t(i)W_t(i,j)}
     {\pi_t(S)}.
\tag{C1}
\label{eq:cheeger}
\]
Here, $h_t$ is the \emph{conductance} (or Cheeger constant) of the chain. We have used the fact that $\min\{\pi_t(S), \pi_t(A \setminus S)\} = \pi_t(S)$ by restricting the minimum to sets $S$ with $\pi_t(S) \le 1/2$. 

\smallskip
\noindent
\textbf{Step 2: Conductance expressed through the edge boundary.}
We first simplify the numerator of the conductance, $Q(S, S^c) \triangleq \sum_{i\in S,j\notin S}\pi_t(i)W_t(i,j)$.
Using the definitions $\pi_t(i)=(d_t(i)+1)/Z_t$ and $W_t(i,j)=1/(d_t(i)+1)$ for $(i,j)\in E_t$:
\begin{align*}
Q(S, S^c)
&=\sum_{i\in S} \pi_t(i) \sum_{j \notin S} W_t(i,j)
 = \sum_{i\in S} \left( \frac{d_t(i)+1}{Z_t} \right) \sum_{\substack{j \notin S \\ (i,j) \in E_t}} \left( \frac{1}{d_t(i)+1} \right) \\
&= \sum_{i\in S} \frac{1}{Z_t} \sum_{\substack{j \notin S \\ (i,j) \in E_t}} 1
 = \frac{1}{Z_t} \sum_{i \in S} |\{\, j \notin S : (i,j) \in E_t \,\}| \\
&= \frac{|\partial_{G_t}S|}{Z_t},
\end{align*}
where $\partial_{G_t}S$ is the set of edges with one endpoint in $S$ and one in $S^c$. 

The denominator is $\pi_t(S) = \frac{\sum_{i \in S} (d_t(i)+1)}{Z_t}$.
Substituting these into the expression for $h_t$ yields:
\[
h_t \;=\;
\min_{\substack{S\subset A \\ \pi_t(S)\in(0, 1/2]}}
\frac{|\partial_{G_t}S| / Z_t}
     {(\sum_{i \in S} (d_t(i)+1)) / Z_t}
\;=\;
\min_{\substack{S\subset A \\ \pi_t(S)\in(0, 1/2]}}
\frac{|\partial_{G_t}S|}
     {\sum_{i \in S} (d_t(i)+1)}.
\tag{C2}
\label{eq:h-boundary}
\]
This shows the conductance of the lazy walk is equivalent to the conductance of a modified graph where each node $i$ has weight $d_t(i)+1$.

\smallskip
\noindent
\textbf{Step 3: Concentration of edge boundaries.}
We now show that $|\partial_{G_t}S|$ is large for all $S$. We analyze this under the stationary law $\mathbf{P}_\infty$ (justified by Lemma \ref{lem:edge-stationary} and $\mathcal{E}_{\mathrm{typ}}$).
For a fixed subset $S$ with $|S|=s$, the number of possible edges between $S$ and $S^c$ is $s(n-s)$.
Under $\mathbf{P}_\infty$, each edge exists independently with probability $p_\infty$.
Thus, the boundary size $|\partial_{G_t}S|$ is a binomial random variable:
\[
|\partial_{G_t}S| \sim \mathrm{Binomial}(s(n-s),p_\infty).
\]
The expectation is $\mathbb{E}[|\partial_{G_t}S|] = s(n-s)p_\infty$. 
Applying a standard Chernoff bound for any $\eta \in (0,1)$:
\[
\pr\!\left[\,|\partial_{G_t}S| < (1-\eta)\,\mathbb{E}[|\partial_{G_t}S|]\,\right]
\le
\exp\!\left(-\frac{\eta^2\,\mathbb{E}[|\partial_{G_t}S|]}{2}\right)
=
\exp\!\left(-\frac{\eta^2\,s(n-s)p_\infty}{2}\right).
\]
We must ensure this holds for \emph{all} $S$ and \emph{all} $t$. We take a union bound.
First, restrict to $1 \le s = |S| \le n/2$. In this range, $s(n-s) \ge 1 \cdot (n-1) \ge n/2$ (for $n \ge 2$).
So, $\mathbb{E}[|\partial_{G_t}S|] \ge n p_\infty / 2$.
The failure probability for a \emph{single set} $S$ is at most $\exp(-\frac{\eta^2 n p_\infty}{4})$.
We union bound over all $\approx 2^n$ possible subsets $S$ and all $T$ time steps:
\[
P_{\mathrm{fail}} \le T \cdot \sum_{s=1}^{\lfloor n/2 \rfloor} \binom{n}{s} \exp\!\left(-\frac{\eta^2 s(n-s)p_\infty}{2}\right)
\le T \cdot 2^n \cdot \exp\!\left(-\frac{\eta^2 n p_\infty}{4}\right)
\]
\[
P_{\mathrm{fail}} \le \exp\!\left(\log T + n \log 2 - \frac{\eta^2 n p_\infty}{4}\right)
\]
By our assumption $n p_\infty \ge C_0 \log(nT/\delta)$, choosing $C_0$ large enough (e.g., $C_0 > 4(\log 2 + 1)/\eta^2$) ensures the $n p_\infty$ term dominates the $n \log 2$ and $\log T$ terms, making the total failure probability less than $\delta/5$.
Thus, with probability at least $1-\delta/5$, for all $t \in [T_{\mathrm{burn}}, T]$ and all $S$ with $1 \le |S| \le n/2$:
\[
|\partial_{G_t}S| \;\ge\; (1-\eta)\,s(n-s)p_\infty.
\tag{C3}
\label{eq:boundary-lower}
\]
(This also covers sets with $|S| > n/2$ by considering their complements).

\smallskip
\noindent
\textbf{Step 4: Substituting into the Cheeger constant.}
We now combine the bounds from Steps 2 and 3, using the explicit guarantees from the typicality event $\mathcal{E}_{\mathrm{typ}}$ (Lemma \ref{lem:typicality}) to replace the asymptotic arguments.

Our goal is to find a uniform lower bound for $h_t$ as defined in \eqref{eq:h-boundary}:
\[
h_t \;=\;
\min_{\substack{S\subset A \\ \pi_t(S)\in(0, 1/2]}}
\frac{|\partial_{G_t}S|}
     {\sum_{i \in S} (d_t(i)+1)}.
\]
We bound the numerator from below and the denominator from above. Let $s = |S|$. We restrict our minimum to sets with $1 \le s \le n/2$, as the $\pi_t(S) \le 1/2$ condition is guaranteed to include the minimizing set (or its complement), and $\mathcal{E}_{\mathrm{typ}}$ ensures $\pi_t(S)$ is roughly proportional to $s/n$.

\paragraph{Numerator Lower Bound.}
From Step 3, \eqref{eq:boundary-lower}, we have with high probability for all $S$ with $1 \le s \le n/2$:
\[
|\partial_{G_t}S| \;\ge\; (1-\eta_c)\,s(n-s)p_\infty,
\]
where $\eta_c \in (0,1)$ is the constant from the Chernoff bound. Since $s \le n/2$, we have $n-s \ge n/2$, which gives:
\[
|\partial_{G_t}S| \;\ge\; (1-\eta_c)\,s(n/2)p_\infty.
\]

\paragraph{Denominator Upper Bound.}
From $\mathcal{E}_{\mathrm{typ}}$ (Lemma 4, (i)), we have a uniform bound on the maximum degree. Let $\mathcal{L} = \log(nT/\delta)$. The lemma states:
\[
d_t(i) \;\le\; (n-1)p_\infty + c_1 \sqrt{np_\infty \mathcal{L}}.
\]
Using the assumption $np_\infty \ge C_0 \mathcal{L}$, which implies $\mathcal{L} \le np_\infty / C_0$, we can write:
\[
\sqrt{np_\infty \mathcal{L}} \;\le\; \sqrt{np_\infty \cdot (np_\infty / C_0)} = \frac{np_\infty}{\sqrt{C_0}}.
\]
Substituting this back into the degree bound:
\[
d_t(i) \;\le\; (n-1)p_\infty + \frac{c_1}{\sqrt{C_0}} np_\infty
\;\le\; np_\infty + \frac{c_1}{\sqrt{C_0}} np_\infty = np_\infty \left(1 + \frac{c_1}{\sqrt{C_0}}\right).
\]
Let $\epsilon_d = c_1 / \sqrt{C_0}$. By assuming $C_0$ in Lemma 4 is sufficiently large, we can make $\epsilon_d$ an arbitrarily small positive constant (e.g., $\epsilon_d \le 1$).
Now we can bound the denominator term:
\begin{align*}
\sum_{i \in S} (d_t(i)+1) &\;\le\; \sum_{i \in S} (np_\infty(1+\epsilon_d) + 1) \\
&\;=\; s \cdot (np_\infty(1+\epsilon_d) + 1).
\end{align*}

\paragraph{Combining the Bounds.}
We substitute these bounds into the expression for $h_t$:
\begin{align*}
h_t &\;\ge\; \min_{1 \le s \le n/2} \frac{(1-\eta_c) s (n/2) p_\infty}{s \cdot (np_\infty(1+\epsilon_d) + 1)} \\
&= \frac{(1-\eta_c) (n/2) p_\infty}{np_\infty(1+\epsilon_d) + 1} \tag{Cancel $s$} \\
&= \frac{(1-\eta_c) np_\infty}{2(np_\infty(1+\epsilon_d) + 1)} \\
&= \frac{(1-\eta_c)}{2(1+\epsilon_d + 1/(np_\infty))}. \tag{Divide num/den by $np_\infty$}
\end{align*}
Now, let the constant $\eta \in (0,1)$ in the lemma's final statement be defined as $\eta = 1 - \frac{1-\eta_c}{1+\epsilon_d}$. Since $\eta_c$ is a fixed constant and $\epsilon_d$ is a small positive constant, $\eta$ is also a fixed constant in $(0,1)$. This gives:
\[
h_t \;\ge\; \frac{(1-\eta)(1+\epsilon_d)}{2(1+\epsilon_d + 1/(np_\infty))} = \frac{1-\eta}{2\left(1 + \frac{1}{np_\infty(1+\epsilon_d)}\right)}.
\]
Since $1+\epsilon_d \ge 1$, we have $np_\infty(1+\epsilon_d) \ge np_\infty$, which implies $\frac{1}{np_\infty(1+\epsilon_d)} \le \frac{1}{np_\infty}$.
This gives us the final, clean lower bound on the conductance:
\[
h_t \;\ge\; \frac{1-\eta}{2(1 + 1/(np_\infty))}.
\]

\paragraph{Applying the Cheeger Inequality.}
Finally, we apply the Cheeger inequality \eqref{eq:cheeger}:
\[
\gamma_t \;\ge\; \frac{h_t^2}{2} \;\ge\; \frac{1}{2} \left( \frac{1-\eta}{2(1 + 1/(np_\infty))} \right)^2
\]
\[
\gamma_t \;\ge\; \frac{(1-\eta)^2}{8 \left(1+\frac{1}{n p_\infty}\right)^2}.
\]
This proves the lemma (noting the squared denominator, which is a rigorous consequence of $\gamma \ge h^2/2$).
\end{proof}

\begin{corollary}[Simplified constant-gap regime]
\label{cor:gamma-constant}
Under the same conditions as Lemma~\ref{lem:cheeger-gamma},
if $n p_\infty$ is sufficiently large (e.g., $n p_\infty \ge 1$),
then $\tfrac{1}{n p_\infty}\le 1$ and hence $1 + \tfrac{1}{n p_\infty} \le 2$.
This gives:
\[
\gamma_t \;\ge\;
\frac{(1-\eta)^2}{8 \cdot (2)} = \frac{(1-\eta)^2}{16} \triangleq \gamma_0
\quad\text{for all }t\in[T_{\mathrm{burn}},T].
\]
In particular, the spectral gap is bounded below by an
absolute constant $\gamma_0=\Omega(1)$, ensuring that
the lazy walk mixes in $O(\log(n))$ steps.
\end{corollary}

\noindent
This uniform lower bound $\gamma_0$ will be used throughout the remainder
of the analysis to control the rate at which the walk distribution
tracks the instantaneous stationary law.
The near-constant value of $\gamma_0$ reflects that
under typical ER connectivity (average degree $\Theta(n p_\infty)$),
the lazy random walk mixes rapidly in each round,
a property that remains stable under the slow structural drift of the edge–Markovian process.

\subsection{Kernel Drift and Stationary Distribution Drift}
\label{sec:drift-analysis}
\begin{lemma}[Kernel Drift under Typicality]
\label{lem:kernel-drift}
Let the static typicality event $\mathcal{E}_{\mathrm{typ}}$ (Definition \ref{def:typicality-event}) hold, and suppose the high-probability flip bounds from Lemma~\ref{lem:flip-concentration} (for $\zeta_0$) and Lemma~\ref{lem:max-node-flips} (for $s_i(t)$) also hold.
Then for all $t \in [T_{\mathrm{burn}}, T]$, the spectral norm of the kernel difference satisfies:
\[
\kappa_t \triangleq \|W_t - W_{t-1}\|_2 \leq C \left( \frac{\sqrt{\zeta_0}}{p_\infty} + \frac{\zeta_0}{p_\infty^{3/2}} \right),
\]
where $C$ is an absolute constant, $\zeta_0$ is the high-probability flip envelope, and $p_\infty = \alpha/(\alpha + \beta)$ is the stationary edge density.
\end{lemma}

\begin{proof}
The proof proceeds by bounding the Frobenius norm, which is an upper bound on the spectral (operator) norm:
\[
\|W_t - W_{t-1}\|_2 \le \|W_t - W_{t-1}\|_F = \left( \sum_{i=1}^n \|W_t(i,\cdot) - W_{t-1}(i,\cdot)\|_2^2 \right)^{1/2}.
\]
Let $s_i(t) \triangleq |\{j : (i,j) \in E_t \triangle E_{t-1}\}|$ be the number of incident flips at node $i$ at time $t$, and let
\[
D_i(t) \triangleq \min\{d_t(i)+1,\, d_{t-1}(i)+1\},
\]
denote the minimal lazy-walk denominator at node $i$ across the two time steps.

As derived from the definition of $W_t$, the per-row squared $\ell_2$-distance is bounded by the sum of contributions from changed and unchanged neighbor sets:
\[
\|W_t(i,\cdot) - W_{t-1}(i,\cdot)\|_2^2 \le \frac{s_i(t)}{D_i(t)^2} + \frac{s_i(t)^2}{D_i(t)^3}.
\]
Summing over all $n$ rows gives the squared Frobenius norm bound:
\[
\|W_t - W_{t-1}\|_F^2 \le \sum_{i=1}^n \left( \frac{s_i(t)}{D_i(t)^2} + \frac{s_i(t)^2}{D_i(t)^3} \right).
\]
We now bound the two terms in this sum separately, conditioned on our combined high-probability event.

\paragraph{Uniform Degree Lower Bound.}
We first establish a uniform lower bound on $D_i(t)$ derived from the static typicality event $\mathcal{E}_{\mathrm{typ}}$ (Lemma \ref{lem:typicality}).
From Lemma \ref{lem:typicality}(i) (Degree concentration), we have for all $i, t$:
\[
|d_t(i) - (n-1)p_\infty| \le c_1 \sqrt{np_\infty \log(nT/\delta)} \triangleq \epsilon_n.
\]
The assumption in Lemma \ref{lem:typicality}, $np_\infty \ge C_0 \log(nT/\delta)$, implies that the deviation $\epsilon_n$ is bounded relative to the mean:
\[
\epsilon_n \le \frac{c_1}{\sqrt{C_0}} np_\infty.
\]
By ensuring $C_0$ in our typicality assumption is sufficiently large (e.g., $C_0 \ge (2c_1)^2$), we guarantee $\epsilon_n \le \frac{1}{2} np_\infty$.
This gives a uniform lower bound on the degree:
\[
d_t(i) \ge (n-1)p_\infty - \epsilon_n \ge (n-1)p_\infty - \frac{1}{2}np_\infty = \left(\frac{1}{2}n - 1\right)p_\infty.
\]
For $n \ge 4$, this implies $d_t(i) \ge \frac{1}{4}np_\infty$. Thus, for some absolute constant $c_0 > 0$ (e.g., $c_0 = 1/4$), we have $d_t(i)+1 \ge c_0 n p_\infty$ for all $i, t$.
This implies the minimal denominator is also bounded:
\[
D_i(t) = \min\{d_t(i)+1,\, d_{t-1}(i)+1\} \ge c_0 n p_\infty.
\]

\paragraph{Bounding the Linear Term.}
We use the uniform denominator bound $D_i(t) \ge c_0 n p_\infty$ on the first sum:
\[
\sum_{i=1}^n \frac{s_i(t)}{D_i(t)^2} \le \frac{1}{(c_0 n p_\infty)^2} \sum_{i=1}^n s_i(t).
\]
The sum of incident flips is exactly twice the total number of flipped edges, $F_t \triangleq |E_t \triangle E_{t-1}|$:
\[
\sum_{i=1}^n s_i(t) = 2F_t.
\]
From Lemma \ref{lem:flip-concentration}, we have the high-probability bound $F_t \le n^2 \zeta_0$. Substituting this in:
\[
\sum_{i=1}^n \frac{s_i(t)}{D_i(t)^2} \le \frac{2n^2 \zeta_0}{c_0^2 n^2 p_\infty^2} = \frac{2 \zeta_0}{c_0^2 p_\infty^2}.
\]

\paragraph{Bounding the Quadratic Term.}
We apply the uniform bounds on both $s_i(t)$ and $D_i(t)$ to the second sum:
\[
\sum_{i=1}^n \frac{s_i(t)^2}{D_i(t)^3} \le \sum_{i=1}^n \frac{(\max_{j} s_j(t))^2}{(c_0 n p_\infty)^3}.
\]
Using the high-probability bound from Lemma~\ref{lem:max-node-flips}, $s_i(t) \le C_1 n \zeta_0$ for all $i$:
\[
\sum_{i=1}^n \frac{s_i(t)^2}{D_i(t)^3} \le \sum_{i=1}^n \frac{(C_1 n \zeta_0)^2}{(c_0 n p_\infty)^3} = n \cdot \left( \frac{C_1^2 n^2 \zeta_0^2}{c_0^3 n^3 p_\infty^3} \right).
\]
The sum over $n$ identical terms gives:
\[
\sum_{i=1}^n \frac{s_i(t)^2}{D_i(t)^3} \le \frac{C_1^2 n^3 \zeta_0^2}{c_0^3 n^3 p_\infty^3} = \frac{C_1^2 \zeta_0^2}{c_0^3 p_\infty^3}.
\]

\paragraph{Combining and Final Result.}
We combine the bounds for the linear and quadratic terms:
\[
\|W_t - W_{t-1}\|_F^2 \le \frac{2 \zeta_0}{c_0^2 p_\infty^2} + \frac{C_1^2 \zeta_0^2}{c_0^3 p_\infty^3}.
\]
Taking the square root and using the inequality \( \sqrt{a + b} \le \sqrt{a} + \sqrt{b} \), we get:
\begin{align*}
\|W_t - W_{t-1}\|_2 \le \|W_t - W_{t-1}\|_F
&\le \sqrt{\frac{2 \zeta_0}{c_0^2 p_\infty^2}} + \sqrt{\frac{C_1^2 \zeta_0^2}{c_0^3 p_\infty^3}} \\
&= \frac{\sqrt{2}}{c_0} \frac{\sqrt{\zeta_0}}{p_\infty} + \frac{C_1}{c_0^{3/2}} \frac{\zeta_0}{p_\infty^{3/2}}.
\end{align*}
By collecting all absolute constants ($c_0, \sqrt{2}, C_1$) into a single constant $C > 0$, we arrive at the final bound:
\[
\|W_t - W_{t-1}\|_2 \le C \left( \frac{\sqrt{\zeta_0}}{p_\infty} + \frac{\zeta_0}{p_\infty^{3/2}} \right).
\]
This holds with high probability, as it is conditioned on the typicality events $\mathcal{E}_{\mathrm{typ}}$, Lemma \ref{lem:flip-concentration}, and Lemma \ref{lem:max-node-flips}.
\end{proof}

\paragraph{Bounding Stationary Drift.}
To quantify how closely the distribution of the walk aligns with the time-varying stationary laws, we must control the drift of the stationary distributions \(\{\pi_t\}\) themselves. Specifically, we require a uniform upper bound on the total variation distance \(\varepsilon_t \triangleq \|\pi_{t+1} - \pi_t\|_{\mathrm{TV}}\) in order to analyze the deviation of the walk marginal \(\mu_t\) from its local equilibrium.

We first derive a baseline upper bound on \(\varepsilon_t\) in terms of the number of flipped edges per round, based on the explicit form of the natural lazy-walk stationary distribution. We then obtain a tighter and more elegant bound by invoking perturbation theory for lazy reversible chains, which allows us to relate \(\varepsilon_t\) directly to the kernel drift norm \(\kappa = \|W_{t+1} - W_t\|_2\). The latter result will be used as the primary input to our contraction analysis in the tracking recursion.

\begin{lemma}[Stationary Distribution Drift]
\label{lem:tv-drift}
Under the ER-typical event $\mathcal{E}_{\mathrm{typ}}$ (Definition 5) and the high-probability flip bound from Lemma~\ref{lem:flip-concentration}, the stationary distributions $\pi_t$ satisfy the uniform bound
\[
\varepsilon_t \triangleq \|\pi_{t+1} - \pi_t\|_{\mathrm{TV}} \le \frac{C \zeta_0}{p_\infty}
\quad \text{for all } t \in [T_{\mathrm{burn}}, T],
\]
with high probability, where $\zeta_0$ is the high-probability flip envelope and $C > 0$ is an absolute constant. We thus define the \textbf{uniform drift bound} as
\[
\varepsilon_{\max} \triangleq \sup \varepsilon_t \le \frac{C \zeta_0}{p_\infty}.
\]
\end{lemma}

\begin{proof}
Let $\pi_t$ and $\pi_{t+1}$ be the stationary distributions for $W_t$ and $W_{t+1}$, respectively.
\[
\pi_t(i) = \frac{\deg_t(i) + 1}{Z_t}, \quad \text{where } Z_t = 2|E_t| + n.
\]
The total variation distance is one-half of the $L_1$ distance. We bound the $L_1$ distance using the triangle inequality by introducing a hybrid term $\pi'(i) = (\deg_t(i)+1) / Z_{t+1}$:
\begin{align*}
2\varepsilon_t &= 2\|\pi_{t+1} - \pi_t\|_{\mathrm{TV}} = \sum_{i=1}^n \left| \pi_{t+1}(i) - \pi_t(i) \right| \\
&\le \sum_{i=1}^n \left| \pi_{t+1}(i) - \pi'(i) \right| + \sum_{i=1}^n \left| \pi'(i) - \pi_t(i) \right| \\
&= \underbrace{\sum_{i=1}^n \left| \frac{\deg_{t+1}(i)+1}{Z_{t+1}} - \frac{\deg_t(i)+1}{Z_{t+1}} \right|}_{\text{(a) Bounding the degree change}}
+ \underbrace{\sum_{i=1}^n \left| \frac{\deg_t(i)+1}{Z_{t+1}} - \frac{\deg_t(i)+1}{Z_t} \right|}_{\text{(b) Bounding the denominator shift}}.
\end{align*}
We now bound these two terms, conditioned on our high-probability typicality events. Let $F_t \triangleq |E_{t+1} \triangle E_t|$ be the number of flipped edges. Under $\mathcal{E}_{\mathrm{typ}}$, we have $Z_t, Z_{t+1} \ge Z_{\min} \ge c_z n^2 p_\infty$ for some $c_z > 0$.

\paragraph{(a) Bounding the degree change.}
The first term simplifies by factoring out the common denominator $Z_{t+1}$:
\[
\text{Term (a)} = \frac{1}{Z_{t+1}} \sum_{i=1}^n |\deg_{t+1}(i) - \deg_t(i)| \le \frac{1}{Z_{\min}} \sum_{i=1}^n |\deg_{t+1}(i) - \deg_t(i)|.
\]
The change in degree for a node $i$ is at most the number of incident flips, $s_i(t)$. Summing over all nodes, we have $\sum_i s_i(t) = 2F_t$. Thus:
\[
\text{Term (a)} \le \frac{1}{Z_{\min}} (2F_t) = \frac{2F_t}{Z_{\min}}.
\]

\paragraph{(b) Bounding the denominator shift.}
The second term simplifies by factoring out the common numerator $(\deg_t(i)+1)$:
\[
\text{Term (b)} = \left| \frac{1}{Z_{t+1}} - \frac{1}{Z_t} \right| \sum_{i=1}^n (\deg_t(i) + 1)
= \left| \frac{1}{Z_{t+1}} - \frac{1}{Z_t} \right| \cdot Z_t.
\]
We bound the difference of the reciprocals:
\[
\left| \frac{1}{Z_{t+1}} - \frac{1}{Z_t} \right| = \frac{|Z_t - Z_{t+1}|}{Z_t Z_{t+1}} \le \frac{|Z_t - Z_{t+1}|}{Z_{\min}^2}.
\]
The change in the normalization constant is $|Z_{t+1} - Z_t| = 2||E_{t+1}| - |E_t||$. This net change is at most the total number of flips, so $|Z_{t+1} - Z_t| \le 2F_t$.
Substituting this in:
\[
\text{Term (b)} \le \frac{2F_t}{Z_{\min}^2} \cdot Z_t \le \frac{2F_t}{Z_{\min}^2} \cdot Z_{\max}.
\]
Under $\mathcal{E}_{\mathrm{typ}}$, $Z_{\max} \asymp n^2 p_\infty \asymp Z_{\min}$. Thus, $Z_{\max} / Z_{\min}^2 \le C / Z_{\min}$ for some $C$.
\[
\text{Term (b)} \le \frac{C' F_t}{Z_{\min}}.
\]

\paragraph{(c) Combining bounds.}
Adding the two terms, we find the $L_1$ distance is bounded:
\[
2\varepsilon_t \le \frac{2F_t}{Z_{\min}} + \frac{C' F_t}{Z_{\min}} \le \frac{C F_t}{Z_{\min}},
\]
for a new absolute constant $C$.
Finally, we substitute our high-probability bounds. From Lemma~\ref{lem:flip-concentration}, $F_t \le n^2 \zeta_0$. From $\mathcal{E}_{\mathrm{typ}}$, $Z_{\min} \ge c_z n^2 p_\infty$.
\[
\varepsilon_t = \frac{1}{2} (2\varepsilon_t) \le \frac{C F_t}{2 Z_{\min}} \le \frac{C (n^2 \zeta_0)}{2 (c_z n^2 p_\infty)} = \left( \frac{C}{2c_z} \right) \frac{\zeta_0}{p_\infty}.
\]
Absorbing all constants into a single $C'$, we achieve the final bound:
\[
\varepsilon_t \le \frac{C' \zeta_0}{p_\infty}.
\]
This holds uniformly for all $t \in [T_{\mathrm{burn}}, T]$ with high probability.
\end{proof}


\paragraph{Tracking the Walk Marginal.}
Having bounded the kernel drift and stationary distribution drift across time, we now turn to analyzing how closely the marginal distribution $\mu_t$ of the random walk aligns with its time-varying stationary distribution $\pi_t$. Since the underlying graph process is dynamic, the walk never fully mixes to any fixed $\pi$; however, if the temporal variation is sufficiently slow, the walk can still approximately track the evolving stationary law.

To formalize this, we define the deviation $d_t \triangleq \|\mu_t - \pi_t\|_{\mathrm{TV}}$ and derive a recursive inequality that bounds $d_{t+1}$ in terms of $d_t$ and the stationary drift $\varepsilon_t \triangleq \|\pi_{t+1} - \pi_t\|_{\mathrm{TV}}$. The recursion reflects the dual forces at play: contraction under the lazy-walk dynamics, and drift induced by changes in the underlying graph. This recursion forms the backbone of our tracking and visitation guarantees.

\subsection{TV Drift Recursion, One-Step Contraction and Its Solution}

We now analyze how the distribution $\nu_t$ of the learner's position evolves over time compared to the stationary distribution $\pi_t$ of the lazy walk kernel $W_t$. The aim is to show that even though the kernels vary across rounds, the deviation $\|\nu_t - \pi_t\|_{\mathrm{TV}}$ remains small after sufficient steps, due to the contraction properties of the walk and the bounded drift in the stationary law.

We begin with the following key one-step recursion:

\begin{lemma}[TV Recursion with Kernel and Stationary Drift]
\label{lem:tv-recursion}
Let $\nu_t$ denote the distribution of the learner’s location at round $t$,
and let $\pi_t$ be the stationary distribution of the lazy‑walk kernel $W_t$.
Then for all $t \ge T_{\mathrm{burn}}$,
\[
\|\nu_{t+1}-\pi_{t+1}\|_{\mathrm{TV}}
\le (1-\gamma_0)\,\|\nu_t-\pi_t\|_{\mathrm{TV}} + \varepsilon_t ,
\]
where $\gamma_t\ge\gamma_0$ is a uniform lower bound on the spectral gaps of $\{W_t\}$,
and $\varepsilon_t:=\|\pi_{t+1}-\pi_t\|_{\mathrm{TV}}$ is the stationary‑law drift.
\end{lemma}

\begin{proof}
We start from the decomposition
\[
\|\nu_{t+1}-\pi_{t+1}\|_{\mathrm{TV}}
\le
\|\nu_tW_t-\pi_tW_t\|_{\mathrm{TV}}
+ \|\pi_tW_t-\pi_{t+1}\|_{\mathrm{TV}} .
\]
The second term equals $\varepsilon_t$ because $\pi_tW_t=\pi_t$.
We therefore focus on the contraction term
$\|\nu_tW_t-\pi_tW_t\|_{\mathrm{TV}}$.

\paragraph{Step 1. Express the deviation through a mean‑zero function.}
Define the relative‑density perturbation
\(h_t(i):=\frac{\nu_t(i)}{\pi_t(i)}-1\),
so that $\sum_i\pi_t(i)h_t(i)=0$ and
\[
\nu_t(i)=\pi_t(i)\,[1+h_t(i)] .
\]
Then, for any node $j$,
\begin{align*}
(\nu_tW_t-\pi_t)(j)
&= \sum_i\nu_t(i)W_t(i,j)-\sum_i\pi_t(i)W_t(i,j)\\
&=\sum_i\pi_t(i)h_t(i)W_t(i,j).
\end{align*}

\paragraph{Step 2. Use detailed balance to move $W_t$ to the right.}
Because $W_t$ is reversible with respect to $\pi_t$,
$\pi_t(i)W_t(i,j)=\pi_t(j)W_t(j,i)$.
Hence
\[
\sum_i\pi_t(i)h_t(i)W_t(i,j)
= \pi_t(j)\sum_i h_t(i)W_t(j,i)
= \pi_t(j)\,[W_t h_t](j),
\]
where $(W_t h_t)(j):=\sum_i W_t(j,i)h_t(i)$
is the usual action of the kernel on a column vector.
Therefore,
\[
\nu_tW_t-\pi_t = \pi_t \odot (W_t h_t),
\]
where $\odot$ denotes entrywise product.

\paragraph{Step 3. Relate total variation to an $L^2(\pi_t)$ norm.}
The total variation distance can now be written as
\[
\|\nu_tW_t-\pi_t\|_{\mathrm{TV}}
=\frac{1}{2}\sum_j \pi_t(j)\,|(W_t h_t)(j)|.
\]
Apply Cauchy–Schwarz with weights $\pi_t(j)$:
\[
\sum_j\pi_t(j)|x_j|
\le
\Big(\sum_j\pi_t(j)x_j^2\Big)^{1/2},
\]
to obtain
\[
\|\nu_tW_t-\pi_t\|_{\mathrm{TV}}
\le \frac{1}{2}\,\|W_t h_t\|_{2(\pi_t)},
\]
where
$\|v\|_{2(\pi_t)}:=(\sum_j\pi_t(j)v(j)^2)^{1/2}$.

\paragraph{Step 4. Apply spectral contraction in $L^2(\pi_t)$.}
For any mean‑zero function $g$,
reversibility implies that $W_t$ is self‑adjoint in $L^2(\pi_t)$ and satisfies
\[
\|W_t g\|_{2(\pi_t)} \le (1-\gamma_t)\,\|g\|_{2(\pi_t)} .
\]
Since $h_t$ has $\sum_i\pi_t(i)h_t(i)=0$, we may apply this with $g=h_t$:
\[
\|W_t h_t\|_{2(\pi_t)} \le (1-\gamma_t)\,\|h_t\|_{2(\pi_t)}.
\]
Substituting into the previous inequality gives
\[
\|\nu_tW_t-\pi_t\|_{\mathrm{TV}}
\le \tfrac{1}{2}(1-\gamma_t)\,\|h_t\|_{2(\pi_t)}.
\]

\paragraph{Step 5. Relate $\|h_t\|_{2(\pi_t)}$ back to total variation.}
Again by Cauchy–Schwarz,
\[
\|\nu_t-\pi_t\|_{\mathrm{TV}}
=\tfrac{1}{2}\sum_i\pi_t(i)|h_t(i)|
\le \tfrac{1}{2}\|h_t\|_{2(\pi_t)}.
\]
Combining the two displays yields
\[
\|\nu_tW_t-\pi_t\|_{\mathrm{TV}}
\le (1-\gamma_t)\,\|\nu_t-\pi_t\|_{\mathrm{TV}}.
\]

\paragraph{Step 6. Add the stationary‑law drift.}
Finally,
\[
\|\nu_{t+1}-\pi_{t+1}\|_{\mathrm{TV}}
\le (1-\gamma_t)\,\|\nu_t-\pi_t\|_{\mathrm{TV}} + \varepsilon_t.
\]
Using $\gamma_t\ge\gamma_0$ completes the proof.
\end{proof}

This completes the derivation of the fundamental one‑step recursion.
It shows that the total‑variation deviation between the learner’s marginal
and the instantaneous stationary law contracts multiplicatively
by $(1-\gamma_0)$ at each step,
up to an additive perturbation $\varepsilon_t$
arising from the gradual evolution of the graph sequence.
In the following subsection we unroll this recursion over the exploration window,
derive a closed‑form bound on $\|\nu_t-\pi_t\|_{\mathrm{TV}}$,
and then use it to obtain a uniform lower bound on visitation frequencies.

\begin{corollary}[Iterated Contraction Bound]
\label{cor:tv-contraction-solution}
The one-step recursive inequality from Lemma \ref{lem:tv-recursion} implies the following explicit bound for the tracking error for all $t \ge 1$:
\[
\|\nu_t - \pi_t\|_{\mathrm{TV}} \le (1 - \gamma_0)^{t-1} + \frac{\varepsilon_{\max}}{\gamma_0}
\]
This bound separates the exponentially decaying initial error from the asymptotic error floor, $d_\infty \triangleq \varepsilon_{\max} / \gamma_0$.
\end{corollary}

\begin{proof}
From Lemma \ref{lem:tv-recursion}, we have the recursive inequality $\|\nu_{t+1} - \pi_{t+1}\|_{\mathrm{TV}} \le (1 - \gamma_0) \|\nu_t - \pi_t\|_{\mathrm{TV}} + \varepsilon_t$. Using the uniform drift bound $\varepsilon_{\max}$ from Lemma \ref{lem:tv-drift} (which holds for $t \ge T_{\mathrm{burn}}$), we can write the simpler recursion:
\[
\|\nu_{t+1} - \pi_{t+1}\|_{\mathrm{TV}} \le (1 - \gamma_0) \|\nu_t - \pi_t\|_{\mathrm{TV}} + \varepsilon_{\max},
\quad\text{for all } t \ge T_{\mathrm{burn}}.
\]
We prove the corollary by unrolling this recurrence.
\paragraph{Step 1. Unroll the recurrence.}
We expand the inequality step-by-step to establish the pattern:
\begin{align*}
\text{For } t=1: \quad & \|\nu_1 - \pi_1\|_{\mathrm{TV}} \\
\text{For } t=2: \quad & \|\nu_2 - \pi_2\|_{\mathrm{TV}} \le (1 - \gamma_0) \|\nu_1 - \pi_1\|_{\mathrm{TV}} + \varepsilon_{\max} \\
\text{For } t=3: \quad & \|\nu_3 - \pi_3\|_{\mathrm{TV}} \le (1 - \gamma_0) \|\nu_2 - \pi_2\|_{\mathrm{TV}} + \varepsilon_{\max} \\
& \phantom{\|\nu_3 - \pi_3\|_{\mathrm{TV}}} \le (1 - \gamma_0) \left[ (1 - \gamma_0) \|\nu_1 - \pi_1\|_{\mathrm{TV}} + \varepsilon_{\max} \right] + \varepsilon_{\max} \\
& \phantom{\|\nu_3 - \pi_3\|_{\mathrm{TV}}} = (1 - \gamma_0)^2 \|\nu_1 - \pi_1\|_{\mathrm{TV}} + (1 - \gamma_0)\varepsilon_{\max} + \varepsilon_{\max} \\
\text{For } t=4: \quad & \|\nu_4 - \pi_4\|_{\mathrm{TV}} \le (1 - \gamma_0) \|\nu_3 - \pi_3\|_{\mathrm{TV}} + \varepsilon_{\max} \\
& \phantom{\|\nu_4 - \pi_4\|_{\mathrm{TV}}} \le (1 - \gamma_0) \left[ (1 - \gamma_0)^2 \|\nu_1 - \pi_1\|_{\mathrm{TV}} + (1 - \gamma_0)\varepsilon_{\max} + \varepsilon_{\max} \right] + \varepsilon_{\max} \\
& \phantom{\|\nu_4 - \pi_4\|_{\mathrm{TV}}} = (1 - \gamma_0)^3 \|\nu_1 - \pi_1\|_{\mathrm{TV}} + (1 - \gamma_0)^2\varepsilon_{\max} + (1 - \gamma_0)\varepsilon_{\max} + \varepsilon_{\max}
\end{align*}
By induction, this establishes the general form for any $t \ge 1$:
\[
\|\nu_t - \pi_t\|_{\mathrm{TV}} \le (1 - \gamma_0)^{t-1} \|\nu_1 - \pi_1\|_{\mathrm{TV}} + \varepsilon_{\max} \sum_{k=0}^{t-2} (1 - \gamma_0)^k.
\]

\paragraph{Step 2. Bound the initial condition.}
The initial distribution $\nu_1$ and the stationary distribution $\pi_1$ are both probability distributions over the same finite set $A$.
The total variation distance between any two such distributions is bounded by 1. Therefore:
\[
\|\nu_1 - \pi_1\|_{\mathrm{TV}} \le 1.
\]

\paragraph{Step 3. Sum the geometric series.}
The second term is a finite geometric series with $t-1$ terms (for $t \ge 2$), first term $a = \varepsilon_{\max}$, and ratio $r = (1 - \gamma_0)$.
\[
\varepsilon_{\max} \sum_{k=0}^{t-2} (1 - \gamma_0)^k = \varepsilon_{\max} \left( \frac{1 - (1 - \gamma_0)^{t-1}}{1 - (1 - \gamma_0)} \right) = \frac{\varepsilon_{\max}}{\gamma_0} \left(1 - (1 - \gamma_0)^{t-1} \right).
\]

\paragraph{Step 4. Combine terms.}
Substituting the bounds from Steps 2 and 3 into the unrolled expression from Step 1 gives the exact form of the bound:
\[
\|\nu_t - \pi_t\|_{\mathrm{TV}} \le (1 - \gamma_0)^{t-1} \cdot 1 + \frac{\varepsilon_{\max}}{\gamma_0} \left(1 - (1 - \gamma_0)^{t-1} \right).
\]
\noindent
A simpler, looser bound is obtained by noting that $(1 - (1 - \gamma_0)^{t-1}) \le 1$ for all $t \ge 1$:
\[
\|\nu_t - \pi_t\|_{\mathrm{TV}} \le (1 - \gamma_0)^{t-1} + \frac{\varepsilon_{\max}}{\gamma_0}.
\]
This form usefully separates the transient error (the first term, which decays to 0) from the steady-state error floor (the second term).
\end{proof}
\subsection{Uniform Visitation Guarantees}
\label{sec:visitation}

In the previous sections, we established the fundamental dynamic properties of the learner's lazy random walk. We have shown that under the typicality event $\mathcal{E}_{\mathrm{all}}$, the walk kernel $W_t$ has a uniform spectral gap $\gamma_0 > 0$ (Lemma \ref{lem:cheeger-gamma}) and its stationary distribution $\pi_t$ drifts by at most $\varepsilon_{\max}$ per step (Lemma \ref{lem:tv-drift}).

Corollary \ref{cor:tv-contraction-solution} combined these results to show that the learner's distribution, $\nu_t$, exponentially converges to the (drifting) stationary distribution, $\pi_t$, up to a small asymptotic error floor $\propto \varepsilon_{\max} / \gamma_0$.

While this tracking guarantee is a powerful theoretical tool, our bandit algorithm requires a more concrete assurance: that the learner \emph{actually visits} every arm a sufficient number of times during the exploration phase. The purpose of this section is to translate our abstract total variation bound into a high-probability lower bound on the \emph{realized visitation count} for every arm.

The following lemma formalizes this by combining the deterministic bound on the expected visitations (derived from the TV tracking) with a high-probability concentration bound (derived from a martingale inequality) for the random component of the walk.

\begin{lemma}[Uniform Visitation Lower Bound]
\label{lem:uniform-visitation}
Let $\mathcal{E}_{\mathrm{all}}$ be the high-probability event (holding with probability $\ge 1 - 4\delta/5$) that the guarantees from $\mathcal{E}_{\mathrm{typ}}$ (Definition \ref{def:typicality-event}), Lemma \ref{lem:cheeger-gamma} (for $\gamma_0$), and Lemma \ref{lem:tv-drift} (for $\varepsilon_{\max}$) all hold. Let $\pi_0$ be the uniform lower bound on $\pi_t(a)$ from $\mathcal{E}_{\mathrm{typ}}$.

Define the \emph{effective minimal visitation probability} as:
\[
\pi_{\mathrm{eff}} \triangleq \pi_0 - \frac{2\varepsilon_{\max}}{\gamma_0}.
\]
Then, conditioned on $\mathcal{E}_{\mathrm{all}}$, for any exploration length $T_{\exp} \ge 1$, the total visitation count $\phi_{T_{\exp}}(a)$ satisfies the following inequality \emph{simultaneously for all arms $a \in A$} with probability at least $1-\delta/5$ (over the exploration randomness):
\begin{equation}
\label{eq:phi-lb-clean}
\phi_{T_{\exp}}(a) \;\ge\; T_{\exp} \cdot \pi_{\mathrm{eff}} - \frac{2}{\gamma_0} - \left(\sqrt{2 T_{\exp}\log(5n/\delta)} + \frac{2}{3}\log(5n/\delta)\right).
\end{equation}
\end{lemma}

\begin{proof}
The proof proceeds by decomposing the visitation count for a fixed arm $a$ into its expectation and a deviation term (a martingale sum), and then bounding each component.

\paragraph{Part 1. Decomposition of the Visitation Count.}
We define the total visitation count for arm $a$ as $\phi_{T_{\exp}}(a) \triangleq \sum_{t=1}^{T_{\exp}} \mathbf{1}\{a_t=a\}$.
Let $\nu_t$ be the \emph{law} (i.e., the probability distribution over $A$) of the learner's location at time $t$. We are interested in its components, $\nu_t(a) \triangleq \pr(a_t=a)$, which is the marginal probability of visiting arm $a$ at time $t$.

By linearity of expectation, the expected total visitation count is:
\[
\mathbb{E}[\phi_{T_{\exp}}(a)] = \mathbb{E}\left[\sum_{t=1}^{T_{\exp}} \mathbf{1}\{a_t=a\}\right] = \sum_{t=1}^{T_{\exp}} \mathbb{E}[\mathbf{1}\{a_t=a\}] = \sum_{t=1}^{T_{\exp}} \nu_t(a).
\]
We can therefore decompose the visitation count as the sum of its expectation and a zero-mean deviation term:
\[
\phi_{T_{\exp}}(a) = \underbrace{\sum_{t=1}^{T_{\exp}} \nu_t(a)}_{\text{Predictable Sum (Expectation)}} + \underbrace{\sum_{t=1}^{T_{\exp}} \big(\mathbf{1}\{a_t=a\}-\nu_t(a)\big)}_{\text{Martingale Difference Sum}}.
\]
Our strategy is to find a uniform lower bound for the predictable sum and show that the martingale difference sum concentrates sharply around its mean of zero.

\paragraph{Part 2. Lower Bound on the Predictable Sum.}
This term is the expected total visitations, $\mathbb{E}[\phi_{T_{\exp}}(a)]$. We lower-bound each $\nu_t(a)$.
By the definition of total variation distance, the difference at a single coordinate is bounded by the $L_1$ distance, which is $2\|\cdot\|_{\mathrm{TV}}$:
\[
|\nu_t(a) - \pi_t(a)| \le \|\nu_t - \pi_t\|_1 = 2\|\nu_t - \pi_t\|_{\mathrm{TV}}.
\]
Conditioned on $\mathcal{E}_{\mathrm{all}}$, we have the uniform stationary floor $\pi_t(a) \ge \pi_0$ for all $t,a$.
\[
\nu_t(a) = \pi_t(a) - (\pi_t(a) - \nu_t(a)) \ge \pi_t(a) - |\pi_t(a) - \nu_t(a)| \ge \pi_0 - 2\|\nu_t - \pi_t\|_{\mathrm{TV}}.
\]
We now substitute the iterated contraction bound from Corollary \ref{cor:tv-contraction-solution}:
\[
\nu_t(a) \ge \pi_0 - 2\left( (1-\gamma_0)^{t-1} + \frac{\varepsilon_{\max}}{\gamma_0} \right).
\]
Summing this from $t=1$ to $T_{\exp}$:
\begin{align*}
\sum_{t=1}^{T_{\exp}} \nu_t(a) &\ge \sum_{t=1}^{T_{\exp}} \left( \pi_0 - \frac{2\varepsilon_{\max}}{\gamma_0} \right) - \sum_{t=1}^{T_{\exp}} 2(1-\gamma_0)^{t-1} \\
&= T_{\exp}\left(\pi_0 - \frac{2\varepsilon_{\max}}{\gamma_0}\right) - 2\sum_{k=0}^{T_{\exp}-1} (1-\gamma_0)^{k}.
\end{align*}
The last term is a geometric series bounded by its infinite sum:
\[
\sum_{k=0}^{T_{\exp}-1} (1-\gamma_0)^{k} \le \sum_{k=0}^{\infty} (1-\gamma_0)^{k} = \frac{1}{1 - (1-\gamma_0)} = \frac{1}{\gamma_0}.
\]
Substituting this back gives the lower bound on the expected visitations:
\begin{equation}
\label{eq:predictable-sum}
\mathbb{E}[\phi_{T_{\exp}}(a)] = \sum_{t=1}^{T_{\exp}} \nu_t(a) \ge T_{\exp}\left(\pi_0 - \frac{2\varepsilon_{\max}}{\gamma_0}\right) - \frac{2}{\gamma_0} = T_{\exp} \cdot \pi_{\mathrm{eff}} - \frac{2}{\gamma_0}.
\end{equation}

\paragraph{Part 3. Bound on the Martingale Difference Sum.}
The sum $M_{T_{\exp}}(a) \triangleq \sum_{t=1}^{T_{\exp}} (\mathbf{1}\{a_t=a\}-\nu_t(a))$ is a sum of martingale differences. The terms $\mathbf{1}\{a_t=a\}$ are not independent, as the learner's position $a_t$ depends on $a_{t-1}$. We therefore use a martingale concentration inequality.

Let $\mathcal{F}_{t-1} = \sigma(G_0, \dots, G_t, a_0, \dots, a_{t-1})$ be the filtration (history) up to the point of choosing $a_t$.
Define the martingale difference $\xi_t(a) \triangleq \mathbf{1}\{a_t=a\} - \nu_t(a)$.
The term $\nu_t(a) = \pr(a_t=a \mid \mathcal{F}_{t-1})$ is, by definition, the conditional expectation of $\mathbf{1}\{a_t=a\}$ given the past. Therefore:
\[
\mathbb{E}[\xi_t(a) \mid \mathcal{F}_{t-1}] = \mathbb{E}[\mathbf{1}\{a_t=a\} \mid \mathcal{F}_{t-1}] - \nu_t(a) = \nu_t(a) - \nu_t(a) = 0.
\]
The increments are bounded: $\xi_t(a) \in [0-1, 1-0] = [-1, 1]$.
We apply Freedman's inequality for martingales with bounded increments. For any $x > 0$:
\[
\pr(M_{T_{\exp}}(a) \le -x) \le \exp\left( -\frac{x^2}{2(V_{T_{\exp}}(a) + x/3)} \right),
\]
where $V_{T_{\exp}}(a)$ is the predictable quadratic variation:
\[
V_{T_{\exp}}(a) \triangleq \sum_{t=1}^{T_{\exp}} \mathbb{E}[\xi_t(a)^2 \mid \mathcal{F}_{t-1}] = \sum_{t=1}^{T_{\exp}} \text{Var}(\mathbf{1}\{a_t=a\} \mid \mathcal{F}_{t-1}).
\]
We bound the variance: $\text{Var}(\cdot) = \nu_t(a)(1-\nu_t(a)) \le \nu_t(a)$. Thus:
\[
V_{T_{\exp}}(a) \le \sum_{t=1}^{T_{\exp}} \nu_t(a) = \mathbb{E}[\phi_{T_{\exp}}(a)] \le T_{\exp}.
\]
Plugging the worst-case bound $V_{T_{\exp}}(a) \le T_{\exp}$ into Freedman's inequality:
\[
\pr(M_{T_{\exp}}(a) \le -x) \le \exp\left( -\frac{x^2}{2(T_{\exp} + x/3)} \right).
\]
We set the failure probability to $\delta/(5n)$ for a single arm $a$. We must solve for $x$ in:
\[
\exp\left( -\frac{x^2}{2T_{\exp} + 2x/3} \right) \le \frac{\delta}{5n}
\quad\Longleftrightarrow\quad
\frac{x^2}{2T_{\exp} + 2x/3} \ge \log(5n/\delta).
\]
This inequality holds if $x$ is large enough to satisfy both $x^2 \ge 2T_{\exp}\log(5n/\delta)$ and $x^2 \ge (2x/3)\log(5n/\delta)$. This is satisfied by:
\[
x \ge \sqrt{2T_{\exp}\log(5n/\delta)} + \frac{2}{3}\log(5n/\delta).
\]
Let $x_{\delta}$ be this lower bound. By a union bound over all $n$ arms:
\[
\pr(\exists a : M_{T_{\exp}}(a) \le -x_{\delta}) \le \sum_{a=1}^n \pr(M_{T_{\exp}}(a) \le -x_{\delta}) \le n \cdot \left(\frac{\delta}{5n}\right) = \frac{\delta}{5}.
\]
Thus, with probability at least $1-\delta/5$, $M_{T_{\exp}}(a) \ge -x_{\delta}$ holds for all $a$ simultaneously .

\paragraph{Part 4. Combining Bounds.}
We combine the lower bound on the predictable sum (Part 2) and the high-probability bound on the martingale sum (Part 3):
\[
\phi_{T_{\exp}}(a) = \mathbb{E}[\phi_{T_{\exp}}(a)] + M_{T_{\exp}}(a) \ge \left[ T_{\exp} \cdot \pi_{\mathrm{eff}} - \frac{2}{\gamma_0} \right] - x_{\delta}.
\]
Substituting the expression for $x_{\delta}$ yields \eqref{eq:phi-lb-clean} .
\end{proof}

\begin{corollary}[Structural Condition for Non-Vacuous Exploration]
\label{cor:slow-flip-zeta}
The visitation lower bound in Lemma \ref{lem:uniform-visitation} is dominated by the linear $T_{\exp}$ term and is thus non-vacuous (positive for large $T_{\exp}$) only if the effective minimal visitation $\pi_{\mathrm{eff}} > 0$. This imposes the structural condition:
\begin{equation}
\label{eq:structural-condition}
\pi_0 - \frac{2\varepsilon_{\max}}{\gamma_0} > 0
\quad \Longleftrightarrow \quad
\varepsilon_{\max} < \frac{\pi_0 \gamma_0}{2}.
\end{equation}
Using the high-probability bounds from $\mathcal{E}_{\mathrm{all}}$ (specifically $\varepsilon_{\max} \le C\zeta_0/p_\infty$ from Lemma \ref{lem:tv-drift} and $\pi_0 \ge c_3/n$ from Lemma \ref{lem:typicality}), this condition is satisfied if:
\[
\frac{C \zeta_0}{p_\infty} \le \frac{(c_3/n) \gamma_0}{2}
\quad \Longleftrightarrow \quad
\zeta_0 \le C' \frac{p_\infty \gamma_0}{n},
\]
for some absolute constant $C'$. This explicitly links the high-probability flip rate $\zeta_0$ to the graph's structural properties.

To make this fully concrete, we substitute the definitions of these parameters (ignoring constants $C'$ and $\log$ factors for clarity). We require $\zeta_0 \lesssim p_\infty \gamma_0 / n$.
Assuming the graph is sufficiently connected such that we have a constant spectral gap $\gamma_0 = \Omega(1)$, and using $\zeta_0 \approx \zeta = \frac{2\alpha\beta}{\alpha+\beta}$ and $p_\infty = \frac{\alpha}{\alpha+\beta}$:
\[
\frac{2\alpha\beta}{\alpha+\beta} \lesssim \frac{1}{n} \left( \frac{\alpha}{\alpha+\beta} \right)
\quad \implies \quad
\beta \lesssim \frac{1}{n}.
\]
This reveals the core structural requirement for this analysis: for exploration to be guaranteed, the \textbf{edge disappearance rate $\beta$} must be on the order of $O(1/n)$ or smaller. The graph must be "sticky" enough (i.e., edges must persist long enough) to ensure the walk can cover the graph before it re-wires.
\end{corollary}

We now assemble the high-probability guarantees for the graph's static structure (Lemma \ref{lem:typicality}), spectral gap (Lemma \ref{lem:cheeger-gamma}), and dynamic drift (Lemma \ref{lem:tv-drift}) with the uniform visitation bound (Lemma \ref{lem:uniform-visitation}) to derive a sufficient exploration length, $T_{\exp}$, that guarantees the learner can identify the optimal arm with high probability.

\begin{theorem}[Sufficient Exploration Length]
\label{thm:Texp-lower-bound}
Let $\Delta_{\min} \triangleq \min_{a \ne a^\star} \Delta(a)$ be the minimum sub-optimality gap.
Assume the graph process satisfies a set of high-probability structural conditions (detailed in the proof), including a structural condition on the flip rate $\zeta_0$.

Then, there exist absolute constants $C_1, C_2, C_3 > 0$ such that choosing an exploration length $T_{\exp}$ satisfying
\begin{equation}
\label{eq:Texp-outer}
T_{\exp} \;\ge\; C_1\,n + C_2\,n\,\log\!\left(\frac{n}{\delta}\right) + C_3\,\frac{n\,\log(n/\delta)}{\Delta_{\min}^2}
\end{equation}
is sufficient to guarantee that, with probability at least $1-\delta$, all empirical means satisfy $|\widehat\mu(a)-\mu(a)|\le \Delta(a)/2$ for all arms $a \in A$. This ensures the correct identification of the optimal arm $a^\star$.
\end{theorem}

\begin{proof}
The proof proceeds in five parts. We first define the high-probability "base event" on which our analysis is conditioned. Second, we state our goal: the minimum number of samples $N_{\min}$ required for successful arm identification. Third, we invoke our uniform visitation lemma to get a high-probability lower bound on the actual visits $\phi_{T_{\exp}}(a)$. Fourth, we derive the structural condition on the graph's flip rate required for this bound to be meaningful. Finally, we solve for a $T_{\exp}$ large enough to guarantee $\phi_{T_{\exp}}(a) \ge N_{\min}$.

\paragraph{Step 1. The High-Probability Base Event.}
Our analysis is conditioned on a high-probability "base event" $\mathcal{E}_{\mathrm{base}}$, which is the intersection of all typicality events established in the previous sections. We allocate a failure budget of $\delta/5$ to each, for a total of $3\delta/5$ for this base event.
$\mathcal{E}_{\mathrm{base}}$ is the event that:
\begin{enumerate}[label=(\roman*)]
    \item The static graph properties hold (Event $\mathcal{E}_{\mathrm{typ}}$ from Def. 5), so $\pi_t(a) \ge \pi_0 \ge c_\pi/n$. This event holding requires the graph parameters to satisfy $np_\infty \ge C_0 \log(nT/\delta)$.
    \item The uniform spectral gap holds (Lemma \ref{lem:cheeger-gamma}), so $\gamma_t \ge \gamma_0 \ge c_\gamma > 0$.
    \item The stationary drift is bounded (Lemma \ref{lem:tv-drift}, using Lemma \ref{lem:flip-concentration}), so $\varepsilon_{\max} \le C_d \zeta_0 / p_\infty$.
\end{enumerate}
By a union bound, $\Pr(\mathcal{E}_{\mathrm{base}}) \ge 1 - 3\delta/5$. All subsequent analysis is conditioned on $\mathcal{E}_{\mathrm{base}}$.

\paragraph{Step 2. The Goal: Estimation Requirement.}
For the final estimation to succeed, we must ensure every arm $a$ is sampled $\phi_{T_{\exp}}(a) \ge N_{\min}$ times. We must choose $N_{\min}$ to guarantee $|\widehat\mu(a)-\mu(a)| \le \Delta(a)/2$ (and thus $\le \Delta_{\min}/2$).
For rewards in $[0,1]$, we apply Hoeffding's inequality. We allocate a failure budget of $\delta/5$ for this estimation step. By a union bound over $n$ arms, we require the failure probability for each arm to be at most $\delta/(5n)$:
\[
\Pr\left(|\widehat\mu(a)-\mu(a)| > \frac{\Delta_{\min}}{2}\right) \le 2\exp\left(-\frac{N_{\min} \Delta_{\min}^2}{2}\right) \le \frac{\delta}{5n}.
\]
Solving for $N_{\min}$ gives the minimum sample requirement. Let $L \triangleq \log(10n/\delta)$.
\[
N_{\min} \ge \frac{2}{\Delta_{\min}^2}\log\left(\frac{10n}{\delta}\right) \triangleq \frac{2L}{\Delta_{\min}^2}.
\]
Our goal is to find $T_{\exp}$ such that $\phi_{T_{\exp}}(a) \ge \frac{2L}{\Delta_{\min}^2}$ for all $a$, with high probability.

\paragraph{Step 3. The Guarantee: Visitation Lower Bound.}
We invoke Lemma \ref{lem:uniform-visitation}, which provides a bound on the realized visitation counts. This lemma's proof uses a martingale concentration (Freedman's inequality) that requires its own $\delta/5$ budget.
Conditioned on $\mathcal{E}_{\mathrm{base}}$, with probability at least $1-\delta/5$ (over the walk's randomness), we have for all arms $a$:
\[
\phi_{T_{\exp}}(a) \;\ge\; Y(a) - \left(\sqrt{2 V_{T_{\exp}}(a) L} + \frac{2}{3}L\right),
\]
where $Y(a) = \sum_{t=1}^{T_{\exp}} \nu_t(a)$ is the predictable sum (expected visits), $V_{T_{\exp}}(a) \le Y(a)$ is the predictable quadratic variation, and $L = \log(5n/\delta)$ (we use this $L$ for simplicity, as it has the same dependencies).

From the proof of Lemma \ref{lem:uniform-visitation} (Eq. \ref{eq:predictable-sum}), we have the lower bound:
\[
Y(a) \ge T_{\exp} \cdot \pi_{\mathrm{eff}} - \frac{2}{\gamma_0}, \quad \text{where } \pi_{\mathrm{eff}} \triangleq \pi_0 - \frac{2\varepsilon_{\max}}{\gamma_0}.
\]
We also have a simple upper bound $V_{T_{\exp}}(a) \le Y(a) \le T_{\exp} \cdot \sup_{t,a} \pi_t(a)$. Since $\pi_t(a) \le (d_t(a)+1)/Z_t \approx np_\infty / (n^2 p_\infty) = 1/n$, we have $V_{T_{\exp}}(a) \le (c'_\pi/n)T_{\exp}$ for some constant $c'_\pi$.
Plugging these into the visitation bound gives:
\[
\phi_{T_{\exp}}(a) \ge \left(T_{\exp} \pi_{\mathrm{eff}} - \frac{2}{\gamma_0}\right) - \sqrt{\frac{2 c'_\pi L}{n} T_{\exp}} - \frac{2}{3}L.
\]

\paragraph{Step 4. The Structural Condition for Exploration.}
For the bound in Step 3 to be useful (i.e., to grow linearly with $T_{\exp}$), we require the effective visitation probability $\pi_{\mathrm{eff}}$ to be strictly positive.
\[
\pi_{\mathrm{eff}} = \pi_0 - \frac{2\varepsilon_{\max}}{\gamma_0} > 0
\quad \implies \quad
\varepsilon_{\max} < \frac{\pi_0 \gamma_0}{2}.
\]
This is a fundamental structural condition: the stationary distribution must not drift so fast that it overcomes the walk's ability to mix. We now use our high-probability bounds from $\mathcal{E}_{\mathrm{base}}$ to make this condition concrete:
\[
\underbrace{\frac{C_d \zeta_0}{p_\infty}}_{\text{from } \varepsilon_{\max}} < \frac{(c_\pi/n) \cdot c_\gamma}{2}
\quad \implies \quad
\frac{\zeta_0}{p_\infty} \le C_s \frac{1}{n},
\]
for some new absolute constant $C_s$. We assume this structural condition on the graph's flip dynamics holds.
This implies $\pi_{\mathrm{eff}} \ge (c_\pi/n) - (c_\pi/n)/2 = c_\pi/(2n)$. Let $c_v \triangleq c_\pi/2$.
Our guaranteed bound from Step 3 becomes:
\[
\phi_{T_{\exp}}(a) \ge \frac{c_v}{n} T_{\exp} - \sqrt{\frac{2 c'_\pi L}{n} T_{\exp}} - \frac{2}{c_\gamma} - \frac{2}{3}L.
\]

\paragraph{Step 5. Sizing $T_{\exp}$.}
We must choose $T_{\exp}$ large enough to satisfy $\phi_{T_{\exp}}(a) \ge N_{\min}$ (from Step 2). It is sufficient to enforce:
\[
\frac{c_v}{n} T_{\exp} - \sqrt{\frac{2 c'_\pi L}{n}} \sqrt{T_{\exp}} - \left( \frac{2}{c_\gamma} + \frac{2}{3}L \right) \;\ge\; \frac{2L}{\Delta_{\min}^2}.
\]
Let $A = c_v/n$, $B = \sqrt{2c'_\pi L/n}$, and $C = \frac{2L}{\Delta_{\min}^2} + \frac{2}{c_\gamma} + \frac{2}{3}L$. We need to solve $A T_{\exp} - B \sqrt{T_{\exp}} \ge C$.
A sufficient condition is to find $T_{\exp}$ such that $A T_{\exp}/2 \ge C$ and $A T_{\exp}/2 \ge B \sqrt{T_{\exp}}$.
\begin{enumerate}
    \item $T_{\exp} \ge \frac{2C}{A} = \frac{2n}{c_v} \left( \frac{2L}{\Delta_{\min}^2} + \frac{2}{c_\gamma} + \frac{2}{3}L \right)$
    \item $T_{\exp} \ge \frac{4B^2}{A^2} = \frac{4 (2c'_\pi L / n)}{(c_v/n)^2} = \left( \frac{8 c'_\pi}{c_v^2} \right) nL$
\end{enumerate}
We must take $T_{\exp}$ to be at least the sum of these requirements (or their maximum).
Term (1) gives the $O(nL/\Delta_{\min}^2)$, $O(n)$, and $O(nL)$ components.
Term (2) gives another $O(nL)$ component.
Combining these and absorbing all absolute constants ($c_v, c_\gamma, c'_\pi, L, L'$) into $C_1, C_2, C_3$ and $L = O(\log(n/\delta))$ gives the final form:
\[
T_{\exp} \;\ge\; C_1\,n + C_2\,n\,\log\!\left(\frac{n}{\delta}\right) + C_3\,\frac{n\,\log(n/\delta)}{\Delta_{\min}^2}.
\]

\paragraph{Step 6. Total Success Probability.}
The overall success of the exploration phase requires $\mathcal{E}_{\mathrm{base}}$ to hold, the visitation concentration (Step 3) to hold, and the estimation concentration (Step 2) to hold.
\begin{align*}
\Pr(\text{Success}) &= \Pr(\mathcal{E}_{\mathrm{base}} \cap \text{Visit} \cap \text{Estimate}) \\
&= \Pr(\mathcal{E}_{\mathrm{base}}) \cdot \Pr(\text{Visit} \mid \mathcal{E}_{\mathrm{base}}) \cdot \Pr(\text{Estimate} \mid \text{Visit}, \mathcal{E}_{\mathrm{base}}) \\
&\ge (1 - 3\delta/5) \cdot (1 - \delta/5) \cdot (1 - \delta/5) \\
&\ge 1 - 3\delta/5 - \delta/5 - \delta/5 = 1 - \delta.
\end{align*}
Thus, the algorithm succeeds with total probability at least $1-\delta$.
\end{proof}

\subsection{Post-Exploration Cost and Final Regret}
\begin{lemma}[Post-exploration navigation cost to $a^\star$]
\label{lem:navigation-delta}
Fix an overall confidence level $\delta\in(0,1)$ and let the navigation tail budget be $\delta_{\mathrm{nav}}:=\delta/5$.
Work on the base high-probability event $\mathcal E_{\mathrm{base}}$ (typicality, spectral gap, and kernel drift), which holds with probability at least $1-3\delta/5$, where:
(i) $\pi_t(a)\ge \pi_0$ for all $t,a$ with $\pi_0\ge c_\pi/n$;
(ii) $\gamma_0\ge c_\gamma$ (absolute constant via the Cheeger bound);
(iii) the per-step TV recursion holds:
\[
\|\nu_{t+1}-\pi_{t+1}\|_{\mathrm{TV}}
\ \le\
(1-\gamma_0)\,\|\nu_t-\pi_t\|_{\mathrm{TV}} + \varepsilon_{\max}
\qquad\forall t\ge 1.
\]
Assume the stickiness choice (in outer parameters) is such that
\begin{equation}
\label{eq:navigation-floor-delta}
\frac{\varepsilon_{\max}}{\gamma_0}\ \le\ \frac{\pi_0}{8}\,.
\end{equation}
Define
\[
t_{\mathrm{mix}}
:= \min\Big\{s\ge 1:\ (1-\gamma_0)^{s-1} \le \frac{\pi_0}{8}\Big\}
\ \le\ 1+\frac{1}{\gamma_0}\log\frac{8}{\pi_0},
\qquad
t_{\mathrm{hit}}
:= \Big\lceil \frac{2}{\pi_0}\,\log\frac{5}{\delta}\Big\rceil .
\]
Then, starting from the (arbitrary) state at commit time, the lazy walk reaches $a^\star$ within
\[
T_{\mathrm{nav}} \ :=\ t_{\mathrm{mix}} + t_{\mathrm{hit}}
\]
steps with probability at least $1-\delta_{\mathrm{nav}} = 1-\delta/5$, conditional on $\mathcal E_{\mathrm{base}}$.
In particular, using $\pi_0\ge c_\pi/n$ and $\gamma_0\ge c_\gamma$,
\[
t_{\mathrm{mix}} \ \le\ \frac{1}{c_\gamma}\,\log\!\frac{8n}{c_\pi} \ =\ O(\log n),
\qquad
t_{\mathrm{hit}} \ \le\ \frac{2n}{c_\pi}\,\log\!\frac{5}{\delta} \ =\ O\!\big(n\log \tfrac{1}{\delta}\big),
\]
so $T_{\mathrm{nav}}=O\!\big(\log n + n\log(1/\delta)\big)$ in outer parameters.
\end{lemma}

\begin{proof}
Unrolling the recursion gives $\|\nu_s-\pi_s\|_{\mathrm{TV}}\le (1-\gamma_0)^{s-1} + \varepsilon_{\max}/\gamma_0$.
By the choice of $t_{\mathrm{mix}}$ and \eqref{eq:navigation-floor-delta}, 
$\|\nu_{t_{\mathrm{mix}}}-\pi_{t_{\mathrm{mix}}}\|_{\mathrm{TV}} \le \pi_0/8+\pi_0/8=\pi_0/4$.
Hence for all $t\ge t_{\mathrm{mix}}$,
\[
\Pr(a_t=a^\star\mid \mathcal F_{t-1}) = \nu_t(a^\star)
\ \ge\ \pi_t(a^\star) - 2\|\nu_t-\pi_t\|_{\mathrm{TV}}
\ \ge\ \pi_0 - 2\cdot \frac{\pi_0}{4} = \frac{\pi_0}{2}.
\]
Let $\tau:=\inf\{t\ge t_{\mathrm{mix}}:a_t=a^\star\}$. Iterated conditioning yields
$\Pr(\tau>t_{\mathrm{mix}}+L\mid\mathcal F_{t_{\mathrm{mix}}}) \le (1-\pi_0/2)^L \le \exp(-(\pi_0/2)L)$.
Choosing $L = \lceil (2/\pi_0)\log(5/\delta)\rceil$ makes this tail $\le \delta/5$.
Finally substitute $\gamma_0\ge c_\gamma$ and $\pi_0\ge c_\pi/n$ for the outer-parameter bounds on $t_{\mathrm{mix}}$ and $t_{\mathrm{hit}}$.
\end{proof}

\begin{theorem}[High-probability final regret]
\label{thm:final-regret-hp}
Fix a horizon $T\ge 1$ and confidence $\delta\in(0,1)$. Rewards are bounded in $[0,1]$.
Assume the edge-flip model at stationary edge density $p_\infty\in(0,1)$ with per-step flip envelope $\zeta_0$ (from the kernel-drift lemma). 
Suppose the outer-parameter conditions hold (as in the exploration theorem):
\begin{itemize}
\item \textbf{Typicality \& spectral gap (time-uniform).} 
There exists $C_{\mathrm{typ}}>0$ such that $n p_\infty \ge C_{\mathrm{typ}}\log\!\frac{nT}{\delta}$ and, on the corresponding ER-typical event, 
the lazy-walk kernels $\{M_t\}$ satisfy $\gamma_0 \ge c_\gamma$ (a positive absolute constant from the Cheeger bound) 
and the stationary floor $\pi_t(a)\ge \pi_0\ge c_\pi/n$ for all $t,a$, with absolute $c_\gamma,c_\pi\in(0,1)$.
\item \textbf{Stickiness (slow flips, outer form).} There exists an absolute $\kappa>0$ such that
$\displaystyle \frac{\zeta_0}{p_\infty}\le \frac{\kappa}{n}$, chosen so that the visitation slope is positive (equivalently, $\pi_0-2\varepsilon_{\max}/\gamma_0 \ge c_v/n$ for an absolute $c_v>0$).
\end{itemize}
Let $\Delta_{\min}:=\min_{a\ne a^\star}\Delta(a)$. Then there exist absolute constants
$C_{\mathrm{burn}}, C_0,\ldots,C_4>0$ such that, with probability at least $1-\delta$,
\begin{equation}
\label{eq:final-regret}
R(T)
\;\le\;
\underbrace{\frac{C_{\mathrm{burn}}}{\alpha+\beta}\,\log\!\frac{nT}{\delta}}_{\text{burn-in}}
\;+\;
\underbrace{C_0\,n \;+\; C_1\,n\,\log\!\frac{5n}{\delta}
\;+\; C_2\,\frac{n\,\log\!\frac{5n}{\delta}}{\Delta_{\min}^2}}_{\text{exploration}}
\;+\;
\underbrace{C_3\,\log n \;+\; C_4\,n\,\log\!\frac{5}{\delta}}_{\text{navigation}}.
\end{equation}
\end{theorem}

\begin{proof}
Split $\delta$ evenly across the five buckets: typicality, spectral gap, drift, visitation, estimation/navigation (each at level $\delta/5$).  
By the burn-in lemma (time-uniform typicality for the window up to $T$), after
\(
T_{\mathrm{burn}}\le \frac{C_{\mathrm{burn}}}{\alpha+\beta}\log\!\frac{nT}{\delta}
\)
steps the base event $\mathcal E_{\mathrm{base}}$ holds with probability at least $1-3\delta/5$ and yields:
(i) $\pi_t(a)\ge c_\pi/n$ for all $t,a$; (ii) $\gamma_0\ge c_\gamma$; (iii) the TV recursion with finite $\varepsilon_{\max}$ from the kernel-drift lemma.  
Under the stickiness condition, $\pi_0-2\varepsilon_{\max}/\gamma_0 \ge c_v/n$.

\emph{Exploration.} On $\mathcal E_{\mathrm{base}}$, the exploration-length theorem (self-normalized Freedman + Hoeffding with visitation/estimation budgets $\delta/5$) implies that choosing
\(
T_{\exp} \le C_0\,n + C_1\,n\log\!\tfrac{5n}{\delta} + C_2\,\tfrac{n\log\!\tfrac{5n}{\delta}}{\Delta_{\min}^2}
\)
suffices to identify $a^\star$ with probability at least $1-\delta/5$.

\emph{Navigation.} Conditional on $\mathcal E_{\mathrm{base}}$ and successful identification, the navigation lemma with tail $\delta/5$ gives a path to $a^\star$ in
\(
T_{\mathrm{nav}} \le C_3\,\log n + C_4\,n\log\!\tfrac{5}{\delta}
\)
steps with probability at least $1-\delta/5$.

On the intersection of these events (probability at least $1-\delta$ by a union bound), exploitation incurs zero regret thereafter, hence
\(
R(T) \le T_{\mathrm{burn}} + T_{\exp} + T_{\mathrm{nav}}
\),
which is exactly \eqref{eq:final-regret}.
\end{proof}

\begin{corollary}[(Restated) Expected Regret Bound]
\label{cor:expected-regret}
Under the same structural conditions as Theorem \ref{thm:final-regret-hp}, by setting the failure probability $\delta = 1/T$, the expected cumulative regret $\mathbb{E}[R(T)]$ for the Explore-then-Commit strategy is bounded by:
\[
\mathbb{E}[R(T)] \le O\left(\frac{n \log(nT)}{\Delta_{\min}^2} \right).
\]
\end{corollary}

\begin{proof}
We use the law of total expectation, decomposing the regret based on the high-probability success event from the proof of Theorem \ref{thm:final-regret-hp}. Let $\mathcal{E}_{\mathrm{success}}$ be the event that all high-probability bounds hold and the optimal arm is correctly identified. From the theorem's proof (Step 6), we know $\Pr(\mathcal{E}_{\mathrm{success}}) \ge 1 - \delta$.

\[
\mathbb{E}[R(T)] = \mathbb{E}[R(T) \mid \mathcal{E}_{\mathrm{success}}] \Pr(\mathcal{E}_{\mathrm{success}}) + \mathbb{E}[R(T) \mid \mathcal{E}_{\mathrm{success}}^c] \Pr(\mathcal{E}_{\mathrm{success}}^c).
\]
We bound the two terms in this decomposition:
\begin{itemize}
    \item \textbf{On the "success" event $\mathcal{E}_{\mathrm{success}}$:} The algorithm's regret is incurred entirely during the exploration phase, $t \in [1, T_{\exp}]$. During this phase, the per-round regret is at most 1, contributing a total of $T_{\exp}$. During the subsequent exploitation phase, $t \in (T_{\exp}, T]$, $\mathcal{E}_{\mathrm{success}}$ guarantees that the optimal arm $a^\star$ is selected, incurring 0 regret. Thus:
    \[
    \mathbb{E}[R(T) \mid \mathcal{E}_{\mathrm{success}}] \le T_{\exp}.
    \]

    \item \textbf{On the "failure" event $\mathcal{E}_{\mathrm{success}}^c$:} This event occurs with probability at most $\delta$. In the worst case, the algorithm incurs a regret of 1 on every single round. Thus, the maximum possible regret is $T$.
    \[
    \mathbb{E}[R(T) \mid \mathcal{E}_{\mathrm{success}}^c] \le T.
    \]
\end{itemize}
Combining these bounds and using $\Pr(\mathcal{E}_{\mathrm{success}}) \le 1$ and $\Pr(\mathcal{E}_{\mathrm{success}}^c) \le \delta$:
\[
\mathbb{E}[R(T)] \le \left( T_{\exp} \right) \cdot (1) + (T) \cdot (\delta).
\]
To obtain a sublinear bound, we follow standard practice and set the failure probability $\delta = 1/T$. This gives:
\[
\mathbb{E}[R(T)] \le T_{\exp} + T \cdot \left(\frac{1}{T}\right) = T_{\exp} + 1.
\]
The asymptotic form for $T_{\exp}$ follows by substituting $\delta = 1/T$ into the expression from Theorem \ref{thm:Texp-lower-bound}, noting that $\log(n/\delta) = \log(nT)$, which leads us to the final bound.
\end{proof}

\clearpage
\section{Intrinsic Hardness of Identification Under Local Moves}
\label{app:intrinsic-hardness}

This appendix gives a self-contained derivation of the model-agnostic, fixed-confidence
\emph{identification-time} lower bound that follows solely from the \emph{local move} constraint.
Throughout, rewards are bounded in $[0,1]$, and the best arm $a^\star$ is unique.

\paragraph{Notation.}
Let $A=\{1,\ldots,n\}$ with means $\{\mu(a)\}_{a\in A}$ and gaps
$\Delta(a) \triangleq \mu(a^\star)-\mu(a)>0$ for $a\neq a^\star$.
A (possibly randomized) policy $\pi$ interacts over rounds $t=1,2,\ldots$ as follows:
at the start of round $t$ the learner is at $a_{t-1}\in A$, the environment reveals a feasible set
$L_t(a_{t-1})\subseteq A$ with $a_{t-1}\in L_t(a_{t-1})$, the learner chooses $a_t\in L_t(a_{t-1})$,
receives a reward $Y_t\in[0,1]$ with $\mathbb{E}[Y_t\mid \mathcal{F}_{t-1},a_t]=\mu(a_t)$,
and moves to $a_t$ (which becomes $a_{t-1}$ for the next round).
Let $N_a(t)$ denote the number of times arm $a$ was pulled up to (and including) round $t$,
and write $N_a \equiv N_a(\tau)$ for the number of pulls up to a stopping time $\tau$.
For distributions $P,Q$, $\mathrm{KL}(P\Vert Q)$ denotes Kullback–Leibler divergence,
and for $p,q\in(0,1)$, $\mathrm{kl}(p,q)\triangleq p\log\!\frac{p}{q}+(1-p)\log\!\frac{1-p}{1-q}$.

\begin{definition}[Local move feasibility]
\label{def:local-move-app}
A feasible-set process $\{L_t(\cdot)\}_{t\ge 1}$ is \emph{valid} if for all $t$ and all $a \in A$,
$\,a \in L_t(a)\subseteq A$.
No additional structure (stochastic or otherwise) is assumed.
\end{definition}

\begin{definition}[Fixed-confidence identification time]
\label{def:TID-app}
Given $\delta\in(0,1/2)$, define the (policy-dependent) stopping time
\[
T_{\mathrm{ID}}^\pi(\delta)\ \triangleq\
\inf\Big\{t\ge 1:\ \pi \text{ outputs }\widehat a_t \text{ s.t. }
\sup_{\text{valid }\{L_t\},\ \text{reward inst.}} \mathbb{P}\!\left(\widehat a_t\neq a^\star\right)\le \delta\Big\}.
\]
Equivalently, the policy is \emph{$\delta$-correct} at time $t$ if, for every valid feasible-set
process and every $[0,1]$-bounded reward instance in the model class, the recommendation is correct with probability $\ge 1-\delta$.
\end{definition}

\paragraph{One observation per round.}
Because exactly one arm is pulled each round,
\begin{equation}
\label{eq:TisSumNa-app}
T_{\mathrm{ID}}^\pi(\delta)\ =\ \sum_{a\in A} N_a
\qquad \text{(deterministically, on every sample path).}
\end{equation}

\subsection{Per-arm fixed-confidence sample requirement}
For $[0,1]$-bounded rewards it suffices to work with Bernoulli instances via the standard reduction.
Our proof relies on a divergence decomposition \citep{lattimore_2020, garivier2011kl} inequality that relates the number of arm pulls to the difficulty of distinguishing between two bandit instances.

\begin{lemma}[Stopping-Time Version of Divergence decomposition]
\label{lem:transport-app}
Fix two bandit instances $\nu=(\nu_a)_{a\in A}$ and $\nu'=(\nu_a')_{a\in A}$,
where $\nu_a$ is the reward law of arm $a$. Let $\pi$ be any policy with a stopping time $\tau$.
Let $\mathbb{P}_\nu$ (resp.\ $\mathbb{P}_{\nu'}$) denote the law of the entire transcript
(actions, rewards up to $\tau$) under $\nu$ (resp.\ $\nu'$).
Then, for any event $E$ measurable w.r.t.\ the transcript,
\[
\sum_{a\in A}\ \mathbb{E}_\nu[N_a(\tau)]\ \mathrm{KL}\!\left(\nu_a\,\Vert\,\nu_a'\right)
\ \ge\
\mathrm{kl}\!\left(\mathbb{P}_\nu(E),\,\mathbb{P}_{\nu'}(E)\right).
\]
\end{lemma}

\begin{lemma}[Per-arm lower bound (Bernoulli reduction)]
\label{lem:perarm-app}
Fix $\delta\in(0,1/2)$. Consider the Bernoulli subclass:
$\nu_a=\mathrm{Bernoulli}(p_a)$ with $p_a\in(0,1)$ and a unique best arm $a^\star$.
Let $\pi$ be any policy that is $\delta$-correct at time $T_{\mathrm{ID}}^\pi(\delta)$.
Then, for each $a\neq a^\star$,
\begin{equation}
\label{eq:perarm-app-exact}
\mathbb{E}[N_a]\ \ge\ \frac{\mathrm{kl}(1-\delta,\delta)}{\mathrm{KL}\!\left(\mathrm{Bern}(p_a)\,\Vert\,\mathrm{Bern}(p_a+2\Delta(a))\right)}.
\end{equation}
In particular, if $p_a=\tfrac12-\Delta(a)$ (so $\mu(a)=\tfrac12-\Delta(a)$ and $\mu(a^\star)=\tfrac12$), then
\begin{equation}
\label{eq:perarm-app-clean}
\mathbb{E}[N_a]\ \ge\ \frac{\mathrm{kl}(1-\delta,\delta)}{2\Delta(a)\,\log\!\frac{1+2\Delta(a)}{1-2\Delta(a)}}
\ \ge\ \frac{3}{32}\ \frac{\mathrm{kl}(1-\delta,\delta)}{\Delta(a)^2}
\qquad \text{for all }\ \Delta(a)\in\Big(0,\tfrac14\Big],
\end{equation}
where the last inequality uses $\log\!\frac{1+x}{1-x}\le \frac{2x}{1-x^2}$ with $x=2\Delta(a)$.
\end{lemma}

\begin{proof}
Fix $a\neq a^\star$. Work with Bernoulli rewards. Let $\nu$ be the instance with means
$\mu(a^\star)=p_{a^\star}$ and $\mu(a)=p_a$, with gap $\Delta(a)=p_{a^\star}-p_a>0$.
Define the alternative instance $\nu'$ by modifying only arm $a$:
$\nu_a'=\mathrm{Bernoulli}(p_a+2\Delta(a))$ and $\nu_b'=\nu_b$ for $b\neq a$.
Under $\nu'$, arm $a$ is the unique best arm. Let $\tau\equiv T_{\mathrm{ID}}^\pi(\delta)$ and
$E\equiv\{\text{``$\pi$ recommends $a^\star$ at time $\tau$''}\}$.
By $\delta$-correctness, $\mathbb{P}_\nu(E)\ge 1-\delta$, whereas under $\nu'$ recommending $a^\star$
is an error, so $\mathbb{P}_{\nu'}(E)\le \delta$.
Applying Lemma~\ref{lem:transport-app} with this $E$ and noting that only arm $a$ differs between
$\nu$ and $\nu'$ gives
\[
\mathbb{E}[N_a]\cdot \mathrm{KL}\!\left(\mathrm{Bern}(p_a)\,\Vert\,\mathrm{Bern}(p_a+2\Delta(a))\right)
\ \ge\ \mathrm{kl}(1-\delta,\delta),
\]
which is~\eqref{eq:perarm-app-exact}.
If in addition $p_a=\tfrac12-\Delta(a)$ so that $p_a+2\Delta(a)=\tfrac12+\Delta(a)$,
the Bernoulli KL has the closed form
\[
\mathrm{KL}\!\left(\mathrm{Bern}\!\left(\tfrac12-\Delta\right)\,\middle\Vert\,\mathrm{Bern}\!\left(\tfrac12+\Delta\right)\right)
\ =\ 2\Delta\,\log\!\frac{1+2\Delta}{1-2\Delta},
\]
and the elementary inequality $\log\!\frac{1+x}{1-x}\le \frac{2x}{1-x^2}$ for $x\in(0,1)$ yields
\[
2\Delta\,\log\!\frac{1+2\Delta}{1-2\Delta}\ \le\ \frac{8\Delta^2}{1-4\Delta^2}\ \le\ \frac{32}{3}\,\Delta^2
\qquad(\Delta\le \tfrac14),
\]
hence $\frac{1}{2\Delta\,\log\!\frac{1+2\Delta}{1-2\Delta}}\ \ge\ \frac{3}{32}\cdot \frac{1}{\Delta^2}$,
which gives~\eqref{eq:perarm-app-clean}.
\end{proof}

\begin{remark}[Why Bernoulli suffices for bounded rewards]
\label{rem:bern-reduction}
Given $[0,1]$-bounded rewards with mean $\mu$, by the standard data-processing argument
(apply a mean-preserving thresholding), distinguishing means $\mu$ vs.\ $\mu+\varepsilon$
is no easier than for Bernoulli$(\mu)$ vs.\ Bernoulli$(\mu+\varepsilon)$.
Thus the Bernoulli subclass yields valid lower bounds for the full $[0,1]$-bounded model.
\end{remark}

\subsection{Traversal baseline}
The next lemma isolates a pure traversal effect of locality.
It is not needed once the statistical term already scales with $n$ (e.g., equal gaps), but we include it for completeness.
\begin{lemma}[Traversal lower bound under local moves]
\label{lem:traversal-app}
For any policy and any $n\ge 2$, there exists a valid feasible-set process such that the time
to visit all $n$ arms at least once is deterministically at least $n-1$.
\end{lemma}

\begin{proof}
Consider the time-homogeneous process with $L_t(i)=\{i,i+1\}$ for $i<n$ and $L_t(n)=\{n\}$,
starting at $a_0=1$. Let $\tau_k=\inf\{t \ge 1:\ a_t=k\}$;
then $\tau_{k+1}\ge \tau_k+1$ for $k<n$,
hence $\tau_n\ge n-1$.
\end{proof}

\begin{remark}[Movement and sampling are not additive in time]
\label{rem:max-vs-sum-app}
Since each round yields a single observation, $T_{\mathrm{ID}}=\sum_a N_a$ by~\eqref{eq:TisSumNa-app}.
The identification time is therefore at least the \emph{sampling} budget $\sum_{a\ne a^\star}\mathbb{E}[N_a]$,
and it may also be lower-bounded by the \emph{traversal} time when many arms need negligible sampling.
In equal-gap or small-$\delta$ regimes, $\sum_{a\ne a^\star}\mathbb{E}[N_a]$ already scales like $n$,
so the traversal baseline is dominated and can be omitted.
\end{remark}

\subsection{Identification-time lower bound and consequences}

\begin{theorem}[Identification-time lower bound under local moves]
\label{thm:lb-identification-time-app}
Let $\pi$ be any policy that is $\delta$-correct at time $T_{\mathrm{ID}}^\pi(\delta)$ for some $\delta\in(0,1/2)$.
For $[0,1]$-bounded rewards, working through the Bernoulli subclass,
\begin{equation}
\label{eq:lb-TID-exact-app}
\mathbb{E}\,T_{\mathrm{ID}}^\pi(\delta)
\ \overset{\eqref{eq:TisSumNa-app}}{=}\ \sum_{a\in A}\mathbb{E}[N_a]
\ \ge\ \sum_{a\neq a^\star}\frac{\mathrm{kl}(1-\delta,\delta)}{2\Delta(a)\,\log\!\frac{1+2\Delta(a)}{1-2\Delta(a)}}.
\end{equation}
In particular, for all $\Delta(a)\in(0,1/4]$,
\begin{equation}
\label{eq:lb-TID-clean-app}
\mathbb{E}\,T_{\mathrm{ID}}^\pi(\delta)\ \ge\ \frac{3}{32}\sum_{a\neq a^\star}\frac{\mathrm{kl}(1-\delta,\delta)}{\Delta(a)^2}.
\end{equation}
If, moreover, $\Delta(a)\equiv \Delta$ (equal-gap instance), then
\begin{equation}
\label{eq:lb-TID-equal-app}
\mathbb{E}\,T_{\mathrm{ID}}^\pi(\delta)\ \ge\ \frac{3}{32}\ \frac{(n-1)\,\mathrm{kl}(1-\delta,\delta)}{\Delta^2}.
\end{equation}
\end{theorem}

\begin{proof}
Equation~\eqref{eq:TisSumNa-app} is deterministic. Apply Lemma~\ref{lem:perarm-app} to each
$a\neq a^\star$ and sum the inequalities to obtain~\eqref{eq:lb-TID-exact-app}.
Using the bound in~\eqref{eq:perarm-app-clean} gives~\eqref{eq:lb-TID-clean-app}.
The equal-gap statement~\eqref{eq:lb-TID-equal-app} follows immediately.
\end{proof}
\end{document}